%% file: 00.arXiv__main.tex
\documentclass[11pt]{article}
\usepackage[a4paper,top=3cm,bottom=2cm,left=2cm,right=2cm,marginparwidth=1.75cm]{geometry}

\usepackage[utf8]{inputenc} 
\usepackage[T1]{fontenc}
\usepackage[numbers]{natbib}
\usepackage[colorlinks=True,citecolor=blue,urlcolor=blue,pagebackref=true,backref=true]
{hyperref}
\renewcommand*{\backrefalt}[4]{%
    \ifcase #1 \footnotesize{(Not cited.)}%
    \or        \footnotesize{(Cited on page~#2.)}%
    \else      \footnotesize{(Cited on pages~#2.)}%
    \fi}

\usepackage[none]{hyphenat}
\usepackage{breakcites}
\usepackage[utf8]{inputenc} 
\usepackage[T1]{fontenc}    
\usepackage{hyperref}       
\hypersetup{
  colorlinks = true,
  urlcolor = blue,
  linkcolor = blue,
  citecolor = blue
}
\usepackage{booktabs}       
\usepackage{amsfonts}       
\usepackage{nicefrac}       
\usepackage{microtype}      
\usepackage{xcolor}         
\usepackage{caption}
\usepackage{subcaption}
\usepackage{float}
\usepackage{amsmath,amsthm}
\usepackage{amssymb}
\usepackage{mathtools}
\usepackage{natbib}
\usepackage{soul}
\usepackage{bm}
\usepackage{enumitem}
\usepackage{multirow}
\usepackage{wrapfig}
\usepackage{scalerel}
\allowdisplaybreaks
\usepackage{dsfont}
\usepackage{algorithm}

\usepackage[algo2e,linesnumbered]{algorithm2e}
\SetKwInput{KwInput}{Input}
\SetKwInput{KwOutput}{Output}

\usepackage{makecell}

\usepackage{svg}

\usepackage{xcolor}
\usepackage{thmtools,thm-restate}
\usepackage{cleveref}

 \newtheorem{theorem}{Theorem}
 \newtheorem{definition}{Definition}
 \newtheorem{lemma}{Lemma}
 \newtheorem{remark}{Remark}
 \newtheorem{assumption}{Assumption}
 \newtheorem{proposition}{Proposition}
 \newtheorem{corollary}{Corollary}

\title{$k$-SVD with Gradient Descent}
\author{Yassir Jedra\footnote{This work was done when Yassir Jedra was at MIT. For future correspondence please use: \texttt{y.jedra@imperial.ac.uk}} \\ MIT \\\texttt{jedra@mit.edu} \and  Devavrat Shah \\ MIT \\devavrat@mit.edu}

\date{}

\input{formatting}

\begin{document}

\maketitle

\setlength{\parindent}{0cm}

\setlength{\parskip}{6pt}

\input{00.arXiv_abstract}

\input{01.arXiv_intro}

\input{02.arXiv_main_results}

\input{03.arXiv_related}
\input{04.arXiv_prelim}

\input{05.arXiv_analysis}

\input{06.arXiv_experiments}

\input{07.arXiv_conclusion}

\newpage 
\bibliographystyle{alpha}
\bibliography{references}

\appendix

\input{09.arXiv_appB_analysis}

\input{10.arXiv_appC_acceleration}

\input{11.arXiv_appD_miscellineous}

\input{12.arXiv_appE_experiments}

\end{document}

%% file: formatting.tex
\newcommand{\cC}{\mathcal{C}}

\newcommand{\cX}{\mathcal{X}}

\newcommand{\cN}{\mathcal{N}}

\newcommand{\PP}{\mathbb{P}}
\newcommand{\RR}{\mathbb{R}}

\DeclareMathOperator{\rank}{rank}
\DeclareMathOperator{\diag}{diag}
\DeclareMathOperator{\tr}{tr}

\newcommand{\F}{\textup{F}}

%% file: 00.arXiv_abstract.tex
\begin{abstract}
    The emergence of modern compute infrastructure for iterative optimization has led to great interest in developing optimization-based approaches for a scalable computation of $k$-SVD, i.e., the $k\geq 1$ largest singular values and corresponding vectors of a matrix of rank $d \geq 1$. Despite lots of exciting recent works, 
    all prior works fall short in this pursuit. Specifically, the existing results are either for the \emph{exact-parameterized} (i.e., $k = d$) and \emph{over-parameterized} (i.e., $k > d$) settings; or only establish local convergence guarantees; or use a step-size that requires problem-instance-specific oracle-provided information. In this work, we complete this pursuit by providing a gradient-descent method with a simple, universal rule for step-size selection (akin to pre-conditioning), that provably finds $k$-SVD for a matrix of any rank $d \geq 1$. We establish that the gradient method with random initialization enjoys global linear convergence for any $k, d \geq 1$. Our convergence analysis reveals that the gradient  method has an attractive region, and within this attractive region, the method behaves like Heron's method (a.k.a. the Babylonian method). Our analytic results about the said attractive region imply that the gradient method can be enhanced by means of Nesterov's momentum-based acceleration technique. The resulting improved convergence rates match those of rather complicated methods typically relying on Lanczos iterations or variants thereof. We provide an empirical study to validate the theoretical results. 
\end{abstract}

%% file: 01.arXiv_intro.tex
\section{Introduction}

\medskip
\textbf{The task.} 
Consider $M \in \RR^{m \times n}$ a matrix of rank $d \le m \wedge n$. Let
SVD of $M$ be given as $M = U \Sigma V^\top$, where $\Sigma = \diag(\sigma_1, \dots, \sigma_d) \in \RR^{d \times d}$, 
with $ \sigma_1, \dots,\sigma_d$ being the singular values of $M$ in decreasing order,  
$U \in \RR^{n \times d}$ (resp. $V \in \RR^{d \times d}$) be semi-orthogonal matrix containing 
the left (resp. right) singular vectors.\footnote{We adopt the convention that the $\ell$-th element of 
the diagonal of $\Sigma$ is $\sigma_\ell$, and that $\ell$-th column of $U$ (resp. of $V$) denoted 
$u_\ell$ (resp. denoted $v_\ell$) are its corresponding $\ell$-th left (resp. left) singular vector. 
The condition number of $M$ is defined as $\kappa = \sigma_1 / \sigma_d$.} 
Our objective is to find the the $k$-SVD of $M$, i.e. the leading $k$ singular 
values, $\sigma_i, i\leq k$,  corresponding left (resp. right) singular vectors 
$u_i, i\leq k$ (resp. $v_i, i\leq k$). For ease and clarity of exposition, we will 
consider $M \in \RR^{n\times n}$ that is symmetric, positive semi-definite 
in which case $u_i = v_i$ for all $i \in [d]$. Indeed, we can always reduce the problem of 
$k$-SVD for any matrix $M$ to that of solving the $k$-SVD of $M M^\top$ and $M^\top M$, which are both
symmetric, positive semi-definite.

\medskip
\noindent {\bf A bit of history.}
Singular Value Decomposition (SVD) is an essential tool in modern machine learning with applications spanning numerous fields such as biology, statistics, engineering, natural language processing, econometric and finance, etc (see \cite{chen2021spectral} for a non-exhaustive list of examples). It is a fundamental linear algebraic operation with a very rich history (see \cite{stewart1993early} for an early history). 

Traditional algorithms for performing SVD or related problems like principal component analysis, or eigenvalue and eigenvector computation, have mostly relied on iterative methods such as the Power method \cite{mises1929praktische}, Lanczos method \cite{lanczos1950iteration}, the QR algorithm \cite{francis1961qr}, or variations thereof for efficient and better numerical stability  \cite{cullum2002lanczos}. Notably, these methods have also been combined with stochastic approximation schemes to handle streaming and random access to data \cite{oja1985stochastic}. The rich history of these traditional algorithms still impacts the solutions of modern machine learning questions \cite{xu2018convergence, arora2012stochastic, hardt2014noisy, garber2015fast, shamir2015stochastic, garber2016faster, allen2016lazysvd, xu2018accelerated}.

\medskip
\noindent {\bf Why study a gradient-based method?} Given this rich history, one wonders why look for another method? Especially a gradient-based one. 
To start with, the $k$-SVD problem is typically formulated as a non-convex optimization problem \cite{auchmuty1989unconstrained} and the ability to understand success (or failure) of gradient-based methods for non-convex settings can expand our understanding for optimization in general. For example, the landscape of loss function associated with PCA has served as a natural nonconvex surrogate to understand landscape of solutions that arise in training for neural networks \cite{baldi1989neural}. It is also worth mentioning some of the recent works have been motivated by this very reason \cite{ge2015escaping, ge2016matrix, lee2016gradient, chi2019nonconvex}. Moreover, gradient-based methods are known to be robust with respect to noise or missing values \cite{hazan2016introduction}. For example, problems like matrix completion \cite{chen2020noisy} or phase retrieval \cite{candes2015phase} can be formulated as nonconvex matrix factorization problems and are typically solved using (stochastic) gradient-based methods (see \cite{chi2019nonconvex} and references therein). This is precisely what fueled the recent interest in understanding gradient methods for non-convex matrix factorization \cite{de2015global, jain2017global, stoger2021small, de2015global, jain2017global, stoger2021small, zhang2023preconditioned, jia2024preconditioning, li2024crucial}. Finally, the recent emergence of scalable compute coupled with software infrastructure for iterative optimization like gradient-based methods can naturally enable computation of $k$-SVD for large scale matrices in a seamless manner, and potentially overcome existing challenges with the scaling for SVD computation, see \cite{struski2024efficient}.

\medskip
\noindent {\bf Question of interest: $k$-SVD with gradient descent.} This work aims to develop a gradient descent method for $k$-SVD for any matrix. 
This is similar in spirit to earlier works devoted to developing gradient-based approaches for solving maximum eigen-pair problems \cite{auchmuty1989unconstrained, auchmuty1989variational, auchmuty1991globally, mongeau2004computing, gao2015barzilai, jiang2014unconstrained, shi2016limited}. These works proposed methods that, indeed, leverage gradient information, but their design is different from that of the standard 
gradient descent methods and their global convergence guarantees remain poorly understood.  
Recent works on gradient descent for non-convex matrix factorization \cite{de2015global, jain2017global, stoger2021small, de2015global, jain2017global, stoger2021small, zhang2023preconditioned, jia2024preconditioning, li2024crucial} are also useful 
for computing $k$-SVD; however, they fall short in doing so due to various limitations 
(see Table \ref{tab:prior_work} for a summary and discussion of related works in \textsection\ref{sec:related}). 
This has left the question of whether gradient descent can provably compute $k$-SVD for any matrix unresolved.

\begin{table}[h!]
    \caption{\small Here, we contextualize and compare our contributions to prior work, specifically, on computing $k$-SVD using gradient-descent for non-convex matrix factorization (i.e., minimizing $\Vert M - X X^\top \Vert_F^2$ over $X \in \RR^{n \times k}$ with $\rank(M)=d$).}\label{tab:prior_work}
      \begin{center}
    \resizebox{\columnwidth}{!}{
      \begin{tabular}{lcccccc}
        \toprule
        &  \multirow{2}{*}{\makecell{Parametrization \\ Regime}}  & \multirow{2}{*}{\makecell{Linear \\ Convergence}} & \multicolumn{2}{c}{Step-size Selection} & 
        \multirow{2}{*}{\makecell{Random \\ Initialization}}  & \multirow{2}{*}{\makecell{(Local)\\ Acceleration}}   \\
        \cmidrule(lr){4-5} 
        & & & Parameter-free & Pre-conditioning &  & \\
        \midrule
        \multirow{3}{*}{Prior work} & $(k > d)$ & \cite{ jiang2023algorithmic, stoger2021small, zhang2023preconditioned}  & \cite{zhang2023preconditioned} & \cite{zhang2023preconditioned, li2024crucial}  & \cite{li2024crucial} & --  \\
        \cmidrule(lr){2-7} 
         & $(k = d)$ &  \cite{tong2021accelerating, zhang2023preconditioned, jiang2023algorithmic, jia2024preconditioning, stoger2021small} & \cite{tong2021accelerating, zhang2023preconditioned, jia2024preconditioning}  & \cite{jia2024preconditioning,li2024crucial} & \cite{jia2024preconditioning,li2024crucial} & -- \\
         \cmidrule(lr){2-7}
          & ($k < d$) & \cite{chi2019nonconvex, jiang2023algorithmic} & -- & \cite{li2024crucial} &  \cite{li2024crucial} & -- \\
        \midrule
        \multirow{1}{*}{This work} &  (any $k$)  &  \multirow{1}{*}{\checkmark} & \multirow{1}{*}{\checkmark} & \multirow{1}{*}{\checkmark}  & \multirow{1}{*}{\checkmark} & \checkmark\\ 
        \bottomrule
      \end{tabular}%
    }
\end{center}
\end{table}

\medskip
\noindent {\bf Summary of Contributions.} The primary contribution of this work is to provide a gradient descent method for $k$-SVD for any given matrix. 
The method is parameter-free, enjoys global linear convergence, and works for any setting. In contrast to prior works, as summarized in Table \ref{tab:prior_work}, this is the first of such result for the {\em under-parametrized} setting, i.e. $k$ less than rank of matrix, including $k = 1$. 
The method can also be immediately enhanced by means of acceleration techniques such as that of Nesterov's. 
The accelerated method is algorithmically simple, yet enjoys an improved performance. 
Extensive empirical results corroborate our theoretical findings. 
The proof of the global linear convergence result is novel. 
Critical to the analysis is the observation that the gradient descent method for $k$-SVD is similar to the classical  
Heron's method
\footnote{To find the square root of number $a$, Heron's method consist in running the iterations $z_{t+1} = \frac{1}{2}(z_t + \frac{a}{z_t})$ starting 
with some initial point $z_1 > 0$. Heron's method is guaranteed to converge to $\sqrt{a}$ at a quadratic rate.} 
(a.k.a. Babylonian method) for root finding. 
This offers an interesting explanation to why pre-conditioning works. 
We believe this insight might shed further light in the study of gradient descent for 
generic non-convex loss landscapes.

\medskip
\noindent{\bf Organization.} The remainder of the manuscript is organized as follows. In Section \textsection\ref{sec:main}, we present the main results.
In Section \textsection\ref{sec:related}, we present related work. In Section \textsection\ref{sec:warmup}, we provide intuition explaining 
why gradient descent behaves like Heron's method for minimizing $\Vert M - x x^\top \Vert_F^2$. Section \textsection\ref{sec:proof} provides 
the proof of our main Theorem \ref{thm:convergence-1}. Section \textsection\ref{sec:experiments} presents empirical results to validate 
our theoretical findings. It should be noted that these empirical results are limited given the focus on this work being primarily theoretical in nature. Indeed, an important future direction is to develop a robust, scalable implementation of the method proposed in this work.


%% file: 02.arXiv_main_results.tex
\section{Main Results}\label{sec:main}

\medskip
\noindent{\bf Gradient descent for $k$-SVD.} Like the power method, the algorithm proceeds by sequentially finding one singular value and its corresponding vector at a time, in a decreasing order until all the $k \geq 1$ leading vectors are found. To find the top singular value and vector of $M$, we minimize the objective
\begin{align}\label{eq:objective}
    g(x; M) = \frac{1}{4} \Vert M - xx^\top \Vert_F^2.
\end{align}
Gradient-descent starts by randomly sampling an initial point $x_0 \in \RR^{n}$  as follows: 
\begin{align}\label{eq:random-init}
    x_0 = M x, \qquad \text{with} \qquad x \sim \cN(0, I_n), 
\end{align}
then updates for $t \geq 0$, 
\begin{align}\label{eq:gradient}
x_{t+1} & = x_{t} - \frac{\eta}{\Vert x_t \Vert^2} \nabla g (x_t; M), 
\end{align}
where $\eta \in (0,1)$. For the above algorithm, we establish the following:
\begin{theorem}\label{thm:convergence-1}
Let $\epsilon> 0$ and $M \in \RR^{n\times n}$ be a symmetric, positive semi-definite  with $\sigma_1 - \sigma_2 > 0$.  Running gradient descent iterations as described in \eqref{eq:gradient} with the choice $\eta=1/2$, ensures that for $t \ge 1$, {$ \vert \Vert x_t \Vert^2 - \sigma_1 \vert \le \epsilon \sigma_1$, $\Vert \Vert x_t \Vert^{-1} x_t - u_1  \Vert \wedge \Vert \Vert x_t \Vert^{-1} x_t + u_1  \Vert \le \epsilon$}, and  
$
\Vert x_{t} + \sqrt{\sigma_1} u_1 \Vert \wedge \Vert x_t - \sqrt{\sigma_1} u_1 \Vert \le \epsilon$
\footnote{The notation $\wedge$ is for $\max$ and $\vee$ is for $\min$.},
for any $\epsilon \in (0, 1)$, so long as 
{
\begin{align*}
    t \ge  \underbrace{\frac{c_1\sigma_1}{\sigma_1-\sigma_2} \log\left( \frac{e}{\epsilon} \frac{\sigma_1}{(\sigma_1 - \sigma_2) }   \right)}_{\textrm{number of iterations to converge within attracting region}} + \underbrace{c_2 \log\left( e \left(\frac{1}{\sigma_1} + \sigma_1\right)\right)}_{\textrm{number of iterations to reach attracting region}}. 
\end{align*}}
where $c_1, c_2$ are constants that only depend on the initial point $x_0$; { with the random initialization \eqref{eq:random-init}, the constants $c_1, c_2$ are almost surely strictly positive.} 
{\color{blue}}
\end{theorem}

\begin{figure}[h]
    \begin{subfigure}[t]{0.48\linewidth}
    \includesvg[width=1\linewidth]{fig_001_GD_traj.svg}
    \caption{ \emph{The Gradient Descent Trajectory.} We generate a random matrix $M$ of dimension $2 \times 2$ and visualize the trajectories of the iterates $(x_{t})_{t\ge0}$ when running the gradient descent iterations \eqref{eq:gradient} with specified initial points along the loss landscape $g$. }
    \label{fig:convergence-acceleration}
    \end{subfigure}\hfill  
   \begin{subfigure}[t]{.48\linewidth}
   \includesvg[width=1\linewidth]{fig_002_lower_env.svg}
    \caption{\emph{Gap vs. Number of Iterations.} Comparison between the power method and accelerated variants of gradient descent when minimizing $g$ for different values of the gap $\sigma_1 - \sigma_2$. See \textsection\ref{sec:experiments} for a detailed description of the experimental setup.}
    \label{fig:convergence-acceleration-2}
   \end{subfigure}
\end{figure}
The proof of Theorem \ref{thm:convergence-1} is presented in \textsection\ref{sec:proof}. In Figure \ref{fig:convergence-acceleration}, we illustrate the convergence result of the gradient descent method. 
{We remark that the condition $\sigma_1 - \sigma_2 > 0$ is not necessary for gradient descent \eqref{eq:gradient} to converge. We only require it to simplify the exposition of our results and  analysis. Indeed, when $\sigma_1 = \sigma_2 = \dots = \sigma_i$ and $\sigma_{i}- \sigma_{i+1}>0$ for some $i < d$, it can be shown that the iterations \eqref{eq:gradient} still converge to some $x_\star$ where $\Vert x_\star \Vert = \sqrt{\sigma_1}$, and $x_\star \in \mathrm{Span}\lbrace u_1, \dots, u_i\rbrace$ with a numer of iterations that is of order $\widetilde{\Omega}\left(\frac{\sigma_i}{\sigma_i - \sigma_{i+1}} \log\left(\frac{1}{\epsilon}\right)\right)$. Furthermore, the requirement that $M$ must be a symmetric, positive semi-definite matrix can be relaxed to that of $M$ being simply a symmetric semi-definite matrix with $\sigma_{1}(M) = \max_{i \in [d]} \vert \sigma_i \vert $. Indeed, our proofs naturally extend under this more general setting, and this requirement is only made to simplify the analysis.  
}

As a consequence of Theorem \ref{thm:convergence-1} it immediately follows that by sequential application of \eqref{eq:gradient}, we can recover all $k$ singular values and vectors (up to a desired precision). 


\begin{theorem}\label{thm:bigtheorem}
    Let $\epsilon > 0$ and $M \in \RR^{n\times n}$ be a symmetric, positive semi-definite matrix with {$\frac{\sigma_i - \sigma_{i+1}}{2 \sigma_i} \ge \epsilon$} for all $i \in [k]$. Sequential application of \eqref{eq:gradient} with $\eta = 1/2$ and  { the random initialization \eqref{eq:random-init}}, recovers $\hat{\sigma}_1, \dots, \hat{\sigma}_k$, and $\hat{u}_1, \dots, \hat{u}_k$ such that: $\vert \hat{\sigma}_i - \sigma_i \vert \le \epsilon \sigma_i  \quad \text{and} \quad \Vert \hat{u}_i + u_i\Vert \wedge \Vert \hat{u}_i - u_i \Vert \le \epsilon,
    $
    so long as the number of iterations, denoted $t$, per every application of \eqref{eq:gradient} satisfies:  
    {
    \begin{align*}
     t \ge C_1 k \max_{i \in [k]}\left(\frac{\sigma_i}{\sigma_i- {\sigma}_{i + 1}} \right) \log\left(\frac{\sigma_1}{\sigma_k} \max_{i \in [k]}\left(\frac{ {\sigma}_i}{(\sigma_i - \sigma_{i + 1}) } \right) \frac{k}{\epsilon}\right) + C_{2} \log\left( e\left( \sigma_1 + \frac{1}{\sigma_k} \right) \right),
    \end{align*}
    where $C_1$ and $C_2$ are constants that are almost surely strictly positive.
    }
\end{theorem}

{
The proof of Theorem \ref{thm:bigtheorem} is deferred to Appendix \ref{app:proof-big-thm}. We remark that the condition $\frac{\sigma_i - \sigma_{i+1}}{2\sigma_i} \ge \epsilon$ is an artifact of the requirement $\sigma_1 - \sigma_2 > 0$ from Theorem \ref{thm:convergence-1}, which is made to simplify the analysis. Moreover, we remark that for a non-symmetric matrix $N \in \RR^{n \times m}$ with SVD $N = U_N \Sigma_N V_N^\top$, computing it $k$-SVD, may be done by running the proposed $k$-SVD with gradient descent on the matrices $N N^\top = U_N \Sigma_N^2 U_N^\top$ and $N^\top N = V_N \Sigma^2 V_N^\top$. However, we may also resort to the \emph{dilation trick}, namely by defining 
$$\mathcal{H}(N) := \begin{bmatrix}
    0 & N \\
    N^\top & 0
\end{bmatrix} = \frac{1}{\sqrt{2}}\begin{bmatrix}
    U_N & U_N \\
    V_N & -V_N
\end{bmatrix} 2 \begin{bmatrix}
    \Sigma_N & 0 \\
    0 & - \Sigma_N
\end{bmatrix}  \frac{1}{\sqrt{2}}\begin{bmatrix}
    U_N & U_N \\
    V_N & -V_N
\end{bmatrix}^\top, 
$$
then running the proposed gradient descent method for minimizing the loss 
\begin{align*}
    \frac{1}{4}\left\Vert \mathcal{H}(N) - \begin{bmatrix}
        x \\ y
    \end{bmatrix} \begin{bmatrix}
        x \\ y
    \end{bmatrix}^\top  \right\Vert_F^2  = \frac{1}{2} \Vert N - x y^\top \Vert_F^2 + \frac{1}{4}( \Vert x \Vert^4 + \Vert y \Vert^4).
\end{align*}
For the above loss, $\mathcal{H}(N)$ is a symmetric, semi-definite matrix, and as highlighted before similar results to those of Theorem \ref{thm:convergence-1} hold. Thus, we may use $k$-SVD on this matrix instead.
}

\medskip
\noindent{\bf Accelerated gradient descent for $k$-SVD.} 
Gradient methods can achieve better performance when augmented with acceleration schemes \cite{nesterov2013introductory, d2021acceleration}. 
Building upon the known literature, we propose to accelerate \eqref{eq:gradient} as follows: 
{
\begin{equation}\label{eq:accelerated gd general}
\begin{aligned}
    y_t &  = x_t + \frac{\sqrt{\rho}}{1+ \sqrt{\rho}} (v_t - x_t)   \\ 
    x_{t+1} & = y_t - \frac{\eta}{\Vert y_t \Vert^2} \nabla g (y_t),   \\
    v_{t+1} & = \left(1- \frac{1}{\sqrt{\rho}}\right) v_t + \frac{1}{\sqrt{\rho}} \left(y_t - \frac{1}{\mu} \nabla g(y_t)\right), 
\end{aligned}
\end{equation}
where $\eta, \mu, \rho > 0$. This accelerated gradient descent method is an adaptation of the general scheme of the so-called optimal method proposed by Nesterov  \cite{nesterov2013introductory}[Chapter 2.2.1]. We establish in Theorem \ref{thm:local acceleration} below, that this scheme offers improved bounds then those obtained in Theorem \ref{thm:convergence-1}, at the expense of the convergence being only local.
}
\begin{theorem} \label{thm:local acceleration} Let $M \in \RR^{n \times n}$ be a symmetric, positive semi-definite matrix with $\sigma_1 - \sigma_2 > 0$. Assume that $x_0 \in \RR^n$, such that $\Vert x_0  \pm \sqrt{\sigma_1}u_1 \Vert \le (\sigma_1 - \sigma_2)^{3/2} /(90\sqrt{2}\sigma_1)$, $v_0 = x_0$. Running the accelerated gradient-descent iterations \eqref{eq:accelerated gd general}
with  $\eta = 1/6$, $\mu < (\sigma_1 - \sigma_2)/4$, $\rho = 9 \sigma_1/\mu $, ensures 
\begin{align}
    & \frac{1}{\sigma_1^{3/2}} \left\vert \frac{\Vert x_{t+1} \Vert}{\sqrt{\sigma_1}} - 1 \right\vert^2  \vee \sqrt{\sigma_1} \left\Vert \frac{x_{t+1}}{\Vert x_{t+1} \Vert} \pm  u_1   \right\Vert^2  \lesssim \left( 1 - \sqrt{\rho} \right)^t \wedge \left( \frac{\rho \sqrt{\mu}}{2 \sqrt{\rho} + t^2 }\right) 
\end{align}
\end{theorem}

The proof of Theorem \ref{thm:local acceleration} is deferred to Appendix \ref{app:acceleration}. 
{When one chooses $\mu$ of order $\Theta(\sigma_1 -\sigma_2)$, and sets $\rho = \Theta(\frac{\sigma_1}{\sigma_1 - \sigma_2})$, then in view Theorem \ref{thm:local acceleration}, the gradient descent scheme \eqref{eq:accelerated gd general} enjoys a convergence rate of order $\left(1-\Theta(\sqrt{\frac{\sigma_1}{\sigma_1 - \sigma_2}})\right)^t$ after $t$ iterations. Thus, to achieve an approximation error of order $\epsilon$, we require $\Omega\left( \sqrt{\frac{\sigma_1}{\sigma_1 - \sigma_2}} \log\left(\frac{1}{\epsilon}\right)\right)$ iterations. In contrast, the gradient descent iterations \eqref{eq:gradient}, in view of Theorem \ref{thm:convergence-1}, would require $\widetilde{\Omega}\left( \frac{\sigma_1}{\sigma_1 - \sigma_2} \log\left(\frac{1}{\epsilon}\right)\right)$ iterations. Finally, we also remark that the presented result shows a gap-independent rate which is similar in spirit to the results of \cite{musco2015randomized, allen2016lazysvd}. 


Indeed the accelerated gradient scheme \eqref{eq:accelerated gd general} requires to adequately choose $\rho$ and $\mu$ along with proper initialization to attain improved convergence rates. In that sense, Theorem \ref{thm:local acceleration} provides evidence of ability to achieve acceleration over vanilla gradient descent but requiring problem specific parameter setting. However, through empirical study, we find a simpler version \eqref{eq:accelerated gd general} of it that achieves similar acceleration. In this variation, starting with random initial point $x_0$ as per \eqref{eq:random-init}, the 
iterations follow 
\begin{equation}\label{eq:accelerated gd}
    \begin{aligned}
    x_{t+1} & = x_t + \beta(x_t - x_{t-1}) - \frac{\eta}{\Vert y_t \Vert^2} \nabla g(y_t; M)  \\
    y_t  & = x_t + \alpha (x_t - x_{t-1}) 
\end{aligned}
\end{equation}
where parameter $\eta, \beta \in (0,1)$ and $\alpha \in \lbrace0, \beta \rbrace$. We set $\alpha = \beta$ (resp. $\alpha = 0$), to recover a reminiscent version of Nesterov's acceleration \cite{nesterov2013introductory} (resp. Polyak's heavy ball method \cite{nemirovskiy1984iterative}) but with an adaptive step-size selection rule. Both methods are appealingly simple. Moreover, as demonstrated in Figure \ref{fig:gap-to-iterations}, both acceleration schemes yield a performance improvement from $\frac{\sigma_1}{\sigma_1 - \sigma_2}$ to $\sqrt{\frac{\sigma_1}{\sigma_1 - \sigma_2}}$ when $\beta$ is well chosen. This suggests that gradient descent with the acceleration schemes \eqref{eq:accelerated gd} is globally convergent, despite Theorem \ref{thm:local acceleration} only showing local convergence. 
}

%% file: 03.arXiv_related.tex
\section{Related Work}\label{sec:related}


Our work is primarily concerned with $k$-SVD computation which is of fundamental importance and has a very long history. 
Our results broadly relate to three lines of research reviewed below. 

\medskip
\noindent{\bf Methods from numerical linear algebra.} Computation of SVD is typically done via iterative methods 
that rely on one or more key algorithmic building blocks that includes power method \cite{mises1929praktische}, 
Lanczos' iterations \cite{lanczos1950iteration}, and the QR decomposition \cite{francis1961qr}. As is, these
algorithmic building blocks are not always numerically stable. This has led to a significant body of work to identify
stable, efficient algorithms. The stable, efficient variant of algorithms based on these building blocks typically 
transforms the matrix of interest to a bidiagonal matrix and then uses a variant of the QR algorithm, 
see \cite{demmel1990accurate} and \cite{cullum2002lanczos} for example, for an extensive literature overview. 
Such an algorithm, in some form or other, tries to identify singular vectors and values iteratively (either individually or in a block manner). The number of iterations taken by such an algorithm typically depends on the {\em gap} between singular values and the accuracy desired. 
In recent years, there have been {\em gap} independent (but still accuracy dependent) guarantees established 
when the desired accuracy is far above the {\em gap} \cite{musco2015randomized, allen2016lazysvd}. 
It is an understatement that such an algorithm is a workhorse of modern scientific computation and, more specifically, machine learning, cf. \cite{arora2012stochastic,  garber2015fast, shamir2015stochastic, garber2016faster}. The power method, despite not being optimized, is part of this workhorse due to its simplicity \cite{hardt2014noisy}. 
In this work, our objective is not necessarily to develop a method that is {\em better} than that known in literature. Instead, we seek to develop a fundamental understanding of the performance of a gradient based approach for the $k$-SVD problem. Subsequently, it will help understand the implications of acceleration methods on performance improvement. In the process, compare it with the rich prior literature and understand relative strengths and limitations. 






\medskip
\noindent{\bf Gradient descent and nonconvex optimization.} The convergence of gradient descent for 
nonconvex optimization has been studied extensively recently due to interest in empirical risk minimization in the context of large neural network or deep learning \cite{baldi1989neural, zhang2016riemannian, ge2015escaping, lee2016gradient, ge2016matrix, ge2017no, jin2017escape, du2018algorithmic}. 
In \cite{ge2015escaping}, authors established that the stochastic variant of gradient descent, the stochastic gradient descent (SGD), converges asymptotically to local minima and escapes strict saddle points for the tensor decomposition problem. 
In \cite{lee2016gradient}, it was further established that even the vanilla gradient descent converges to 
local minima provided that all saddle points are strict and the step-size is relatively small. 
An efficient procedure for saddle point escaping were also proposed in \cite{jin2017escape}. 
For eigenvalue problem, matrix completion, robust PCA, and matrix sensing tasks, \cite{auchmuty1989unconstrained, auchmuty1989variational, auchmuty1991globally, ge2016matrix,ge2017no} have highlighted that optimization landscapes 
have no spurious local minima, i.e., the objectives possesses only strict saddle points and all their local minima 
are also global. In \cite{balzano2010online, zhang2016riemannian}, authors established convergence of Riemannian 
or geometric gradient descent method.  

While all these works are quite relevant, they primarily provide asymptotic results, require the landscape of objective to satisfy strong conditions, or utilize complicated step-size selection rules where often knowledge of problem specific quantities is needed which are not obvious how to compute apriori. 
As we shall see, this work overcomes such limitations (see Theorem \ref{thm:convergence-1}).

\medskip
\noindent{\bf Matrix factorization as nonconvex optimization.} Matrix factorization can be viewed as a minimization of a nonconvex objective. Specifically, given a symmetric matrix $M$, to obtain rank $k$ approximation, the objective to minimize is  
$ \Vert M - XX^\top \Vert_F^2  $ for $X \in \RR^{n \times k}$. 
The parameterization $X X^\top$ has been often referred to as the Burrer-Monteiro matrix factorization \cite{burer2003nonlinear}. Recently, the study of gradient descent for solving this minimization problem has received a lot of interest, see \cite{candes2015phase, de2015global,jain2017global,chi2019nonconvex, chen2019gradient, tong2021accelerating, stoger2021small, zhang2023preconditioned, xu2023power, jiang2023algorithmic, jiang2014unconstrained, jia2024preconditioning, li2024crucial} and references therein. 
Specifically, progress has been made in three different regimes, namely, \emph{over-parameterized setting} when $k > d$, \emph{exact-parameterized setting} when $k = d$, and \emph{under-parameterized setting} when $k < d$. 
Linear convergence rates for gradient descent were initially established only locally for all regimes (see \cite{du2018algorithmic, chi2019nonconvex}. 
In \cite{stoger2021small}, it was shown that small random initialization is enough to ensure global linear convergence provided gradient descent uses a fixed but small enough step-size, which is problem instance dependent. 
Originally, this was restricted to the regime of $k \ge d$, and subsequently was extended to the regime of $k < d$ in \cite{jiang2014unconstrained}.  
In \cite{tong2021accelerating}, authors proposed using gradient descent with preconditioning, $X_{t+1} \gets X_t - \eta \nabla g(X_t; M) (X_t^\top X_t)^{-1}$, with $\eta > 0$ that is constant. They established linear convergence for $k = d$ with spectral initialization. But then it requires already knowing the SVD of the matrix!  
In \cite{zhang2023preconditioned} extended this result to $k > d$. In \cite{jia2024preconditioning}, authors established that preconditioning with random initialization is in fact sufficient to ensure global linear convergence, when $k = d$. In \cite{li2024crucial}, authors extend these results to $k > d$ and showed that even quadratic convergence rates can be 
achieved using the so-called Nyström initialization. 

\begin{table}[h!]
    \caption{\small Here, we contextualize and compare our contributions to prior work. Specifically, on computing $k$-SVD using gradient-descent for non-convex matrix factorization (i.e., minimizing $\Vert M - X X^\top \Vert_F^2$ over $X \in \RR^{n \times k}$).}\label{tab:prior_work_2}
      \begin{center}
        \resizebox{\columnwidth}{!}{
      \begin{tabular}{ccccccc}
        \toprule
        &  \multirow{2}{*}{\makecell{Parametrization \\ Regime}}  & \multirow{2}{*}{\makecell{Linear \\ Convergence}} & \multicolumn{2}{c}{Step-size Selection} & 
        \multirow{2}{*}{\makecell{Initialization}}  & \multirow{2}{*}{\makecell{(Local)\\ Acceleration}}   \\
        \cmidrule(lr){4-5} 
        & & & Parameter-free & Pre-conditioning &  & \\
        \midrule
        \multirow{1}{*}{Stoger et al. \cite{stoger2021small}} & \multirow{4}{*}{($k > d$)} & \checkmark  & --  & -- & \texttt{small-random-init}  & -- \\
        \multirow{1}{*}{Jia et al. \cite{jiang2023algorithmic}} &  & \checkmark  & 
 --  & -- & \texttt{small-random-init} &--  \\
        \multirow{1}{*}{Zhang et al. \cite{zhang2023preconditioned}} &  & \checkmark  &  --  & -- & \texttt{spectral-init} & --  \\
        \multirow{1}{*}{Li et al. \cite{li2024crucial}} &    & \checkmark  &  --  & \checkmark  & \textcolor{blue}{\texttt{random-init}}  & --\\
        \midrule
        \multirow{1}{*}{Zhang et al. \cite{zhang2023preconditioned}} & \multirow{6}{*}{($k = d$)}  & \checkmark  &  --  & -- & \texttt{spectral-init} & --  \\ 
        \multirow{1}{*}{Tong et al. \cite{tong2021accelerating}} &  & \checkmark  &  --  & \checkmark & \texttt{spectral-init} & --  \\
         \multirow{1}{*}{Stoger et al. \cite{stoger2021small}} &  & \checkmark  &  --  & --  & \texttt{small-random-init} & --  \\
        \multirow{1}{*}{Jia et al. \cite{jiang2023algorithmic}} &  & \checkmark  & 
 --  & -- & \texttt{small-random-init} & --  \\
        \multirow{1}{*}{Jia et al. \cite{jia2024preconditioning}} &  & \checkmark  & \checkmark  & \checkmark  & \textcolor{blue}{\texttt{random-init}} & --  \\ 
         \multirow{1}{*}{Li et al. \cite{li2024crucial}} &   & \checkmark  &  --  & \checkmark  & \textcolor{blue}{\texttt{random-init}}  & --\\

        \midrule
        \multirow{1}{*}{Chi et al. \cite{chi2019nonconvex}} & \multirow{3}{*}{$(k < d)$} & \checkmark  & --  & --  & \texttt{spectral-init} & --  \\
        \multirow{1}{*}{Jiang et al. \cite{jiang2023algorithmic}} &  & \checkmark  & 
 --  & -- & \texttt{small-random-init} &--  \\
        \multirow{1}{*}{Li et al. \cite{li2024crucial}} &  & --  &  --  & \checkmark  & \textcolor{blue}{\texttt{random-init}}  & --\\
        \midrule
        \multirow{1}{*}{This work} &  (any $k$)  &  \multirow{1}{*}{\checkmark} & \multirow{1}{*}{\checkmark} & \multirow{1}{*}{\checkmark}  & \multirow{1}{*}{\textcolor{blue}{\texttt{random-init}}} & \checkmark\\ 
        \bottomrule
      \end{tabular}%
    }
\end{center}
\end{table}

Despite all this progress, none of the aforementioned works consider the \emph{under-parameterized regimes}, except \cite{chi2019nonconvex, jiang2014unconstrained, li2024crucial}. Indeed, \cite{chi2019nonconvex} provided local convergence results for gradient descent with fixed step-size when $k = 1$, see \cite[Theorem 1]{chi2019nonconvex}. In \cite{jiang2014unconstrained}, global linear convergence was established but required small random initialization and fixed step-size when $k \le d$.  In \cite{li2024crucial}, authors considered gradient descent with preconditioning when $k < d$. However, they only showed sub-linear convergence, requiring  $O((1/\epsilon)\log(1/\epsilon))$ to find an $\epsilon$-optimal solution, see \cite[Theorem 2]{li2024crucial}. 
In contrast to the prior works, this work establishes global linear convergence, that is  gradient descent requires $O(\log(1/\epsilon))$ iterations to find an $\epsilon$-optimal solution, see Theorem \ref{thm:convergence-1} and Theorem \ref{thm:bigtheorem}.

\paragraph{Why study the under-parameterized setting?} While the \emph{under parameterized regime}, especially the case $k = 1$ is of fundamental importance as highlighted in \cite{chi2019nonconvex}, note that for $k > 1$, even if one finds an exact solution $X_\star$ that minimizes the objective $\Vert M - X X^\top \Vert_F^2$, then for any $k \times k$, orthogonal matrix $Q$, $X_\star Q$ also minimizes $\Vert M - X X^\top \Vert_F^2$. This rotation problem poses a challenge in using the objective $\Vert M - X X^\top \Vert_F^2$ for $k$-SVD while the objective $\Vert M - x x^\top \Vert_F^2$ doesn't. More importantly, being able to solve the \emph{under parameterized regime} with $k = 1$, allows us to perform $k$-SVD for any $k \ge 1$.





%% file: 04.arXiv_prelim.tex
\section{Building Intution with rank $1$ Matrix: Gradient descent is Heron's method}\label{sec:warmup}

A key insight from our analysis is the observation that gradient descent with adaptive step-size \eqref{eq:gradient} behaves like Heron's method. This is in stark contrast, with the observation of \cite{stoger2021small} suggesting that gradient descent for low-rank matrix factorization with a fixed step-size and small random initialization behaves like a power-method at an initial phase. 
To explain this and build intuition, we consider a simple but instructive setup where matrix $M$ has rank $1$. 

First, note that the gradient of the function $g$, as defined in \eqref{eq:objective}, at $x \in \RR^{n}$ is given by
\begin{align}
    \nabla g (x;M) = - M x + \Vert x \Vert^2 x. 
\end{align}
When $M$ is exactly of rank $1$, i.e., $M = \sigma_1 u_1 u_1^\top$, the gradient updates \eqref{eq:gradient} become: for all $t \ge 0$,
\begin{align}\label{eq:grad-iterations-ada-rank1}
    x_{t+1} = \frac{1}{2} \left(x_t +  \frac{ \sigma_1  u_1^\top x_t }{\Vert x_t \Vert^2} u_1\right)
\end{align}
We consider random initialization 
such that $x \sim \mathcal{N}(0, I_n)$ and 
\begin{align}\label{eq:init-grad}
    x_0 & = M x. 
\end{align}
Thus, we see that $x_0 = \sigma_1 (u_1^\top x) u_1$. This, together with the iterations \eqref{eq:grad-iterations-ada-rank1}
implies that for all $t \ge 0$,  $x_t = \Vert x_t \Vert u_1$. Hence, for all $t \ge 0$, we have
\begin{align}\label{eq:heron's iterates}
    \Vert x_{t+1} \Vert u_1 & = \frac{1}{2}\left( \Vert x_t \Vert +  \frac{\sigma_1}{\Vert x_t \Vert} \right) u_1.
\end{align}
We see then that $\Vert x_t \Vert$ is evolving precisely according to Heron's method (a.k.a. the Babylonian method) for finding the square root of $\sigma_1$. This leads to the following proposition.

\begin{proposition}\label{prop:grad convergence rank 1}
Let $M = \sigma_1 u_1 u_1^\top$. Under gradient descent as described in \eqref{eq:gradient} with an 
random initialization as in \eqref{eq:random-init},
$\|x_1\| > 0$ with probability $1$. 

Let $\epsilon_t = (\Vert x_t \Vert/\sqrt{\sigma_1} ) - 1$ for $t \geq 0$. 

Then for $t \geq 1$, $0 < \epsilon_{t+1} \le ( \epsilon_t^2 \wedge \epsilon_t)/2$. 

Subsequently, as $t\to\infty$, $\Vert x_t \pm \sqrt{\sigma_1} u_1 \Vert {\longrightarrow} 0$. 
\end{proposition}
The Proposition of \ref{prop:grad convergence rank 1} follows from basic algebra and precisely corresponds to 
the convergence analysis of Heron's method. 
For completeness, proof is provided in  Appendix \ref{ref:appA}. 
The established convergence guarantee indicates that gradient descent converges at a quadratic rate, i.e. 
in order to reach accuracy $\varepsilon$, the method requires $O(\log\log(1/\varepsilon))$ iterations.

It is worth mentioning that when the rank of $M$ is exactly 1, then the objective $g$ corresponds 
to that of an \emph{exact-parameterization} regime. The random initialization scheme \eqref{eq:init-grad} 
we consider is only meant for ease of exposition and has been studied in the concurrent work of \cite{li2024crucial}. 
Similar to this work, authors in \cite{li2024crucial} obtain quadratic convergence in this \emph{exact-parameterization} 
regime. 
Indeed, if one uses an alternative random initialization, say  $x_1 \sim \mathcal{N}(0, I_n)$, then one 
only obtains a linear convergence rate. Indeed, the goal of this work is to understand gradient descent for any matrix, not just
rank $1$. This is precisely done in \textsection\ref{sec:proof} building upon the above discussion.

%% file: 05.arXiv_analysis.tex
\section{Convergence analysis for General Matrix: Establishing Theorem \ref{thm:convergence-1}}\label{sec:proof}

In this section, we present the key ingredients for proving Theorem \ref{thm:convergence-1}. The proof strategy is broken into 
three steps: 
\emph{(i)}  establishing that gradient descent has a natural region of attraction and iterates reach that region; 
\emph{(ii)} once within this region of attraction the iterates $(x_t)_{t \ge 0}$ align with $u_1$ at a linear rate; and 
\emph{(iii)} as the iterate align with $u_1$, the sequence $(\Vert x_t \Vert)_{t \ge 0}$ evolves according to Heron's method for computing $\sqrt{\sigma_1}$ at a linear rate.

To that end, without loss of generality we consider that $x_0$ is randomly initialized as in \eqref{eq:random-init}. This choice of initialization allows us to simplify the analysis as it ensures that $x_1$ is in the image of $M$. Most importantly, our analysis only requires that $\PP( x_1^\top u_1 \neq 0) = 1$. It will be also convenient to introduce, for all $t \ge 0$, $i \in [d]$, the angle $\theta_{i,t}$ between $u_i$ and $x_t$, defined through   
\begin{align}
    \cos(\theta_{i,t}) = \frac{x_t ^\top u_i}{\Vert x_t \Vert}
\end{align}
whenever $\Vert x_t \Vert > 0$. 

\underline{\bf Step 1. Region of attraction.} One hopes that the iterates $x_t$ do not escape to infinity or vanish at $0$. It turns out that this is indeed the case and this is what we present in Lemma \ref{lem:attracting region}. We define: 

{
\begin{align}
    a & = 2\sqrt{\eta (1-\eta)} \left(\vert \cos(\theta_{1,0}) \vert \vee \sqrt{ \frac{\sigma_d}{\sigma_1}} \right) \label{eq:constant a}\\
    b & = 2 (1-\eta)  + \sqrt{\frac{\eta}{1-\eta}} \min\left(\frac{1}{\vert \cos(\theta_{1,0})\vert}, \sqrt{\frac{\sigma_1}{\sigma_d}} \right) \label{eq:constant b}
\end{align}
where it is not difficult to verify that $ 0 \le  a \le 1$ and $b > 1$ almost surely.
}

\begin{lemma}[Region of attraction]\label{lem:attracting region}
Gradient descent as described in \eqref{eq:gradient} with the random initialization \eqref{eq:random-init} 
ensures that the following holds almost surely: 
\begin{itemize}
    \item[(i)] $\forall t > 1,  a \sqrt{\sigma_1}  \le   \Vert x_{t} \Vert  $.
    \item[(ii)] $\forall t > \tau,  \Vert x_{t} \Vert \le b \, \sqrt{\sigma_1}$  where $\tau$ is given by 
\begin{align}
    \tau &= \left\lceil \frac{4}{3\eta}  \log\left( \frac{1}{2}  \left((1- \eta) \frac{\Vert x_0\Vert}{\sqrt{\sigma_1}} + \eta \frac{\sqrt{\sigma_1}}{\Vert x_0\Vert} \right)\right) \right\rceil \vee 1.
\end{align}
\item[(iii)] The sequence $(\vert \cos(\theta_{1,t})\vert )_{t\ge 0}$ is non-decreasing.
\end{itemize}
\end{lemma}
The proof of Lemma \ref{lem:attracting region} is given in Appendix \ref{sec:proof attracting region}. Interestingly, the constants $a$ and $b$ do not arbitrarily degrade with the quality of the initial random point. Thus, the number of iterations required to enter the region $[a \sqrt{\sigma_1}$, $b \sqrt{\sigma_1}]$ can be made constant if for instance one further constrains $\Vert x_1 \Vert = 1$. Furthermore, the fact that $\vert \cos(\theta_{1,t})\vert$ is non-decreasing is remarkable because it means that $ x_t / \Vert x_t \Vert$ can only get closer to $\lbrace -u_1 , u_1 \rbrace$. To see that, note that we have $\Vert  (x_t / \Vert x_t \Vert) \pm u_1 \Vert^2 = 2 \pm 2 \cos(\theta_{1, t}) \ge 2 (1 - \vert \cos(\theta_{1, t})\vert )$. An importance consequence of Lemma \ref{lem:attracting region} is that the gradient descent avoids saddle-point. 

{
\begin{lemma}[Avoiding saddle points]\label{lem:saddle-avoidance}
    Suppose that either $\Vert x_t \Vert^2 \le (\sigma_1 - \sigma_2)/2$ or for some $i \in \lbrace 2, \dots, d \rbrace$, $\vert \Vert x_t \Vert^2 - \sigma_i \vert^2 \le (\sigma_1 - \sigma_2)/2$. Then, gradient descent as described in \eqref{eq:gradient} with the random initialization \eqref{eq:random-init} ensures that 
    \begin{align}
     \Vert \nabla g(x_t; M) \Vert^2  \ge (a^2 \sigma_1 \wedge \Vert x_0\Vert^2)   \left(\frac{ (\sigma_1 - \sigma_{2} )^2}{4}\right) \vert \cos(\theta_{1,0}) \vert^2.   
    \end{align}
  
\end{lemma}
}
We provide a proof of Lemma \ref{lem:saddle-avoidance} in Appendix \ref{sec:proof saddle}.{ We remark that $\nabla g (x; M) = 0$ if  $x = 0$, or $x = \pm \sqrt{\sigma_i} u_i$ for some $i \in \lbrace 2, \dots, d\rbrace$. Therefore, Lemma \ref{lem:saddle-avoidance} entails that during the iteration of the algorithm, the gradient of the iterate cannot vanish if the iterate $x_t$ is close to any of these critical points. The only remaining critical points for which the gradient can vanish are the global optima.}  The lower bound depends on the initial point $x_0$ through $\vert \cos(\theta_{1,0})\vert$ which can be very small. In general, with the random initialization \eqref{eq:random-init}, small values of  $\vert \cos(\theta_{1,0})\vert$ are unlikely.

\underline{\bf Step 2. Alignment at linear rate.} Another key ingredient in the analysis is Lemma \ref{lem:conv singular vector} which we present below:

\begin{lemma}\label{lem:conv singular vector}
    Assume that $\sigma_1 - \sigma_2 > 0$ and let $\tau$ be defined as in Lemma \ref{lem:attracting region}. Gradient descent as described in \eqref{eq:gradient} with random initialization \eqref{eq:random-init} ensures that: $t \ge \tau$, $i \neq 1$, 
    {
    \begin{align}
        \left\Vert \frac{x_t}{\Vert x_t \Vert}  -   u_1  \right\Vert  \left\Vert \frac{x_t }{\Vert x_t \Vert} + u_1  \right\Vert & \le 2 \left(1 - \frac{ \eta (\sigma_1 - \sigma_2) }{ ((1-\eta) b^2 + \eta) \sigma_1 }   \right)^{t-\tau}, \\
        \vert \cos(\theta_{i,t})\vert  & \le \left(1 - \frac{ \eta (\sigma_1 - \sigma_i) }{ ((1-\eta) b^2 + \eta)\sigma_1}   \right)^{t-\tau}.
    \end{align}
    }
\end{lemma}
The proof of Lemma \ref{lem:conv singular vector} is given in  Appendix \ref{sec:proof conv singular vector}. The result shows that gradient descent is implicitly biased towards aligning its iterates with $u_1$. This alignment happens at a linear rate with a factor that depends on $\sigma_1 - \sigma_2$, and this is why we require the condition $\sigma_1 - \sigma_2 > 0$. 
If $\sigma_1 = \sigma_2 > \sigma_3$, then our proof can be 
adjusted to show that $x_t$ will align with the 
subspace spanned by $u_1, u_2$. Thus, our assumptions are without loss of generality.

\underline{\bf Step 3. Convergence with approximate Heron's iterations.} The final step of our analysis is to show that $\left\vert \Vert x_{t}\Vert / \sqrt{\sigma_1} - 1\right\vert$ vanishes at a linear rate which is presented in Lemma \ref{lem:conv singular value}. To establish this result, our proof strategy builds on the insight that the behavior of the iterates $(\Vert x_t \Vert)_{t\ge 0}$ resemble those of an approximate Heron's method. The result is presented when $\eta=1/2$ as this gives the cleanest proof. 
\begin{lemma}\label{lem:conv singular value} 
 Assume that $\sigma_1 - \sigma_2 > 0$. Gradient descent as described in \eqref{eq:gradient} with $\eta = 1/2$ and random initialization \eqref{eq:random-init} ensures that: for all $t > \tau_\star$, 

{
\begin{align}
     \left\vert \frac{\Vert x_{t+2} \Vert}{\sqrt{\sigma_1}} - 1 \right\vert \le \left( \frac{t -\tau_\star }{4} + \frac{2b}{3}\right) \left( \frac{2}{3}  \vee  \left(1 - \frac{\sigma_1 - \sigma_2}{(b^2 + 1)\sigma_1}\right)^2  \right)^{t-\tau_\star},
\end{align}
where 
\begin{align}
    \tau_\star = \left\lceil \frac{8}{3}  \log\left( \frac{1}{4}  \left( \frac{\Vert x_0\Vert}{\sqrt{\sigma_1}} + \frac{\sqrt{\sigma_1}}{\Vert x_0\Vert} \right)\right) \right\rceil \vee 1 + \frac{( b^2 + 1)\sigma_1}{2(\sigma_1 - \sigma_2)}   \log\left(\frac{2}{\vert \cos(\theta_{1,0}) \vert}\left( b + \frac{1}{a} \right)\right). 
\end{align}
}
\end{lemma}

\medskip
\noindent{\bf Proof sketch of Lemma \ref{lem:conv singular value}.} Taking the scalar product of the two sides of the gradient update equation \eqref{eq:gradient} with $u_1$, we obtain
\begin{align}\label{eq:1}
    \Vert x_{t+1} \Vert = \frac{1}{2}\left(\Vert x_t \Vert + \frac{\sigma_1}{\Vert x_t \Vert} \right) \frac{\vert \cos(\theta_{1, t})\vert}{ \vert \cos(\theta_{1, t+1})\vert }.
\end{align}
Next, we can take advantage of Lemma \ref{lem:conv singular vector} to show that the ratio $ \vert \cos(\theta_{1, t})\vert /  \vert \cos(\theta_{1, t+1})\vert$ converges to $1$ at a linear rate. 
This leads to the form of Heron's iterations \ref{eq:heron's iterates}. Starting with this form, we can show convergence of $\Vert x_t \Vert$ to $\sqrt{\sigma_1}$. 
We spare the reader the tedious details of this part and refer them to the complete proof given in Appendix \ref{sec:proof conv singular value}.

\underline{\bf Step 4. Putting everything together.} Below, we present a result which is an immediate consequence of  Lemma \ref{lem:attracting region}, Lemma \ref{lem:conv singular vector}, and Lemma \ref{lem:conv singular value}. 
\begin{theorem}\label{thm:convergence-2}
Assume that $\sigma_1 - \sigma_2 > 0$. Gradient descent as described in \eqref{eq:gradient} with $\eta = 1/2$ and random initialization \eqref{eq:random-init}
ensures that: for all $t > 0$,  

\begin{align}
     \left\Vert \frac{x_{t+\tau}}{\Vert x_{t+ \tau} \Vert}  -   u_1  \right\Vert  \wedge  \left\Vert \frac{x_{t+\tau}}{\Vert x_{t+\tau} \Vert}  +   u_1  \right\Vert & \le 2 \left(1 - \frac{  (\sigma_1 - \sigma_2) }{ (b^2 + 1) \sigma_1 }   \right)^{t}, \\
    \left\vert \frac{\Vert x_{t + \tau} \Vert}{\sqrt{\sigma_1}} - 1 \right\vert & \le \left( \sqrt{\frac{2}{3}}  \vee  \left(1 - \frac{(\sigma_1 - \sigma_2)}{(b^2 + 1)\sigma_1}\right)  \right)^{t}\\
    \Vert x_{t + \tau} - \sqrt{\sigma_1} u_1 \Vert \wedge \Vert x_{t+ \tau} + \sqrt{\sigma_1} u_1 \Vert,  & \le c_1 \sqrt{\sigma_1}  \left( \sqrt{\frac{2}{3}}  \vee  \left(1 - \frac{(\sigma_1 - \sigma_2)}{(b^2 + 1)\sigma_1}\right)  \right)^{t},
\end{align}
where  
\begin{align*}
    \tau :=  \frac{c_2 \sigma_1 }{\sigma_1 - \sigma_2}\log\left(  \frac{e \sigma_1 }{\sigma_1 - \sigma_2}\right)  + \underbrace{c_3 \log\left( \frac{e(\sigma_1^2 + 1) }{\sigma_1} \right)}_{\text{iterations to reach attracting region}}
\end{align*} 

with  positive constants $c_1, c_2, c_3$ that depend only on $x_0$; with random initialization, $c_1, c_2, c_3$ are strictly positive almost surely.
\end{theorem}
The proof is given in Appendix \ref{sec:proof thm1}. Theorem \ref{thm:convergence-1} is an immediate consequence of Theorem \ref{thm:convergence-2}.

%% file: 06.arXiv_experiments.tex
\section{Experiments}\label{sec:experiments}

We present here few experimental results illustrating the performance of k-SVD with gradient descent, denoted \texttt{GDSVD}. The method was implemented in both \texttt{C} and \texttt{Python}. We mainly compare with the \texttt{Power Method}. 

\subsection{Experimental Data} 

\noindent \textbf{Synthetic data.} 
We utilize synthetic data to evaluate various aspects of the method. This involves generating matrices of the form $M = U \Sigma V^\top$ where $U, V$ are sampled uniformly at random from the space of semi-orthogonal matrices of size $n \times d$, and $\Sigma = \mathrm{diag}(\sigma_1 ,\dots, \sigma_d)$ where $\sigma_1, \dots, \sigma_d$ being the singular values of $M$. The choice of 
number of non-zero singular values, i.e. $d \geq 1$ and their magnitudes are done in a few different ways to capture various types of matrices:
\begin{itemize}
    \item[] {\em Rank-$1$ matrices.} Here $d = 1$. These are used to inspect the performance of the methods with respect to the size of the matrix. The matrices were generated for up to size $10^5 \times 10^5$.
    
    \item[] {\em Rank-$2$ matrices.} Here $d = 2$. These are used to inspect the performance of the methods with respect to the gap $\sigma_1 - \sigma_2$. For these matrices, we set  $\sigma_1 = 1$ and vary $\sigma_1 - \sigma_2$ in $\lbrace 10^{-k/4}: k \in [20]\rbrace$.

    \item[] {\em Rank-$\lfloor \log(n) \rfloor$ matrices.} These are used to inspect the robustness of the compared methods with respect to the heterogeneity in the distribution of singular values. Specifically, three types of distribution are chosen.

\begin{itemize}
    \item[1.] Exponential Decay: Sample $a \sim \mathrm{Unif}(2, \dots, 10)$ and set $\sigma_i = a^{-i}$, for $i \in [d]$.
    \item[2.] Polynomial Decay: Set $\sigma_i = i^{-1} + 1$ for $i \in [d]$.
    \item[3.] Linear Decay: Sample $a \sim \mathrm{Unif}(1, \dots, 10)$, $b \sim \mathrm{Unif}([0,1])$, then set $\sigma_i = a - b i$ for $i \in [d]$.

\end{itemize}
\end{itemize}

\medskip
\noindent \textbf{Real-word matrices.} To evaluate performance on real-world data, we utilize matrices from datasets \texttt{MNIST} and \texttt{MovieLens(10K, 1M, 10M)}. 

\subsection{Implementation Details} 

\medskip \noindent{\bf \texttt{GDSVD}.} The implementation  \texttt{GDSVD} is done by sequentially applying gradient descent to find one singular value and corresponding vector at a time, until all $k \ge 1$ are found. 
For every run of gradient descent, the stopping conditions utilized is: given parameter $\epsilon > 0$, stop at iteration $t \geq 2$ if
\begin{align}
    \Vert \Vert x_t \Vert^{-1} x_{t} -  \Vert x_{t-1}\Vert^{-1} x_{t-1} \Vert & < \epsilon \nonumber \\ 
    \vert \Vert x_t \Vert - \Vert x_{t-1}\Vert \vert & < \epsilon. 
\end{align}
In all the experiments, we utilize $\epsilon = 10^{-8}$. A detailed pseudo-code for the implementation is provided in Algorithm \ref{alg:gdsvd} provided in Appendix \ref{app:experiments}. 

For acceleration method, we implement two approaches: the Polyak's and Nesterov's. The Polyak's acceleration \texttt{GDSVD(Polyak)} uses Polyak's momentum, i.e., the scheme \eqref{eq:accelerated gd}  with $\alpha = 0$. The acceleration is implemented beyond iteration $t > 100 \beta$ steps before adding momentum in order for the method to converge. The Nesterov's acceleration \texttt{GDSVD(Nesterov)} uses Nesterov's momentum, i.e., the scheme \eqref{eq:accelerated gd}  with $\beta = 0$. This method is robust in that the acceleration can be implemented starting from $t \geq 1$. For both \texttt{GDSVD(Polyak)} and \texttt{GDSVD(Nesterov)}, different values of $\beta$ were tested and the best was reported. 

\medskip \noindent {\bf Power iteration.} We implement a power method for finding the $k$-SVD of a symmetric $M$. 
As with the gradient based method, we also proceed sequentially finding one singular value at a time and corresponding vector at a time. 
The method proceeds by iterating  the equations $x_{t+1} =  \Vert M x_t \Vert^{-1} M x_t $ until convergence. 
The stopping conditions utilized: given parameter $\epsilon > 0$, stop at iteration $t \geq 2$ if
\begin{align}
    \Vert x_{t+1} -    x_t  \Vert  & < \epsilon \nonumber \\ 
    \vert \Vert M x_{t+1} \Vert - \Vert M x_{t} \Vert \vert & < \epsilon. 
\end{align}
The method is initialized in the same fashion as the one with gradient descent. See Algorithm \ref{alg:powersvd} in Appendix \ref{app:experiments} for a detailed pseudo-code.

\medskip\noindent{\bf Dealing with asymmetric matrices.} When considering asymmetric matrices $M$ in the experiments, for both \texttt{GDSVD} and \texttt{Power Method}, we apply $k$-SVD to $M M^\top$ to identify the leading $k$ singular values of $M$, $\hat{\sigma}_1, \dots, \hat{\sigma}_k$ and corresponding left singular vectors $\hat{u}_1, \dots, \hat{u}_k$, we then simply recover the right singular vectors by setting $\hat{v}_i = \hat{\sigma_i}^{-1} M^\top \hat{u}_i$, for $i \in [k]$.

\medskip\noindent{\bf Computation of gradients and matrix-vector multiplications.} In all implemented methods in \texttt{C}, the computation of gradients or more generally matrix-vector multiplications, were parallelized using CBLAS (\url{https://www.netlib.org/blas/}) and its default thread parameter settings, or with OpenCilk (\url{https://www.opencilk.org/}) for efficient memory management and for-loops parallelization.

\medskip\noindent{\bf Machine characteristics.} All experiments were carried on machines with a 6-core, 12-thread Intel chip, or a 10-core, 20-thread Intel chip. 

\subsection{Results} 

We start by summarizing the key findings of the experiments along with the supporting evidence from experimental results for these findings. They are as follows.

First, \texttt{GDSVD} is robust to different singular value decay distribution and scale gracefully with the matrix size $n$. This is inline with the theoretical results. This can be concluded from the results listed in Table \ref{tab:results} across various datasets. Also see Figure \ref{fig:gap-to-iterations}. While $\eta = \frac12$ is a good choice for \texttt{GDSVD}, to understand impact of 
$\eta \in (0,1)$ on the performance, as seen from Figure \ref{fig:gap-to-iterations},
\texttt{GDSVD} still works with different values of $\eta \in (0,1)$ despite our proof only holding for $\eta=1/2$. The choice of $\eta=1/2$ in our theoretical statement was made to simplify the exposition of our proof analysis, these experimental results reinforce our claims that the method should converge for all $\eta \in (0,1)$.

Second, the runtime of the implemented \texttt{GDSVD} is competitive if not faster than the \texttt{Power Method} which reaffirms the benefits of the gradient-descent approach for solving $k$-SVD. This can be concluded from the results listed in Table \ref{tab:results} across various datasets. 

{Third}, the scaling of run time for \texttt{GDSVD} scales linearly with the inverse of the gap, $\sigma_1 - \sigma_2$ as suggested by theoretical results. The Figure \ref{fig:gap-to-iterations} provides evidence for this. 

{Fourth}, the scaling of run time for accelerated version of \texttt{GDSVD} scales as square-root of the inverse of the gap, $\sigma_1 - \sigma_2$ as suggested by theoretical results. This confirms the utility of acceleration and ability for optimization based methods to achieve faster scaling inline with other methods. It is worth noting that the  \texttt{GDSVD(Nesterov)} is relatively better and more robust compared to \texttt{GDSVD(Polyak)}. See Figure \ref{fig:convergence-acceleration-2}  where this is clearly demonstrated. It is worth noting that while theoretical result in Theorem \ref{thm:local acceleration} only shows local convergence, the experimental results suggest that \texttt{GDSVD(Nesterov)} enjoys global convergence with improved performance.

\begin{table}[h!]
    \caption{ All the considered methods were tested on the rank-$\log(n)$ settings varying $n \in \lbrace 50, 75, 100, 200, \dots, 1000\rbrace$, as well as on the real-word datasets. The reported values correspond to the mean and standard deviation across different values of $n$ (resp. number of datasets) and number of threads used for parallelization of the methods for the  rank-$\log(n)$  (resp. real-world setting) settings. Here we report the runtime and the $k$-SVD recovery errors measured as $\epsilon_\Sigma = \Vert \Sigma - \hat{\Sigma}\Vert$ and $\epsilon_{U,V} = \max(\Vert UU^\top - \hat{U} \hat{U}^\top \Vert_F, \Vert VV^\top - \hat{V} \hat{V}^\top \Vert_F)$. For the rank-$\log(n)$ matrices, $k = \lfloor \log(n)\rfloor$. For the real-world matrices $k = 10$.}\label{tab:results} 
      \begin{center}
      \begin{tabular}{c c c c}
        \toprule
        \multirow{2}{*}{\textbf{Datasets}} & \multirow{1}{*}{\makecell{\textbf{Algorithms}
        }} &  \multicolumn{1}{c}{\makecell{\texttt{\textbf{GDSVD}} $(\eta = 0.5)$ }}  & \multicolumn{1}{c}{\makecell{\texttt{\textbf{Power Method}}}}   
          \\
         \cmidrule(lr){2-2}
         \cmidrule(lr){3-3} \cmidrule(lr){4-4} 
         & \makecell{Evaluations } & mean (std) & mean (std)   
                \\
        \midrule
          \multirow{3}{*}{Exponential} & Runtime  &  $1.884$	($4.626$)  &  
          $3.598$	($6.185$)  \\
           & $\epsilon_\Sigma$   & $1.9${\footnotesize x $10^{-13}$}($5.4${\footnotesize x$10^{-13}$}) &  
           $1.7${\footnotesize x $10^{-16}$}($1.6${\footnotesize x$10^{-16}$})   \\
           & $\epsilon_{U, V}$ & $2.8${\footnotesize x$10^{-6}$}	($9.0${\footnotesize x$10^{-6}$})  & 
           $3.4${\footnotesize x$10^{-6}$}	($9.2${\footnotesize x$10^{-6}$}) 
            \\
          \midrule
          \multirow{3}{*}{Polynomial}  & Runtime & 
          $5.400$	($13.42$)  & 
          $6.433$ ($12.61$)   \\
          & $\epsilon_\Sigma$   &  
          $2.9${\footnotesize x$10^{-16}$}	($1.1${\footnotesize x$10^{-16}$}) & 
          $2.3${\footnotesize x$10^{-16}$}	($7.3${\footnotesize x$10^{-17}$})\\
          & $\epsilon_{U, V}$  & 
          $6.1${\footnotesize x$10^{-08}$}	($9.4${\footnotesize x$10^{-09}$}) &
          $1.9${\footnotesize x$10^{-08}$}	($4.5${\footnotesize x$10^{-09}$}) \\

          \midrule  
          \multirow{3}{*}{Linear}  & Runtime  & 
          $10.23$ ($23.38$) &  
          $10.56$ ($21.01$)  \\
          & $\epsilon_\Sigma$  & 
          $1.4${\footnotesize x$10^{-14}$}	($2.0${\footnotesize x$10^{-14}$}) & 
          $4.5${\footnotesize x$10^{-15}$}	($5.1${\footnotesize x$10^{-15}$})  \\
          & $\epsilon_{U, V}$ &  
          $6.2${\footnotesize x$10^{-08}$}	($1.1${\footnotesize x$10^{-0.8}$}) & 
          $2.5${\footnotesize x$10^{-08}$}	($5.6${\footnotesize x$10^{-09}$})  \\  
          \midrule
          \multirow{3}{*}{Real-world}  & Runtime  & $227.2$	($509.6$) & 
          $227.2$  ($509.6$)  \\
          & $\epsilon_\Sigma$    & $1.8${\footnotesize x$10^{-05}$}	($3.1${\footnotesize x$10^{-05}$}) & 
          $1.8${\footnotesize x$10^{-05}$}	($3.1${\footnotesize x$10^{-05}$})
          \\
          & $\epsilon_{U, V}$    &  $2.1${\footnotesize x$10^{-07}$}	($5.3${\footnotesize x$10^{-08}$}) & 
           $1.0${\footnotesize x$10^{-07}$}	($4.0${\footnotesize x$10^{-08}$})  \\
        \bottomrule
      \end{tabular}%
\end{center}
\end{table}

\begin{figure}[ht!]
    \centering
    \begin{subfigure}[b]{0.7\linewidth}
    \includesvg[width=1\linewidth]{fig_004_gap_to_time.svg}
    \caption{\emph{Gap vs. Runtime.}}
    \label{fig:gap-to-iterations-1}
    \end{subfigure}%
    \\
    \begin{subfigure}[b]{.7\linewidth}
    \includesvg[width=1\linewidth]{fig_003_gap_to_iters.svg}
    \caption{\emph{Gap vs. Number of iterations.}}
    \label{fig:gap-to-iterations-2}
    \end{subfigure}
    \caption{Here, we illustrate the runtime and convergence performance of \texttt{GDSVD} in both \texttt{C} and \texttt{Python} as we vary the gap $\sigma_1 - \sigma_2$ in the rank-$2$ setting. The implementations of \texttt{GDSVD} with different values of $\eta \in (0,1)$ were compared with  \texttt{Power Method}. The curves were averaged over values of $n \in \lbrace 50, 100, 200, \dots, 1000 \rbrace$, and the shaded areas correspond to the standard deviations. The doted plots correspond to the lowest and uppermost performance over different values of $n$.}
    \label{fig:gap-to-iterations}
\end{figure}

%% file: 07.arXiv_conclusion.tex
\section{ Conclusion}\label{sec:discussion}

In this work, we presented a method for $k$-SVD with gradient descent. The  method uses an appealingly simple adaptive step-size towards optimizing a non-convex matrix factorization objective. We showed that the gradient method enjoys a global linear convergence guarantee with random initialization in spite of the nonconvexity of the optimization landscape. The method is shown to converge in the \emph{under-parameterized} setting, a result which to the best of our knowledge has eluded prior work. Critical to our convergence analysis is the observation that the gradient method behaves like Heron's method. We believe that this observation offers an interesting insight that might be of independent interest in the study of nonconvex optimization landscapes. Finally, we carried out numerical experiments that validate our theoretical findings. Through theory and experiments we have also demonstrated that proposed gradient-descent method can be immediately improved using acceleration schemes like Nesterov's. The resulting method is algorithmically simple which contrasts with the overly complicated methods that rely on Lanczos' iterations. An exciting direction for future work that our work leaves open is the global convergence of accelerated gradient descent for nonconvex matrix factorization.

Due to modern compute infrastructures that allow efficient  deployment of gradient methods at scale, we believe our method is likely to find use for computing $k$-SVD for massively large matrices.  Moreover, gradient methods are known to be robust; thus, we foresee that our method will find widespread applications and further development for settings where data are incomplete or noisy.

\section*{Acknowledgements}

The authors would like to acknowledge Emily Gan who helped with running some experiments in a  previous version of this work.

%% file: 09.arXiv_appB_analysis.tex
\newpage

\section{Global Convergence of Gradient Descent}\label{ref:appA}

In this section, we provide the detailed proofs of Proposition \ref{prop:grad convergence rank 1}, Theorem \ref{thm:convergence-1}, and Theorem \ref{thm:convergence-2}. 

The proof of Theorem \ref{prop:grad convergence rank 1} is self-contained and is provided in \textsection\ref{app:prop:gd-conv-rank1} and serves the purpose of building intuition as of why gradient descent behaves like Heron's method.  Theorem \ref{thm:convergence-1} is an immediate consequence of Theorem \ref{thm:convergence-2}, and both their proofs are given in \textsection\ref{sec:proof thm1}. 

The proof of Theorem \ref{thm:convergence-2} builds on the insight that gradient descent behaves like Heron's method which is captured in the proof of Lemma \ref{lem:conv singular value} given in \textsection\ref{sec:proof conv singular vector}. But before that, we need first to ensure that gradient descent has a region of attraction and this made precise in Lemma \ref{lem:attracting region} and its proof is given in \textsection\ref{sec:proof attracting region}. And secondly, we require the iterates of the gradient descent method to align with the leading singular vector which is established in Lemma \ref{lem:conv singular vector} with its proof given in \textsection\ref{sec:proof conv singular vector}.

\subsection{Proof of Proposition \ref{prop:grad convergence rank 1}}\label{app:prop:gd-conv-rank1}

\begin{proof}[Proof of Proposition \ref{prop:grad convergence rank 1}]
 First, by the random initialization \eqref{eq:random-init}, we see that $\Vert x_0 \Vert > 0$, and $x_0 \in \mathrm{Span}\lbrace u_1 \rbrace$ almost surely. Thus, in view of \eqref{eq:gradient}, for all $t \ge 0$, $\vert x_t^\top u_1 \vert = \Vert x_t \Vert $ almost surely. \\
 Next, projecting \eqref{eq:gradient} onto $u_1$, we have for all $t\ge0$, 
\begin{align*}
    x_{t+1}^\top u_1 = \left( (1- \eta)  + \eta \frac{\sigma_1}{\Vert x_t \Vert^2} \right) x_t^\top u_1 
\end{align*}
which leads to 
\begin{align*}
    \Vert x_{t+1} \Vert =   (1- \eta) \Vert x_t \Vert  + \eta \frac{\sigma_1}{\Vert x_t \Vert}. 
\end{align*}
In the above, we recognize the iterations of a Babylonian method (a.k.a. Heron's method) for finding the square root of a real number. For completeness, we provide the proof of convergence. Let us denote 
\begin{align*}
    \epsilon_t = \frac{\Vert x_t \Vert}{\sigma_1} - 1
\end{align*}
by replacing in the equation above we obtain that 
\begin{align*}
    \epsilon_{t+1} & = (1- \eta)(\epsilon_t + 1) + \frac{\eta}{(\epsilon_t + 1)} - 1 \\
    & = \frac{(1-\eta)\varepsilon_t^2 + (1- 2 \eta)\varepsilon_t}{\varepsilon_t + 1}.
\end{align*}
With the choice $\eta = 1/2$, we obtain that 
\begin{align*}
    \epsilon_{t+1} = \frac{\epsilon_t^2 }{2(\epsilon_t + 1)},
\end{align*}
from which we first conclude that for all $t\ge 1$, $\varepsilon_{t+1} > 0$ (note that we already have $\varepsilon_t > - 1 $). Thus we obtain 
\begin{align*}
    0 < \epsilon_{t+1} \le \frac{1}{2}\min\left( \epsilon_t^2, \epsilon_t  \right).
\end{align*}
This clearly implies that we have quadratic convergence. 
\end{proof}

\subsection{Proof of Theorem \ref{thm:convergence-1} and Theorem \ref{thm:convergence-2}}\label{sec:proof thm1}

Theorem \ref{thm:convergence-2} is an intermediate step in proving Theorem \ref{thm:convergence-1} and is therefore included in the proof below.

{
\begin{proof}[Proof of Theorem \ref{thm:convergence-2}.] Let $t\ge 0$. first, let us start by relating the distance $\left\Vert \frac{x_t}{\Vert x_t \Vert} -  u_1 \right\Vert^2$ to that of the two error $\left\vert \frac{\Vert x_t\Vert}{\sigma_1} - 1\right\vert$ and  $\left\Vert \frac{x_t}{\Vert x_t \Vert} - u_1  \right\Vert \wedge \left\Vert \frac{x_t}{\Vert x_t \Vert} + u_1  \right\Vert $. To that end, we start by noting that  
    \begin{align*}
        \left\Vert \frac{x_t}{\Vert x_t \Vert} -  u_1 \right\Vert^2 \left\Vert \frac{x_t}{\Vert x_t \Vert} +  u_1 \right\Vert^2 & = (2 - 2 \cos(\theta_{i,t}))(2 + 2 \cos(\theta_{i,t})) \\
        & \ge 2 \left(   \left\Vert \frac{x_t}{\Vert x_t \Vert} -  u_1 \right\Vert \wedge  \left\Vert \frac{x_t}{\Vert x_t \Vert} +  u_1 \right\Vert  \right)^2
    \end{align*}
    Thus, starting from an application of Lemma \ref{lem:error-decomposition}, we have
    \begin{align*}
        \Vert x_t - \sqrt{\sigma_1} u_1 \Vert^2 \wedge \Vert x_t + \sqrt{\sigma_1} u_1 \Vert^2  & = \sigma_1 \left(\left( \frac{\Vert x_t \Vert}{\sqrt{\sigma_1}} - 1 \right)^2  + \frac{\Vert x_t \Vert}{\sqrt{\sigma_1}} \left(   \left\Vert \frac{x_t}{\Vert x_t \Vert} -  u_1 \right\Vert \wedge  \left\Vert \frac{x_t}{\Vert x_t \Vert} +  u_1 \right\Vert  \right)^2 \right) \\
        & \le \sigma_1 \left(\left( \frac{\Vert x_t \Vert}{\sqrt{\sigma_1}} - 1 \right)^2  + \frac{\Vert x_t \Vert}{2\sqrt{\sigma_1}}   \left\Vert \frac{x_t}{\Vert x_t \Vert} -  u_1 \right\Vert^2 \left\Vert \frac{x_t}{\Vert x_t \Vert} +  u_1 \right\Vert^2 \right).
    \end{align*}
Next, all that remains is to invoke Lemma \ref{lem:attracting region}, Lemma \ref{lem:conv singular vector}, and Lemma \ref{lem:conv singular value} to conclude. By application of Lemma \ref{lem:attracting region}, for all $t\ge \tau$, we have 
\begin{align}
    \frac{\Vert x_t \Vert}{2\sqrt{\sigma_1}} \le \frac{b}{2}. 
\end{align}
By application of Lemma \ref{lem:conv singular vector}, we have for all $t \ge \tau$
\begin{align}
    \left\Vert \frac{x_t}{\Vert x_t \Vert}  -   u_1  \right\Vert^2  \left\Vert \frac{x_t }{\Vert x_t \Vert} + u_1  \right\Vert^2 & \le 4 \left(1 - \frac{  (\sigma_1 - \sigma_2) }{ (b^2 + 1) \sigma_1 }   \right)^{2(t-\tau)},
\end{align}
and since 
\begin{align*}
     \left\Vert \frac{x_t}{\Vert x_t \Vert}  -   u_1  \right\Vert^2  \wedge  \left\Vert \frac{x_t}{\Vert x_t \Vert}  +   u_1  \right\Vert^2  \le  \frac{1}{2}  \left\Vert \frac{x_t}{\Vert x_t \Vert}  -   u_1  \right\Vert^2  \left\Vert \frac{x_t }{\Vert x_t \Vert} + u_1  \right\Vert^2,
\end{align*}
we further have that 
\begin{align*}
     \left\Vert \frac{x_t}{\Vert x_t \Vert}  -   u_1  \right\Vert^2  \wedge  \left\Vert \frac{x_t}{\Vert x_t \Vert}  +   u_1  \right\Vert^2 \le 2 \left(1 - \frac{  (\sigma_1 - \sigma_2) }{ (b^2 + 1) \sigma_1 }   \right)^{2(t-\tau)},
\end{align*}
By application of Lemma \ref{lem:conv singular value}, for all $t \ge \tau_\star + 2$,
\begin{align}
     \left\vert \frac{\Vert x_{t} \Vert}{\sqrt{\sigma_1}} - 1 \right\vert^2 & \le \left( \frac{t -\tau_\star - 2 }{4} + \frac{2b}{3}\right)^2 \left( \sqrt{\frac{2}{3}}  \vee  \left(1 - \frac{(\sigma_1 - \sigma_2)}{(b^2 + 1)\sigma_1}\right)  \right)^{4(t-\tau_\star - 2)}, \\
     & \le \left( \sqrt{\frac{2}{3}}  \vee  \left(1 - \frac{(\sigma_1 - \sigma_2)}{(b^2 + 1)\sigma_1}\right)  \right)^{2(t-\tau_\star - 2)},
\end{align}
provided that, 
\begin{align}
    \left( \frac{t -\tau_\star - 2 }{4} + \frac{2b}{3}\right) \left( \sqrt{\frac{2}{3}}  \vee  \left(1 - \frac{(\sigma_1 - \sigma_2)}{(b^2 + 1)\sigma_1}\right)  \right)^{(t-\tau_\star - 2)} \le 1
\end{align}
which is true, so long as 
\begin{align}
    t- \tau_\star - 2  \ge \tau_3  & := \frac{2}{ \frac{1}{3 + \sqrt{5}} \wedge \frac{\sigma_1 - \sigma_2 }{(b^2 + 1) \sigma_1}} \log\left(  \frac{3}{4 b} \left(\frac{1} {\frac{1}{3 + \sqrt{5}} \wedge \frac{\sigma_1 - \sigma_2 }{(b^2 + 1) \sigma_1}}  \right) \right) \\
    & = 2 \left( (3 + \sqrt{5} ) \vee \frac{(b^2 + 1)\sigma_1}{\sigma_1 - \sigma_2} \right) \log\left( \frac{2}{4 b}\left( (3 + \sqrt{5} ) \vee \frac{(b^2 + 1)\sigma_1}{\sigma_1 - \sigma_2} \right) \right)
\end{align}
Therefore, we obtain 
\begin{align}\label{eq:thm4:result}
      \Vert x_t - \sqrt{\sigma_1} u_1 \Vert^2 \wedge \Vert x_t + \sqrt{\sigma_1} u_1 \Vert^2  \le 3 b \sigma_1  \left( \sqrt{\frac{2}{3}}  \vee  \left(1 - \frac{(\sigma_1 - \sigma_2)}{(b^2 + 1)\sigma_1}\right)  \right)^{2(t-\tau_\star - 2)}
\end{align}
so long as $t \ge \tau_\star + 2 + \tau_3$. Note that we have 
\begin{align}
    \tau_\star + 2 + \tau_3  & =  \frac{8}{3} \log\left( \frac{e}{4}\left( \frac{\Vert x_0 \Vert}{\sqrt{\sigma_1}} + \frac{\sqrt{\sigma_1}}{\Vert x_0\Vert}  \right) \right) \\
    & + \frac{(b^2 + 1) \sigma_1}{\sigma_1 - \sigma_2} \log\left( \frac{2}{\vert \cos(\theta_{1, 0})\vert} \left( b + \frac{1}{a} \right) \right) + 2 \\ 
    & + 2 \left( (3 + \sqrt{5} ) \vee \frac{(b^2 + 1)\sigma_1}{\sigma_1 - \sigma_2} \right) \log\left( \frac{2}{4 b}\left( (3 + \sqrt{5} ) \vee \frac{(b^2 + 1)\sigma_1}{\sigma_1 - \sigma_2} \right) \right) \\  
    & \le \frac{8}{3} \log\left( \frac{e}{4}\left( \frac{\Vert x_0 \Vert}{\sqrt{\sigma_1}} + \frac{\sqrt{\sigma_1}}{\Vert x_0\Vert}  \right) \right) \\
    & + \frac{(b^2 + 1) \sigma_1}{\sigma_1 - \sigma_2} \log\left( \frac{18}{\vert \cos(\theta_{1, 0})\vert} \left( b + \frac{1}{a} \right) \right) \\ 
    & + 2 \left( (3 + \sqrt{5} ) \vee \frac{(b^2 + 1)\sigma_1}{\sigma_1 - \sigma_2} \right) \log\left( \left( (3 + \sqrt{5} ) \vee \frac{(b^2 + 1)\sigma_1}{\sigma_1 - \sigma_2} \right) \right) \\
    & \le \frac{8}{3} \log\left( \frac{e}{4}\left( \frac{\Vert x_0 \Vert}{\sqrt{\sigma_1}} + \frac{\sqrt{\sigma_1}}{\Vert x_0\Vert}  \right) \right)  \\
    & + 2 \left( (3 + \sqrt{5} ) \vee \frac{(b^2 + 1)\sigma_1}{\sigma_1 - \sigma_2} \right) \log\left(  \frac{8}{\vert \cos(\theta_{1,0})\vert} \left(b + \frac{1}{a}\right)  \left( (3 + \sqrt{5} ) \vee \frac{(b^2 + 1)\sigma_1}{\sigma_1 - \sigma_2} \right) \right)
\end{align}
Thus, we have just shown that \eqref{eq:thm4:result} still holds, when  
\begin{align*}
    t \ge \tau := \frac{c_1 \sigma_1 }{\sigma_1 - \sigma_2}\log\left(  \frac{e \sigma_1 }{\sigma_1 - \sigma_2}\right)  + \underbrace{c_2 \log\left( \frac{e(\sigma_1^2 + 1) }{\sigma_1} \right)}_{\text{iterations to reach attracting region}} 
\end{align*}
where $c_1, c_2$ are the corresponding constant that only depend on $x_0$, through $\cos(\theta_{1,0})$.
This concludes the proof.
\end{proof}

\begin{proof}[Proof of Theorem \ref{thm:convergence-1}.] The proof is an immediate application Theorem \ref{thm:convergence-2}. Indeed, Theorem \ref{thm:convergence-2} entails that 
for all $t \ge 0$, 
\begin{align*}
          \Vert x_{t+\tau} - \sqrt{\sigma_1} u_1 \Vert \wedge \Vert x_{t+\tau} + \sqrt{\sigma_1} u_1 \Vert  & \le \sqrt{3 b \sigma_1}  \left( \sqrt{\frac{2}{3}}  \vee  \left(1 - \frac{(\sigma_1 - \sigma_2)}{(b^2 + 1)\sigma_1}\right)  \right)^{t} \\
          & \le \sqrt{3b\sigma_1} \exp\left( - \left( \frac{1}{3 + \sqrt{6}} \wedge  \frac{\sigma_1 - \sigma_2 }{(b^2  +1) \sigma_1} \right) t \right),
\end{align*}
with $$
\tau = \frac{c_1 \sigma_1 }{\sigma_1 - \sigma_2}\log\left(  \frac{e \sigma_1 }{\sigma_1 - \sigma_2}\right)  + c_2 \log\left( \frac{e(\sigma_1^2 + 1)}{\sigma_1} \right).
$$
Let $\epsilon > 0$. We have 
\begin{align*}
     \Vert x_{t+\tau} - \sqrt{\sigma_1} u_1 \Vert \wedge \Vert x_{t+\tau} + \sqrt{\sigma_1} u_1 \Vert  \le \epsilon. 
\end{align*}
whenever, it also holds that 
\begin{align*}
   \sqrt{3b\sigma_1} \exp\left( - \left( \frac{1}{3 + \sqrt{6}} \wedge  \frac{\sigma_1 - \sigma_2 }{(b^2  +1) \sigma_1} \right) t \right) \le \epsilon
\end{align*}
which is equivalent to the condition 
\begin{align*}
    t \ge \left((3+\sqrt{6}) \vee \frac{(b^2+1)\sigma_1}{\sigma_1 - \sigma_2}\right)\log\left( \frac{\sqrt{3b\sigma_1}}{\epsilon}\right).
\end{align*}
Therefore, for all $t' \ge 0$,
\begin{align*}
     \Vert x_{t'+\tau_\star} - \sqrt{\sigma_1} u_1 \Vert \wedge \Vert x_{t'+\tau_\star} + \sqrt{\sigma_1} u_1 \Vert  \le \epsilon. 
\end{align*}
provided 
\begin{align*}
    \tau_\star = \frac{c_1'\sigma_1}{\sigma_1-\sigma_2} \log\left(\frac{e \sigma_1}{(\sigma_1 - \sigma_2) \epsilon}  \right) + c_2 \log\left( \frac{e(\sigma_1^2 + 1)}{\sigma_1} \right). 
\end{align*}
Following a similar reasoning, we also obtain from Theorem \ref{thm:convergence-2}, that under the same condition above, we have 
\begin{align*}
    \left\Vert  \frac{x_{t' + \tau_\star}}{\Vert x_{t' + \tau_\star}\Vert } - u_1 \right\Vert \wedge  \left\Vert  \frac{x_{t' + \tau_\star}}{\Vert x_{t' + \tau}\Vert } + u_1 \right\Vert  & \le \epsilon \\
    \left\vert \frac{\Vert x_{t' + \tau_\star} \Vert}{\sqrt{\sigma_1}} - 1 \right\vert & \le \epsilon. 
\end{align*}
Finally, using the elementary fact that  $\vert z - 1\vert \le \epsilon$, implies $\vert z^2 - 1 \vert \le 3 \epsilon$  for all $z \in \RR$ and $\epsilon \in (0,1)$,  leads to the desired results.
\end{proof}
}

\subsection{Proof of Lemma \ref{lem:attracting region}}\label{sec:proof attracting region}

{
\begin{proof}[Proof of Lemma \ref{lem:attracting region}.]
First, we note that proposed random initialization \eqref{eq:random-init} ensures that the event $\lbrace x_0 \not\in\mathrm{Ker}(M) \rbrace$ holds almost surely. Thus, the condition to apply Lemma \ref{lem:aux1}, Lemma \ref{lem:aux2}, Lemma \ref{lem:aux3}, and Lemma \ref{lem:aux4} in an almost sure sense hold. These lemmas are the key steps we use to establish the desired statement. \\
\underline{\textbf{Step 1.}} We have by Lemma \ref{lem:aux2}, that almost surely, for all $t\ge 1$, 
\begin{align}\label{eq:lem3:eq1}
    \Vert x_t \Vert \ge a \sqrt{\sigma_1}. 
\end{align}
where $a$ is as defined in Lemma \ref{lem:aux2}. \\
\underline{\textbf{Step 2.}} If $\Vert x_1 \Vert \le b \sqrt{\sigma_1}$, then  in view of \eqref{eq:lem3:eq1}, we have $ a \sqrt{\sigma_1} \le \Vert x_1 \Vert \le b\sqrt{\sigma_1}$. Thus, by Lemma \ref{lem:aux3} and using induction we have almost surely that for all $t \ge 1$, 
\begin{align}
    a \sqrt{\sigma_1} \le \Vert x_t \Vert \le b \sqrt{\sigma_1}.
\end{align}
where $b$ is as defined in Lemma \ref{lem:aux3}. \\
\underline{\textbf{Step 3.}} If $\Vert x_1 \Vert > b \sqrt{\sigma_1}$, then by Lemma \ref{lem:aux4}, we have almost surely that for all $t \ge \tau_0$,
\begin{align}
    a \sqrt{\sigma_1} \le \Vert x_t \Vert \le b \sqrt{\sigma_1}.
\end{align}
with $\tau_0  = \left\lceil \frac{\log\left(\frac{\Vert x_1 \Vert}{b \sqrt{\sigma_1}}\right)}{\log\left(\frac{1}{ (1-\eta) + \frac{\eta}{4}}\right)} \right\rceil \vee 1.$ Thus, we have just shown that almost surely for all $t\ge \tau_0$, 
\begin{align}
    a \sqrt{\sigma_1} \le \Vert x_t \Vert \le b \sqrt{\sigma_1}.
\end{align}
\underline{\textbf{Step 4. Simplifying the statement.}} At this stage we are almost done with the proof. Let us further provide an upper bound $\tau_0$ which simplifies the statement of the result. We have
\begin{align*}
    \tau_0 & = \left\lceil \frac{\log\left(\frac{\Vert x_1 \Vert}{b \sqrt{\sigma_1}}\right)}{\log\left(\frac{1}{ (1-\eta) + \frac{\eta}{4}}\right)} \right\rceil \vee 1 \\
    & \overset{(a)}{\le} \left\lceil   \frac{4}{3\eta} \log\left(\frac{\Vert x_1 \Vert}{b \sqrt{\sigma_1}}\right) \right\rceil \vee 1 \\ 
    & \overset{(b)}{\le} \left\lceil \frac{4}{3\eta}  \log\left(\frac{(1- \eta) \Vert x_0\Vert + \eta \frac{\sigma_1}{\Vert x_0\Vert}}{b\sqrt{\sigma_1}} \right) \right\rceil \vee 1 \\
    & = \left\lceil \frac{4}{3\eta}  \log\left( \frac{1}{b}  \left((1- \eta) \frac{\Vert x_0\Vert}{\sqrt{\sigma_1}} + \eta \frac{\sqrt{\sigma_1}}{\Vert x_0\Vert} \right)\right) \right\rceil \vee 1 \\
    & \overset{(c)}{\le} \underbrace{\left\lceil \frac{4}{3\eta}  \log\left( \frac{1}{2}  \left((1- \eta) \frac{\Vert x_0\Vert}{\sqrt{\sigma_1}} + \eta \frac{\sqrt{\sigma_1}}{\Vert x_0\Vert} \right)\right) \right\rceil \vee 1}_{\tau := } 
\end{align*}
where in inequality $(a)$, we use the fact that $\log((1-\eta) + \frac{\eta}{4}) = \log(1- \frac{3\eta}{4}) \le -\frac{3\eta}{4}$, to obtain inequality $(b)$, we used Lemma \ref{lem:aux1} (see statement $(iii)$) to upper bound $\Vert x_1\Vert$, and to obtain inequality $(c)$ we used the fact that $b \ge 2$ which follows from Lemma \ref{lem:aux3}. Setting $\tau = \left\lceil \frac{4}{3\eta}  \log\left( \frac{1}{2}  \left((1- \eta) \frac{\Vert x_0\Vert}{\sqrt{\sigma_1}} + \eta \frac{\sqrt{\sigma_1}}{\Vert x_0\Vert} \right)\right) \right\rceil \vee 1$, it also holds in particular that almost surely for all $t \ge \tau$, $a \sqrt{\sigma_1} \le \Vert x_{t}\Vert \le b\sqrt{\sigma_1}$. 

Finally, the fact that $\vert \cos(\theta_{1,t})\vert_{t \ge 0} $ is non-increasing follows from Lemma \ref{lem:aux1}. This concludes the proof.
\end{proof}

\begin{lemma}\label{lem:aux1}
   Suppose that $x_0 = M x$ and $ x \not\in \mathrm{Ker}(M)$, then for all $t \ge 0$, we have
   \begin{itemize}
    \item [(i)] $\Vert x_t \Vert > 0$ 
    \item [(ii)] $x_t \in \mathrm{Span}\lbrace u_1, \dots, u_d \rbrace$
    \item [(iii)]  $\left((1-\eta) + \eta \frac{\sigma_d}{\Vert x_t \Vert^2}\right) \Vert x_t \Vert  \le \Vert x_{t+1} \Vert \le  \left( (1-\eta) + \eta \frac{\sigma_1}{\Vert x_t \Vert^2}\right) \Vert x_t \Vert$
   \end{itemize}
\end{lemma}

\begin{proof}[Proof of Lemma \ref{lem:aux1}.]
The first step of our proof is to establish the claim that for all $t \ge 0 $, if $(i)$ and $(ii)$ are true then $(iii)$ holds. To that end, let $t \ge 0$ and suppose that both $\Vert x_t \Vert > 0$ and $ x_t \in \mathrm{Span}\lbrace u_1, \dots, u_d \rbrace$ hold true. Starting from \eqref{eq:gradient} we have that  
\begin{align}
    x_{t+1} & \overset{(a)}{=} (1- \eta) x_t + \eta \frac{ M x_t }{\Vert x_t \Vert^2} \\
    & \overset{(b)}{=}  \sum_{i =1}^d \left((1-\eta) + \eta \frac{\sigma_i}{\Vert x_t \Vert^2}\right) (u_i^\top x_t) u_i,  \label{eq:lem:1}
\end{align}
where $(a)$ follows by substituting $\nabla g(x_t;M) = - Mx_t + \Vert x_t \Vert^2 x_t$ in \eqref{eq:gradient}, and $(b)$ follows by substituting with $M = \sum_{i=1}^d \sigma_i u_i u_i^\top$ and using the fact $x_t \in \mathrm{Span}\lbrace u_1, \dots ,u_d\rbrace$ to substitute with $x_t = \sum_{i=1}^d (u_i^\top x_t) u_i$. From \eqref{eq:lem:1}, we see that $x_{t+1} \in \mathrm{Span}\lbrace u_1, \dots, u_d\rbrace$, thus
\begin{align}
    \Vert x_{t+1} \Vert^2 & = \sum_{i=1}^d \left( (1 - \eta) + \eta \frac{\sigma_i}{\Vert x_t \Vert^2} \right)^2(u_i^\top x_t)^2.  \label{eq:lem:2} \\
    & \overset{(c)}{\ge}  \left( (1 - \eta) + \eta \frac{\sigma_d}{\Vert x_t \Vert^2} \right)^2 \sum_{i=1}^d (u_i^\top x_t)^2 \\
    & \overset{(d)}{=}  \left( (1 - \eta) + \eta \frac{\sigma_d}{\Vert x_t \Vert^2} \right)^2 \Vert x_t \Vert^2, 
\end{align}
where the inequality $(c)$ follows from the fact that $\sigma_1 \ge \cdots \ge \sigma_d$, and the equality $(d)$ follows from the fact that $ x_t \in \mathrm{Span}\lbrace u_1, \dots, u_d \rbrace$, which implies that $\Vert x_t \Vert^2 = \sum_{i=1}^d (u_i^\top x_t)^2$. In a similar fashion, we also have 
\begin{align}
    \Vert x_{t+1} \Vert^2 & \le  \left( (1 - \eta) + \eta \frac{\sigma_1}{\Vert x_t \Vert^2} \right)^2 \Vert x_t \Vert^2. 
\end{align}
This concludes the proof of the claim and we are now ready to complete the proof. To that end, observe that it is sufficient to show that that for all $t \ge 0$, \emph{(i)} and \emph{(ii)} hold to complete the proof. We will proceed by induction. \\
\underline{Base Case. $t = 0$.} First, because $x \not \in \mathrm{Ker}(M)$, we see that $ \Vert x_0 \Vert = \Vert M x\Vert > 0$, thus \emph{(i)} is verified. Next, because $x_0 = Mx$ and $M = \sum_{i=1}^d \sigma_i u_i u_i^\top$, it follows that $x_0 \in \mathrm{Span}\lbrace u_1, \dots, u_d\rbrace$, and thus \emph{(ii)} is verified.  \\
\underline{Induction Step. $t \ge 0$.} Let $t \ge 0$ and suppose that $\Vert x_t \Vert > 0$ and $x_t \in \mathrm{Span}\lbrace u_1, \dots, u_d \rbrace$. First, by the established claim above, we have
\begin{align}
    \Vert x_{t+1} \Vert & \ge \left( (1 - \eta) + \eta \frac{\sigma_1}{\Vert x_t \Vert^2} \right)\Vert x_t \Vert  \\
    & \overset{(e)}{>} 0, 
\end{align}
where inequality $(e)$ follows by the induction hypothesis. Hence $\Vert x_{t+1}\Vert> 0$. Next, observe that equation \eqref{eq:lem:1} already establishes that $x_{t+1} \in \mathrm{Span}\lbrace u_1, \dots, u_d\rbrace$. 
\end{proof}

\begin{lemma}\label{lem:aux2}
    Suppose that $x_0 = Mx$ with $x \notin \mathrm{Ker}(M)$. Then, the following properties hold
    \begin{itemize}
        \item [(i)] $(\vert \cos(\theta_{1,t})\vert)_{t \ge 0}$ is a non-decreasing sequence 
        \item [(ii)] $(\vert \cos(\theta_{d,t})\vert)_{t \ge 0}$ is a non-increasing sequence 
        \item[(iii)] for all $t \ge 1$, $\Vert x_t \Vert \ge a \sqrt{\sigma_1}$, where $a = 2\sqrt{\eta (1-\eta)} \left(\vert \cos(\theta_{1,0}) \vert \vee \sqrt{ \frac{\sigma_d}{\sigma_1}} \right)  \in (0, 1]$.
        \item [(iv)] for all $i \in [d]$, either for all $t \ge 0$, $\cos(\theta_{i,t}) \ge 0$, or  all $t \ge 0$, $\cos(\theta_{i,t}) \le 0$
    \end{itemize}
\end{lemma}

\begin{proof}[Proof of Lemma \ref{lem:aux2}.]
    \underline{\textbf{Proving $(i)$ and $(ii)$.}}  To prove $(i)$ and $(ii)$, we show that for all $t \ge 0$, $\vert \cos(\theta_{1,t+1}) \vert \ge \vert \cos(\theta_{1,t}) \vert$ and $\vert \cos(\theta_{d,t+1}) \vert \le \vert \cos(\theta_{d,t}) \vert$. 
    Let $t \ge 0$, and $i \in [d]$.  Starting from \eqref{eq:gradient}, we have for all $i \in [d]$, by projecting both sides of the equation on $u_i$, we have 
    \begin{align}
        u_i^\top x_{t+1} = \left( 1 - \eta + \eta \frac{\sigma_i}{\Vert x_t \Vert^2} \right) u_i^\top x_t 
    \end{align}
    Next, using the fact that $u_i^\top x_t = \cos(\theta_{i,t}) \Vert x_{t} \Vert$, we obtain 
    \begin{align}\label{eq:aux2:eq5}
        \Vert x_{t+1} \Vert  \cos(\theta_{i,t+1}) = \left( (1- \eta) \Vert x_t \Vert + \eta \frac{\sigma_i}{\Vert x_{t} \Vert} \right) \cos(\theta_{i,t})   
    \end{align}
    Thus, for $i=1$ and $i=d$, and applying the absolute value to both sides of the equations, we obtain 
    \begin{align}
        \Vert x_{t+1} \Vert  \vert \cos(\theta_{1,t+1})\vert  = \left( (1- \eta) \Vert x_t \Vert + \eta \frac{\sigma_1}{\Vert x_{t} \Vert} \right) \vert \cos(\theta_{1,t}) \vert \label{eq:aux2:1}\\
        \Vert x_{t+1} \Vert  \vert \cos(\theta_{d,t+1})\vert  = \left( (1- \eta) \Vert x_t \Vert + \eta \frac{\sigma_i}{\Vert x_{t} \Vert} \right) \vert \cos(\theta_{d,t}) \vert \label{eq:aux2:2}
    \end{align}
    We further have by Lemma \ref{lem:aux1} that 
    \begin{align} 
        \Vert x_{t+1} \Vert & \le \left( 1- \eta + \eta \frac{\sigma_1}{\Vert x_t \Vert^2} \right) \Vert x_t \Vert =  \left( (1- \eta) \Vert x_t \Vert + \eta \frac{\sigma_1}{\Vert x_t \Vert} \right) \label{eq:aux2:3}  \\
        \Vert x_{t+1} \Vert & \ge \left( 1- \eta  + \eta \frac{\sigma_d}{\Vert x_t \Vert^2} \right)  \Vert x_t \Vert = \left((1- \eta) \Vert x_t \Vert  + \eta \frac{\sigma_d}{\Vert x_t \Vert} \right)  \label{eq:aux2:4}
    \end{align}
    Plugging \eqref{eq:aux2:1} (resp. \eqref{eq:aux2:2}) in \eqref{eq:aux2:3} (resp. \eqref{eq:aux2:4})  yields 
    \begin{align}
        \Vert x_{t+1} \Vert \vert \cos(\theta_{1,t+1}) \vert  & \ge  \Vert x_{t+1} \Vert \vert \cos(\theta_{1,t}) \vert \\
        \Vert x_{t+1} \Vert \vert \cos(\theta_{d,t+1}) \vert  & \le  \Vert x_{t+1} \Vert \vert \cos(\theta_{d,t}) \vert 
    \end{align}
    Again, by Lemma \ref{lem:aux1}, we have $\Vert x_{t+1}\Vert > 0$, thus it follows from the above that 
     \begin{align}
         \vert \cos(\theta_{1,t+1}) \vert  & \ge  \vert  \cos(\theta_{1,t}) \vert \\
       \vert \cos(\theta_{d,t+1}) \vert  & \le  \vert \cos(\theta_{d,t}) \vert. 
    \end{align} 
\underline{\textbf{Proving $(iv)$.}} Property $(iv)$ follows immediately from \eqref{eq:aux2:eq5}. \\
\underline{\textbf{Proving $(iii)$.}} Let $t \ge 0$. We have by Lemma \ref{lem:aux1}, that 
\begin{align}
    \Vert x_{t+1} \Vert & \ge \left( (1- \eta) + \eta \frac{\sigma_d}{\Vert x_t \Vert^2}\right) \Vert x_t \Vert \\
    & =   \left( (1- \eta) \Vert x_t \Vert + \eta \frac{\sigma_d}{\Vert x_t \Vert}\right) \\
    & \overset{(a)}{\ge} \inf_{x > 0}  \left( (1- \eta) x + \eta \frac{\sigma_d}{x}\right) \\
    & = 2\sqrt{\eta (1-\eta)\sigma_d} \label{eq:aux2:5}
\end{align}
where inequality $(a)$ follows by Lemma \ref{lem:aux1}, which entails that for all $t \ge 0, \Vert x_t\Vert> 0$. On the other hand, we also have 
\begin{align}
     \Vert x_{t+1} \Vert & \overset{(b)}{\ge}   \Vert x_{t+1} \Vert  \vert \cos(\theta_{1,t+1})\vert \\
     & \overset{(b)}{=} \left( (1- \eta) \Vert x_t \Vert + \eta \frac{\sigma_1}{\Vert x_{t} \Vert} \right) \vert \cos(\theta_{1,t}) \vert \\
     & \overset{(d)}{\ge} \left( (1- \eta) \Vert x_t \Vert + \eta \frac{\sigma_1}{\Vert x_{t} \Vert} \right) \vert \cos(\theta_{1,0}) \vert  \\
     & \overset{(e)}{\ge} \inf_{x > 0}  \left( (1- \eta) x + \eta \frac{\sigma_1}{x}\right) \vert \cos(\theta_{1,0}) \vert \\
    & = 2\sqrt{\eta (1-\eta)\sigma_1}  \vert \cos(\theta_{1,0}) \vert \label{eq:aux2:6}
\end{align}
where inequality $(b)$ follows from the fact that $\Vert x_{t+1} \Vert > 0$ (by Lemma \ref{lem:aux1}) and $ \vert \cos(\theta_{1,t+1})\vert \le 1 $, inequality $(c)$  follows because of inequality \eqref{eq:aux2:1}, inequality $(d)$ follows from the fact that $(\vert \cos(\theta_{1,t}) \vert)_{t \ge 0}$ is a non-decreasing sequence, and $(e)$ follows by Lemma \ref{lem:aux1}, which entails that for all $t \ge 0, \Vert x_t\Vert> 0$. 
We conclude from \eqref{eq:aux2:5} and \eqref{eq:aux2:6} that 
\begin{align}
    \Vert x_{t+1} \Vert \ge \underbrace{2\sqrt{\eta (1-\eta)} \left(\vert \cos(\theta_{1,0}) \vert \vee \sqrt{ \frac{\sigma_d}{\sigma_1}} \right)}_{a := }  \sqrt{\sigma_1},
\end{align}
where we set $a = 2\sqrt{\eta (1-\eta)} \left(\vert \cos(\theta_{1,0}) \vert \vee \sqrt{ \frac{\sigma_d}{\sigma_1}} \right)$ and clearly see that $0 < a \le 1$ (because $\sigma_d > 0$ by assumption from the problem setting). This  concludes the proof. 
\end{proof}

\begin{lemma}\label{lem:aux3}
 Suppose that $x_0 = M x$ with $x \not \in \mathrm{Ker}(M)$. For all $t \ge 0$, if $ a \sqrt{\sigma_1} \le \Vert x_t \Vert \le  b \sqrt{\sigma_1}$, then $ a \sqrt{\sigma_1} \le \Vert x_{t+1} \Vert  \le  b \sqrt{\sigma_1}$, where  
 $
 b = 2\left( (1-\eta)  + \frac{1}{2}\sqrt{\frac{\eta}{1-\eta}} \min\left(\frac{1}{\vert \cos(\theta_{1,1})\vert}, \sqrt{\frac{\sigma_1}{\sigma_d}} \right) \right) \ge 2.
 $
\end{lemma}
\begin{proof}[Proof of Lemma \ref{lem:aux3}.]
We prove the result by establishing the following two claims: Let $t \ge 0$.
\begin{itemize}
    \item \textbf{Claim 1.} If $a \sqrt{\sigma_1} \le \Vert x_t \Vert \le \sqrt{\sigma_1}$, then $ a \sqrt{\sigma_1} \le \Vert x_{t+1} \Vert  \le b\sqrt{\sigma_1}$.
    \item \textbf{Claim 2.}  If $ \sqrt{\sigma_1} \le \Vert x_t \Vert \le b \sqrt{\sigma_1}$, then $ a \sqrt{\sigma_1} \le \Vert x_{t+1} \Vert  \le b\sqrt{\sigma_1}$.
\end{itemize}
Before providing a proof of the two claims, let us verify that $b > 1$. Observe that $b = 2\left( (1-\eta) + \frac{\eta}{a}\right)$. Because $ 0 < a \le 1$ and $\eta \in (0,1)$, we see that $b \ge 2$. 

\underline{\textbf{Proving Claim 1}.} Assume that $a \sqrt{\sigma_1} \le \Vert x_t \Vert \le \sqrt{\sigma_1}$. First, we have by Lemma \ref{lem:aux2} (see statement $(iii)$), that  
\begin{align}
    \Vert x_{t+1} \Vert \ge a \sqrt{\sigma_1}. 
\end{align}
Next, we have 
\begin{align}
    \Vert x_{t+1} \Vert  & \overset{(a)}{\le} \left( 1- \eta + \eta \frac{\sigma_1}{\Vert x_t \Vert^2} \right) \Vert x_t \Vert \\
    & = \left( (1- \eta) \Vert x_t \Vert + \eta \frac{\sigma_1}{\Vert x_t \Vert} \right) \\
    & \overset{(b)}{\le} \left( (1- \eta)\sqrt{\sigma_1} + \eta \frac{\sigma_1}{a \sqrt{\sigma_1}} \right)  \\
    & \overset{(c)}{=}  \left( (1-\eta)  + \frac{1}{2}\sqrt{\frac{\eta}{1-\eta}} \min\left(\frac{1}{\vert \cos(\theta_{1,1})\vert}, \sqrt{\frac{\sigma_1}{\sigma_d}} \right)\right) \sqrt{\sigma_1} \\
    & \le  \underbrace{2 \left( (1-\eta)  + \frac{1}{2}\sqrt{\frac{\eta}{1-\eta}} \min\left(\frac{1}{\vert \cos(\theta_{1,1})\vert}, \sqrt{\frac{\sigma_1}{\sigma_d}} \right)\right)}_{b :=  } \sqrt{\sigma_1}.
\end{align}
where inequality \emph{(a)} comes from applying Lemma \ref{lem:aux1}, inequality \emph{(b)} follows from the fact that by assumption $\Vert x_t \Vert\le \sqrt{\sigma}$ and $\frac{1}{\Vert x_t \Vert} \le \frac{1}{a\sqrt{\sigma_1}}$, and equality \emph{(c)} follows from substituting the expression of $a$ as given in \eqref{eq:constant a}. Defining $b = 2 \left( (1-\eta)  + \frac{1}{2}\sqrt{\frac{\eta}{1-\eta}} \min\left(\frac{1}{\vert \cos(\theta_{1,1})\vert}, \sqrt{\frac{\sigma_1}{\sigma_d}} \right)\right)$, concludes the proof of \textbf{Claim 1}. \\
\underline{\textbf{Proving Claim 2.}} Assume that for $ \sqrt{\sigma_1} \le \Vert x_{t}\Vert \le b \sqrt{\sigma_1}$.  Again, as in the proof of \textbf{Claim 1}, by Lemma \ref{lem:aux2} (see statement \emph{(iii)}), we have that $\Vert x_{t+1}\Vert \ge a \sqrt{\sigma_1}$. Next, we have 
\begin{align}
    \Vert x_{t+1}\Vert & \overset{(a)}{\le} \left( 1-\eta + \frac{\sigma_1}{\Vert x_t\Vert^2}\right) \Vert x_{t}\Vert \\
    & \overset{(b)}{\le} \Vert x_t \Vert \\ 
    & \overset{(c)}{\le} b \sqrt{\sigma_1}
\end{align}
where inequality \emph{(a)} follows by Lemma \ref{lem:aux1} (see statement \emph{(iii)}), inequality \emph{(b)} follows from the fact that by assumption $\frac{\sigma_1}{\Vert x_t\Vert^2} \le 1$, and inequality \emph{(c)} follows from the fact that by assumption $\Vert x_t \Vert \le b \sqrt{\sigma_1}$. This concludes the proof of \textbf{Claim 2}.
\end{proof}

\begin{lemma}\label{lem:aux4}
 Suppose that $x_0 = M x$ with $x \not \in \mathrm{Ker}(M)$. If $ \Vert x_1 \Vert > b \sqrt{\sigma_1}$,  then, there exists $\tau > 0$, such that for all $t \ge \tau$, $a \sqrt{\sigma_1}\le \Vert x_{t} \Vert \le b \sqrt{\sigma_1}$. Specifically, we have 
 $$
 \tau = \left\lceil \frac{\log\left(\frac{\Vert x_1\Vert}{b \sqrt{\sigma_1}}\right)}{\log\left(\frac{1}{ (1-\eta) + \frac{\eta}{4}}\right)} \right\rceil \vee 1.
 $$
\end{lemma}
\begin{proof}[Proof of Lemma \ref{lem:aux4}.]
    Let us start by defining 
    \begin{align}
            t_\star = \inf \left\lbrace t \ge 1: \Vert x_{t}\Vert > b \sqrt{\sigma_1} \right\rbrace. 
    \end{align}  
    If $t_\star < \infty$ exists, then it corresponds to the number of iterations before the iterates $(x_t)$  enter the attraction region $[0, b \sqrt{\sigma_1}]$, and since by Lemma \ref{lem:aux2} (see statement \emph{(iii)}) for all $t \ge 1$, $\Vert x_t \Vert \ge a \sqrt{\sigma_1}$, this also means that $t_\star$ is also the number of iterations required to enter $[a \sqrt{\sigma_1}, b\sqrt{\sigma_1}]$. We establish that indeed $t_\star < \infty$. For all $1 \le t \le t_\star$, we have 
    \begin{align}
        \Vert x_{t+1} \Vert & \overset{(a)}{\le} \left( 1-\eta + \eta \frac{\sigma_1}{\Vert x_t \Vert^2}\right) \Vert x_{t}\Vert  \\
        & \overset{(b)}{\le} \left( 1-\eta + \frac{\eta}{b^2} \right) \Vert x_{t} \Vert  \\
        & \overset{(c)}{\le} \left( 1-\eta + \frac{\eta}{4} \right) \Vert x_{t} \Vert \label{eq:aux4:1}
    \end{align}
    where inequality \emph{(a)} follows by Lemma \ref{lem:aux1}, inequality \emph{(b)} follows from the fact that $\Vert x_t \Vert > b \sqrt{\sigma_1}$ for all $1 \le t \le t_\star$ which itself holds by definition of $\tau$, and inequality \emph{(c)} follows from the fact that $b \ge 2$ which holds by Lemma \ref{lem:aux3}. By iterating the equation \eqref{eq:aux4:1} over $ 1 \le t \le t_\star$, we obtain  
    \begin{align}\label{eq:aux4:2}
        \Vert x_{t_\star+1} \Vert \le \left((1-\eta) + \frac{\eta}{4}  \right)^{t_\star} \Vert x_1 \Vert. 
    \end{align}
    Observe that the RHS of \eqref{eq:aux4:2} is exponentially decreasing since $\left((1-\eta) + \frac{\eta}{4}\right) \in (0,1)$. Therefore, for $t_\star$ large enough, the RHS of \eqref{eq:aux4:2} becomes smaller than $b\sqrt{\sigma_1}$. Hence, it must hold that $t_\star \le \infty$. Moreover, if we define 
    \begin{align}
        \tau & :=  \left\lceil \frac{\log\left(\frac{\Vert x_1\Vert}{b \sqrt{\sigma_1}}\right)}{\log\left(\frac{1}{ (1-\eta) + \frac{\eta}{4}}\right)} \right\rceil \vee 1,
    \end{align}
    we see that 
    \begin{align}
            \left((1-\eta) + \frac{\eta}{4}  \right)^{\tau} \Vert x_1 \Vert \le b\sqrt{\sigma_1},
    \end{align}
    which implies that $t_\star + 1 \le \tau$. We have just shown that $\Vert x_{t_\star + 1} \Vert \le b\sqrt{\sigma_1}$. Now, by application of Lemma \ref{lem:aux2} (see statement $(iii)$), we also have that $\Vert x_{t_\star + 1} \Vert \ge  a \sqrt{\sigma_1}$. Next, using Lemma \ref{lem:aux3}, we conclude that for all $t \ge t_\star + 1$, we have $ a \sqrt{\sigma_1} \le \Vert x_t \Vert \le b\sqrt{\sigma_1}$. Further recalling the definition of $\tau$, we also have that for all $t \ge 0$, $ a \sqrt{\sigma_1} \le \Vert x_{\tau + t}\Vert \le b \sqrt{\sigma_1}$.    
\end{proof}
}

\newpage

\subsection{Proof of Lemma \ref{lem:saddle-avoidance}}\label{sec:proof saddle}

{

\begin{proof}[Proof of Lemma \ref{lem:saddle-avoidance}.] 

Before proceeding with the proof let us first note that by the random initialization \eqref{eq:random-init}, we have that $\PP(x_0 \not \in \mathrm{Ker}(M)) = 1$. This will allow us to apply Lemma \ref{lem:aux1} and Lemma \ref{lem:aux2} in almost sure sense. Throughout the proof, we implicitly use the fact that $x_0 \not \in \mathrm{Ker}(M)$ almost surely. 

Our goal is to show that if $\Vert x_t \Vert^2 \le (\sigma_1 - \sigma_2)/2$ or $\vert \Vert x_t \Vert^2 - \sigma_i \vert\le (\sigma_1 - \sigma_2)/2$ for some $i \in \lbrace 2, \dots, d\rbrace$, then $\Vert \nabla g(x_t; M)\Vert \gg 0$.

We know from Lemma \ref{lem:aux1} that $x_t \in \mathrm{Span}\lbrace u_1, \dots, u_d\rbrace$, which entails $x_t = \sum_{i=1}^d (u_i^\top x_t) u_i$, and recalling that $M = \sum_{i=1}^du_i u_i^\top $ gives 
\begin{align}
    \nabla g(x_t; M) & =  - M x_t + \Vert x_t \Vert^2 x_t  \\
    & = \sum_{i=1}^d  (-\sigma_i + \Vert x_t \Vert^2)  (u_i^\top x_t) u_i.
\end{align}
Now, applying the norm to both sides of the inequality yields
\begin{align}
    & \Vert \nabla g(x_t; M) \Vert^2   = \sum_{i=1}^d (\sigma_i - \Vert x_t \Vert^2)^2 (u_i^\top x_t)^2 \\
    & \quad \overset{(a)}{=}  \Vert x_t \Vert^2  \left( \sum_{i=1}^d  (\sigma_i - \Vert x_t \Vert^2)^2  \vert \cos(\theta_{i,t}) \vert^2 \right) \\
    & \quad  \overset{(b)}{\ge}  \Vert x_t \Vert^2 \left(  (\sigma_1 - \Vert x_t \Vert^2)^2  \vert \cos(\theta_{1,t}) \vert^2  +  \min_{ 2 \le i \le d}(\sigma_i - \Vert x_t \Vert^2)^2   \sum_{i = 2}^d  \vert \cos(\theta_{i,t}) \vert^2    \right) \\
    & \quad \overset{(c)}{=} \Vert x_t \Vert^2 \left(  (\sigma_1 - \Vert x_t \Vert^2)^2  \vert \cos(\theta_{1,t}) \vert^2  +  \min_{ 2 \le i \le d}(\sigma_i - \Vert x_t \Vert^2)^2   \vert \sin(\theta_{1,t})\vert^2 \right) \\
    & \quad \overset{(d)}{\ge} (a^2 \sigma_1 \wedge \Vert x_0\Vert^2) \left(  (\sigma_1 - \Vert x_t \Vert^2)^2  \vert \cos(\theta_{1,t}) \vert^2  +  \min_{ 2 \le i \le d}(\sigma_i - \Vert x_t \Vert^2)^2   \vert \sin(\theta_{1,t})\vert^2 \right) \\
    & \quad \overset{(e)}{\ge} (a^2 \sigma_1 \wedge \Vert x_0\Vert^2) \left(  (\sigma_1 - \Vert x_t \Vert^2)^2  \vert \cos(\theta_{1,0}) \vert^2  +  \min_{ 2 \le i \le d}(\sigma_i - \Vert x_t \Vert^2)^2  \vert \sin(\theta_{1,t})\vert^2 \right) 
\end{align}
where, to obtain equality $(a)$ we plugged in $x_t^\top u_i = \cos(\theta_{i,t}) \Vert x_t \Vert$, to obtain inequality $(b)$, we lower bound $(\sigma_i - \Vert x_t \Vert^2)^2 \ge \min_{ 2 \le j \le d} (\sigma_j - \Vert x_t \Vert^2)^2 $ for all $i \in \lbrace 2, \dots, d\rbrace$, to obtain inequality $(c)$, we use the fact that $\sum_{i = 2}^d  \vert \cos(\theta_{i,t}) \vert^2 = 1- \vert \cos(\theta_{1,t}) \vert^2 = \vert \sin(\theta_{1,t}) \vert^2$ (because $x_t \in \mathrm{Span}\lbrace u_1, \dots, u_d\rbrace$ which holds by Lemma \ref{lem:aux1}), to obtain $(d)$, we use the fact that $\Vert x_t \Vert^2 \ge a \sqrt{\sigma_1}$ for all $t\ge 1$ which follows from Lemma \ref{lem:aux2}, , and to obtain $(e)$ we have used the fact that $\vert \cos(\theta_{1,t})\vert \ge \vert \cos(\theta_{1,0} ) \vert$ for all $t \ge 0$, which holds by application of Lemma \ref{lem:aux2}.

Thus, if for some $i \in \lbrace 2, \dots, d \rbrace$, it holds $\vert \Vert x_t \Vert^2 - \sigma_i \vert \le \frac{\sigma_1 - \sigma_2}{2}$, or $\Vert x_t \Vert^2 \le \frac{\sigma_1 - \sigma_2}{2}$, then $
(\sigma_1 - \Vert x_t \Vert^2)^2 \ge \left(\frac{ (\sigma_1 - \sigma_{2})^2 }{4}\right).
$
Therefore, under such condition it must hold   
\begin{align*}
     \Vert \nabla g(x_t; M) \Vert^2  \ge (a^2 \sigma_1 \wedge \Vert x_0\Vert^2)   \left(\frac{ (\sigma_1 - \sigma_{2} )^2}{4}\right) \vert \cos(\theta_{1,0}) \vert^2.   
\end{align*}
This concludes the proof. 
\end{proof}
}

\subsection{Proof of Lemma \ref{lem:conv singular vector}}\label{sec:proof conv singular vector}
{ 
\begin{proof}[Proof of Lemma \ref{lem:conv singular vector}]

    Using similar derivations as in the proof of Lemma \ref{lem:aux2} (see \eqref{eq:aux2:eq5}), we obtain that, for all $i \in \lbrace 1, \dots, d \rbrace$, $t \ge 0$,  
    \begin{align}\label{eq:csv:eq1}
        \Vert x_{t+1} \Vert  \cos(\theta_{i,t+1}) = \left( (1- \eta) \Vert x_t \Vert + \eta \frac{\sigma_i}{\Vert x_{t} \Vert} \right) \cos(\theta_{i,t}).   
    \end{align}

    Observe that the random initialization \eqref{eq:random-init} ensures that for all $i \in [d]$, $\cos(\theta_{i,0}) \neq 0$ almost surely. Moreover, in view  of Lemma \ref{lem:aux1} which entails that for all $t \ge 0$, $\Vert x_t \Vert > 0$ almost surely, and the above equations, it also holds by induction that for all $i \in [d]$, $t \ge 0$, $\cos(\theta_{i,t}) \neq 0$ almost surely.

    Hence, we can proceed by dividing the equations given in \eqref{eq:csv:eq1}. Specifically, for all $i \in \lbrace 2, \dots, d\rbrace$, dividing the one corresponding to $1$  by that corresonding to $i$, we obtain for all $t \ge 0$
     \begin{align}\label{eq: key equation}
     \frac{\vert \cos(\theta_{1, t+1}) \vert}{\vert \cos(\theta_{i, t+1}) \vert} 
     & = \frac{\left( (1-\eta) \Vert x_t \Vert + \eta \frac{\sigma_1}{\Vert x_t \Vert }\right)}{\left( (1-\eta) \Vert x_t \Vert + \eta \frac{\sigma_i}{\Vert x_t \Vert }\right)} \frac{\vert \cos(\theta_{1, t}) \vert}{\vert \cos(\theta_{i, t}) \vert} \\ 
     & = \left(1+ \frac{\left( \eta \frac{\sigma_1 - \sigma_i}{\Vert x_t \Vert }\right)}{\left( (1-\eta) \Vert x_t \Vert + \eta \frac{\sigma_i}{\Vert x_t \Vert }\right)}  \right)\frac{\vert \cos(\theta_{1, t}) \vert}{\vert \cos(\theta_{i, t}) \vert} \\
     & = \underbrace{\left(1 + \frac{ \eta (\sigma_1 - \sigma_i) }{ (1-\eta) \Vert x_t \Vert^2 + \eta \sigma_i}  \right)}_{ > 1 \text{ because } \sigma_1 - \sigma_i > 0}\frac{\vert \cos(\theta_{1, t}) \vert}{\vert \cos(\theta_{i, t}) \vert}.
    \end{align}
    Further rearranging terms gives for all $i \in \lbrace 2, \dots, d\rbrace$, $t \ge 0$,
    \begin{align}\label{eq:svc:eq2}
         \frac{\vert \cos(\theta_{i, t}) \vert}{\vert \cos(\theta_{i, t+1}) \vert}  
         & = \left(1 + \frac{ \eta (\sigma_1 - \sigma_i) }{ (1-\eta) \Vert x_t \Vert^2 + \eta \sigma_i}  \right)\frac{\vert \cos(\theta_{1, t}) \vert}{\vert \cos(\theta_{1, t+1}) \vert}.
    \end{align}
    Now, by Lemma \ref{lem:attracting region}, for all $t \ge \tau$, we have that 
    \begin{align}\label{eq:svc:eq1}
         a \sqrt{\sigma_1} \le \Vert x_t \Vert \le b \sqrt{\sigma_1} \quad \text{and} \quad \vert \cos(\theta_{1, \tau}) \vert  \ge \vert \cos(\theta_{1, \tau+t+1}) \vert, 
    \end{align}
    where $\tau$ is as given in the Lemma \ref{lem:attracting region}.

    Next, observe that for all $t \ge 0$, $ i \in \lbrace 2, \dots,  d \rbrace $,
    \begin{align*}
        \frac{\vert \cos(\theta_{i, \tau})\vert}{\vert \cos(\theta_{i,\tau + t+1})\vert} & \overset{(a)}{=} \prod_{s = 0}^{t} \frac{\vert \cos(\theta_{i, \tau + s})\vert}{\vert \cos(\theta_{i, \tau + s+1})\vert} \\
         & \overset{(b)}{=} \prod_{s = 0}^{t} \left(1 + \frac{ \eta (\sigma_1 - \sigma_i) }{ (1-\eta) \Vert x_{\tau + s} \Vert^2 + \eta \sigma_i}  \right)\frac{\vert \cos(\theta_{1, \tau + s}) \vert}{\vert \cos(\theta_{1, \tau + s+1}) \vert} \\
         & \overset{(c)}{=} \left(\prod_{s = 0}^{t} \left(1 + \frac{ \eta (\sigma_1 - \sigma_i) }{ (1-\eta) \Vert x_{\tau + s} \Vert^2 + \eta \sigma_i}  \right) \right) \frac{\vert \cos(\theta_{1, \tau}) \vert}{\vert \cos(\theta_{1, \tau + t+1}) \vert} \\
         & \overset{(d)}{\ge} \left(1 + \frac{ \eta (\sigma_1 - \sigma_i) }{ (1-\eta) b^2 \sigma_1 + \eta \sigma_i}   \right)^{t+1} \frac{\vert \cos(\theta_{1, \tau}) \vert}{\vert \cos(\theta_{1, \tau + t+1}) \vert}  \\
         & \overset{(e)}{\ge} \left(1 + \frac{ \eta (\sigma_1 - \sigma_i) }{ (1-\eta) b^2 \sigma_1 + \eta \sigma_i}   \right)^{t+1}, 
    \end{align*}
    where equality $(a)$ follows by writing LHS as a result of a telescoping product, equality $(b)$ follows by using \eqref{eq:svc:eq2}, equality $(c)$ by observing that $\prod_{s=0}^t \frac{\vert \cos(\theta_{1, \tau+s})\vert}{\vert \cos(\theta_{1,\tau+s+1})\vert}$ is a telescoping product, and inequality $(d)$  follows from \eqref{eq:svc:eq1} where more precisely we have 
    \begin{align*}
         \frac{\eta(\sigma_1 - \sigma_i)}{(1- \eta) \Vert x_{\tau + s} \Vert^2 + \eta \sigma_i} \ge  \frac{\eta(\sigma_1 - \sigma_i)}{(1- \eta) b^2 \sigma_1 + \eta \sigma_i} > 0,
    \end{align*}
    and inequality $(e)$ also follows from \eqref{eq:svc:eq1}, specifically from the fact that $\vert \cos(\theta_{1, \tau}) \vert \ge \vert \cos(\theta_{1, \tau+t+1}) \vert$. \\
    We have just shown that for all $i \in \lbrace 2, \dots, d \rbrace$, $t \ge 0$, we have 
    \begin{align}
        \vert \cos(\theta_{i,\tau + t}) \vert  & \le {\underbrace{\left(1 + \frac{ \eta (\sigma_1 - \sigma_i) }{ (1-\eta) b^2 \sigma_1 + \eta \sigma_i}   \right)}_{ < 1 \text{ because } \sigma_1 - \sigma_2 > 0}}^{-t} \vert \cos(\theta_{i,\tau}) \vert \label{eq:svc:eq5} \\ 
        & \le \left(1 + \frac{ \eta (\sigma_1 - \sigma_i) }{ (1-\eta) b^2 \sigma_1 + \eta \sigma_i}   \right)^{-t} 
    \end{align}
    This means that $\max_{2\le i \le d}\vert \cos(\theta_{i, t}) \vert \underset{t \to \infty}{\longrightarrow} 0$.

    Observe that 
    \begin{align}
         \left\Vert \frac{x_{\tau + t}}{\Vert x_{\tau + t}\Vert} - u_1 \right\Vert^2 & = 2 (1- \cos(\theta_{i,\tau + t}))\label{eq:svc:eq3} \\ 
         \left\Vert \frac{x_{\tau + t}}{\Vert x_{\tau + t}\Vert} + u_1 \right\Vert^2 & = 2 (1 + \cos(\theta_{i,\tau + t})). \label{eq:svc:eq4} 
    \end{align}
    Therefore, we have  
    \begin{align}
         \left\Vert \frac{x_{\tau + t}}{\Vert x_{\tau + t}\Vert} - u_1 \right\Vert^2  \left\Vert \frac{x_{\tau + t}}{\Vert x_{\tau + t}\Vert} + u_1 \right\Vert^2 & = 4 (1 - \cos^2(\theta_{i,\tau + t})) \\
          & \overset{(a)}{=} 4 \sum_{i=2}^d \cos^2(\theta_{i,\tau + t}) \\
          & \overset{(b)}{\le} 4 \sum_{i=2}^d  \left(1 + \frac{ \eta (\sigma_1 - \sigma_i) }{ (1-\eta) b^2 \sigma_1 + \eta \sigma_i}   \right)^{-2t} \vert \cos(\theta_{i,\tau})\vert^2 \\
          & \overset{(c)}{\le} 4 \left(1 + \frac{ \eta (\sigma_1 - \sigma_2) }{ (1-\eta) b^2 \sigma_1 + \eta \sigma_2}   \right)^{-2t} \sum_{i=2}^d \vert \cos(\theta_{i,\tau})\vert^2 \\
          & \overset{(d)}{=} 4 \left(1 + \frac{ \eta (\sigma_1 - \sigma_2) }{ (1-\eta) b^2 \sigma_1 + \eta \sigma_2}   \right)^{-2t} (1 - \vert \cos(\theta_{1,\tau})\vert^2) \\
          & \overset{(e)}{\le} 4 \left(1 + \frac{ \eta (\sigma_1 - \sigma_2) }{ (1-\eta) b^2 \sigma_1 + \eta \sigma_2}   \right)^{-2t}.
    \end{align}
    where in equality $(a)$ we plug in \eqref{eq:svc:eq3} and \eqref{eq:svc:eq4}, in inequality $(b)$ we use \eqref{eq:svc:eq5}, in inequality $(c)$ we use the fact that $\sigma_2 \ge \cdots \ge \sigma_d$, in equality $(d)$ we use the fact that $\sum_{i=1}^d \cos^2(\theta_{i,\tau +t}) = 1$ which holds since $x_{\tau + t} \in \mathrm{Span}\lbrace u_1, \dots, u_d\rbrace$ by Lemma \ref{lem:aux1}, and in inequality $(e)$ we use the elementary fact that $(1 - \vert \cos(\theta_{1,\tau})\vert^2) \le 1$.  
\end{proof}
}

\subsection{Proof of Lemma \ref{lem:conv singular value}}\label{sec:proof conv singular value}

\begin{proof}[Proof of Lemma \ref{lem:conv singular value}.]
    Let us denote for all $t \ge 0$,
    \begin{align}
        \rho_t & = 1 - \frac{\vert \cos(\theta_{1, t})\vert} {\vert \cos(\theta_{1, t+1})\vert}, \\
        \epsilon_t & = \frac{\Vert x_t \Vert }{\sqrt{\sigma_1}} - 1.
    \end{align}
{
Using similar derivations as in the proof of Lemma \ref{lem:aux2} (see derivations leading to \eqref{eq:aux2:eq5}), we obtain   
    \begin{align}
        \Vert x_{t+1} \Vert \vert \cos(\theta_{1, t+1})\vert = \left((1-\eta)\Vert x_t \Vert + \eta \frac{\sigma_1}{\Vert x_t \Vert} \right) \vert \cos(\theta_{1, t})\vert.
    \end{align}
    Thus, since for all $t \ge 0$, $\Vert x_{t}\Vert = \sqrt{\sigma_1}(\epsilon_t + 1)$, and by definition of $\rho_t$,  we have 
    \begin{align}
        \epsilon_{t+1} & = \left( (1- \eta) (\epsilon_t + 1) + \frac{\eta}{\epsilon_t + 1}\right) (1 - \rho_t) - 1  \\
        & =  \left( (1- \eta) (\epsilon_t + 1) + \frac{\eta}{\epsilon_t + 1}\right)   - 1 - \rho_t \left( (1- \eta) (\epsilon_t + 1) + \frac{\eta}{\epsilon_t + 1}\right) \\
        & = \frac{(1-\eta)\epsilon_t^2 + (1-2\eta)\epsilon_t }{(\epsilon_t + 1)}  - \rho_t \left( (1- \eta) (\epsilon_t + 1) + \frac{\eta}{\epsilon_t + 1}\right). 
    \end{align}
    With the choice $\eta=1/2$, it follows that
    \begin{align}
        \epsilon_{t+1} = \frac{\epsilon_t^2}{2 (\epsilon_t + 1)} - \frac{\rho_t}{2} \left( (\epsilon_t + 1) + \frac{1}{\epsilon_t + 1}\right), \label{eq:csval:eq3}
    \end{align}
    With this we clearly see the similarity between gradient descent and Heron's method. Our goal is to show that both sequence $(\rho_t)_{t \ge 0}$ and $(\epsilon_t)_{t \ge 0}$ converge to zero as $t \to \infty$. \\
\underline{\textbf{Step 1. Convergence of $(\rho_t)_{t \ge 0}$.}} First, we verify that $\rho_t \underset{t \to \infty}{\longrightarrow} 0$. Let $\tau_1 = \tau$ where $\tau$ is as defined in Lemma \ref{lem:attracting region}.} We have {
for all $t \ge \tau_1$,
\begin{align}
    \rho_t & = 1 - \frac{\vert \cos(\theta_{1,t} ) \vert }{\vert \cos(\theta_{1,t+1} ) \vert} \\ 
    & = \frac{\vert \cos(\theta_{1,t+1} ) \vert - \vert \cos(\theta_{1,t} ) \vert }{\vert \cos(\theta_{1,t+1} ) \vert} \\
    &  \overset{(a)}{\le}  \frac{1 - \vert \cos(\theta_{1,t} ) \vert }{\vert \cos(\theta_{1,t+1} ) \vert} \\
    & \overset{(b)}{\le}  \frac{1 - \vert \cos(\theta_{1,t} ) \vert^2 }{\vert \cos(\theta_{1,t+1} ) \vert} \\
    & \overset{(c)}{\le}  \frac{1 - \vert \cos(\theta_{1,t} ) \vert^2 }{\vert \cos(\theta_{1,0} ) \vert} \\
    & \overset{(d)}{=} \frac{1}{\vert \cos(\theta_{1,0} ) \vert} \left(  \sum_{i=2}^d \cos^2(\theta_{i, t})\right) \\
    & \overset{(e)}{\le} \frac{1}{\vert \cos(\theta_{1,0} ) \vert} \left(   1 + \frac{\eta(\sigma_1 - \sigma_2)}{(1-\eta)b^2 \sigma_1 + \eta \sigma_2}\right)^{-2(t-\tau)} \label{eq:csval:eq}
\end{align}
where inequality $(a)$ follows by upper bounding $\vert \cos(\theta_{i,t+1})\vert \le 1$, inequality $(b)$ follows from $\vert \cos(\theta_{i,t})\vert^2 \le \vert \cos(\theta_{i,t})\vert \le 1$, inequality $(c)$ follows from Lemma \ref{lem:aux2} where we have $\vert \cos(\theta_{1, t+1})\vert \ge \vert \cos(\theta_{1, 0})\vert$, and equality $(e)$ follows by carrying similar derivations as in the proof of Lemma \ref{lem:conv singular vector}. This clearly shows that $\rho_t \underset{t \to \infty}{\longrightarrow} 0$.\\
}

\medskip 

{\underline{\textbf{Step 2. Convergence of $(\epsilon_t)_{t \ge 0}$.}} Now, let us show that $\epsilon_t \underset{t \to \infty}{\longrightarrow} 0$. 
    By Lemma \ref{lem:attracting region}, for all $t \ge \tau_1 = \tau $, we have $ a \sqrt{\sigma_1} \le \Vert x_{\tau} \Vert \le b \sqrt{\sigma_1}$ where $\tau$ is as given in Lemma \ref{lem:attracting region}. This means, by definition of $\epsilon_t$, that 
    $0 < a  \le \epsilon_t + 1 \le b$. Therefore, starting from  \eqref{eq:csval:eq3}, we obtain  
    \begin{align}
        \epsilon_{t+1} \ge - \frac{\rho_t}{2}\left(b + \frac{1}{a}\right). \label{eq:csval:eq4}
    \end{align}
    Recalling \eqref{eq:csval:eq}, we see that for $t$ large enough we have $- \frac{\rho_t}{2}\left(b + \frac{1}{a}\right) \ge -\frac{1}{4}$ which would also mean that $\epsilon_{t+1} \ge -\frac{1}{4}$. Indeed, from \eqref{eq:csval:eq}, we have 
    \begin{align}\label{eq:csval:eq7}
        \frac{\rho_t}{2}\left(b + \frac{1}{a}\right) \le \frac{1}{2 \vert \cos(\theta_{1,0} ) \vert} \left(b + \frac{1}{a}\right) \left(   1 + \frac{\eta(\sigma_1 - \sigma_2)}{(1-\eta)b^2 \sigma_1 + \eta \sigma_2}\right)^{-2(t-\tau_1)} \le \frac{1}{4}
    \end{align}
    so long as, 
    \begin{align}
        t- \tau_1 \ge  \frac{\log\left(\frac{2}{\vert \cos(\theta_{1,0}) \vert}\left( b + \frac{1}{a} \right)\right)}{2 
        \log\left(   1 + \frac{\eta(\sigma_1 - \sigma_2)}{(1-\eta)b^2 \sigma_1 + \eta \sigma_2}\right) 
        }. 
    \end{align}
    We can further simplify the lower bound above by noting that  
    \begin{align*}
        \frac{\log\left(\frac{2}{\vert \cos(\theta_{1,0}) \vert}\left( b + \frac{1}{a} \right)\right)}{2 
        \log\left(   1 + \frac{\eta(\sigma_1 - \sigma_2)}{(1-\eta)b^2 \sigma_1 + \eta \sigma_2}\right)} & = \frac{\log\left(\frac{2}{\vert \cos(\theta_{1,0}) \vert}\left( b + \frac{1}{a} \right)\right)}{- 2 
        \log\left(  1 - \frac{\eta(\sigma_1 - \sigma_2)}{((1-\eta)b^2  + \eta) \sigma_1} \right)} \\
        & \overset{(a)}{\le} \tau_2 := \frac{((1-\eta) b^2 + \eta)\sigma_1}{2\eta(\sigma_1 - \sigma_2)}   \log\left(\frac{2}{\vert \cos(\theta_{1,0}) \vert}\left( b + \frac{1}{a} \right)\right)
    \end{align*}
    where in inequality $(a)$ we used that the fact that $\exp(-x) \ge 1 - x$

    Therefore, recalling \eqref{eq:csval:eq4}, we have for all $t \ge \tau_\star := \tau_1 + \tau_2 $, 
    \begin{align}\label{eq:csval:eq6}
        \epsilon_{t+1} \ge - \left[\left(\frac{\rho_t}{2}\left(b + \frac{1}{a}\right)\right) \wedge \frac{1}{4} \right]. 
    \end{align}
    Next, let $t \ge \tau_\star + 1$, we have just shown that $\epsilon_t \ge -\frac{1}{4}$. Thus, starting from \eqref{eq:csval:eq3}, we have
    \begin{align}
        \epsilon_{t+1} & = \frac{\epsilon_t^2}{2 (\epsilon_t + 1)} - \frac{\rho_t}{2} \left( (\epsilon_t + 1) + \frac{1}{\epsilon_t + 1}\right) \\
         & \overset{(a)}{\le} \frac{\epsilon_t^2}{2 (\epsilon_t + 1)} \\
        & \overset{(b)}{=}  \frac{\epsilon_t^2}{2(\epsilon_t + 1)} \mathbf{1} \left\lbrace \vert \epsilon_t \vert \le \frac{1}{4} \right\rbrace + \frac{\epsilon_t^2}{2(\epsilon_t + 1)} \mathbf{1} \left\lbrace  \epsilon_t  > \frac{1}{4} \right\rbrace  \\
        & \overset{(c)}{\le}  \frac{2\epsilon_t^2}{3} \mathbf{1}\left\lbrace \vert \epsilon_t \vert \le \frac{1}{4} \right\rbrace + \frac{\epsilon_t}{2} \mathbf{1} \left\lbrace  \epsilon_t  > \frac{1}{4} \right\rbrace   \\
        & \overset{(d)}{\le}   \frac{2\vert \epsilon_t \vert}{3}   \label{eq:csval:eq5}
    \end{align}
    where inequality $(a)$ holds because $\rho_t > 0$ and $\epsilon_t + 1> 0$, equality $(b)$ follows from the fact that $\epsilon_t \ge -\frac{1}{4}$, inequality $(c)$ holds because $\epsilon_t + 1 \ge \frac{3}{4}$ if $\vert \epsilon_t \vert \le \frac{1}{4}$, and $\epsilon_t + 1 \ge \vert \epsilon_t \vert$ if $\epsilon_t  > \frac{1}{4}$, and inequality $(d)$ follows from the fact that $\epsilon_t^2 \le \vert \epsilon_t \vert$ if $\vert \epsilon_t \vert \le \frac{1}{4}$.\\
    In view of \eqref{eq:csval:eq6} and \eqref{eq:csval:eq5}, we have just shown that $t > \tau_\star$, $
         - \frac{\rho_t}{2}\left(b + \frac{1}{a}\right) \le  \epsilon_{t+1} \le \frac{2}{3}  \vert\epsilon_t \vert$, thus implying that, 
    \begin{align*}
         \vert \epsilon_{t+1} \vert  & \le \frac{2 \vert\epsilon_t \vert}{3} + \frac{\rho_t}{2}\left(b + \frac{1}{a}\right) \\
         & \overset{(a)}{\le} \frac{2 \vert\epsilon_t \vert}{3} + \frac{1}{2 \vert \cos(\theta_{1,0} ) \vert} \left(b + \frac{1}{a}\right) \left(   1 + \frac{\eta(\sigma_1 - \sigma_2)}{(1-\eta)b^2 \sigma_1 + \eta \sigma_2}\right)^{-2(t-\tau_1)} \\
         & =  \frac{2 \vert\epsilon_t \vert}{3} + \frac{1}{2 \vert \cos(\theta_{1,0} ) \vert} \left(b + \frac{1}{a}\right) \left(   1 + \frac{\eta(\sigma_1 - \sigma_2)}{(1-\eta)b^2 \sigma_1 + \eta \sigma_2}\right)^{-2(t - \tau_2-\tau_1  + \tau_2  )} \\
         & \overset{(b)}{\le} \frac{2 \vert\epsilon_t \vert}{3} + \frac{1}{4} \left(   1 + \frac{\eta(\sigma_1 - \sigma_2)}{(1-\eta)b^2 \sigma_1 + \eta \sigma_2}\right)^{-2(t - \tau_\star - 1)},
    \end{align*}
    where in $(a)$ we used \eqref{eq:csval:eq}, in equality $(b)$, and in inequality $(b)$ we used the inequality \eqref{eq:csval:eq7} which holds true $t = \tau_2 + 1$.
    Now, by recursively using the inequality above from $t = \tau_\star + 1$ to $t = \tau_\star + k$, we obtain for all $k \ge 1$, 
    \begin{align}
        \vert \epsilon_{\tau_\star + k + 1 } \vert & \le \left(\frac{2 }{3}\right)^{k} \vert\epsilon_t \vert+ \frac{1}{4}\left( \sum_{i=0}^{k-1}\left(\frac{2}{3}\right)^{k-1-i}   \left( 1 + \frac{\eta(\sigma_1 - \sigma_2)}{(1-\eta)b^2 \sigma_1 + \eta \sigma_2}\right)^{-2i} \right) \\
         & \overset{(a)}{\le} \left(\frac{2 }{3}\right)^{k} (b-1) + \frac{1}{4}\left( \sum_{i=0}^{k-1}\left(\frac{2}{3}\right)^{k-1-i}   \left( 1 + \frac{\eta(\sigma_1 - \sigma_2)}{(1-\eta)b^2 \sigma_1 + \eta \sigma_2}\right)^{-2i} \right) \\
         & \overset{(b)}{\le} \left(\frac{2 }{3}\right)^{k} (b-1)  + \frac{k}{4} \left( \left( \frac{2}{3}\right) \vee \left( 1 - \frac{\eta}{(1-\eta)b^2 + \eta} \left(\frac{\sigma_1 - \sigma_2}{\sigma_1} \right)\right)^2 \right)^{k-1}, \\
         & \overset{(c)}{=} \left(\frac{2 }{3}\right)^{k} (b-1)  + \frac{k}{4} \left( \left( \frac{2}{3}\right) \vee \left( 1 - \frac{1}{b^2 + 1} \left(\frac{\sigma_1 - \sigma_2}{\sigma_1} \right)\right)^2 \right)^{k-1}, \\
         & \le \left( \frac{2(b-1)}{3} +  \frac{k}{4}\right) \exp\left[ -  \left( \frac{1}{3} \wedge  \left( \frac{2(\sigma_1 - \sigma_2)}{(b^2 + 1)\sigma_1 }\right)\right) (k-1)  \right]
    \end{align}
    where in $(a)$ we used the fact that $\vert \epsilon_t\vert \le b-1$ for all $t \ge 1$, in inequality $(b)$, we used Lemma \ref{lem:series inequality} with $\alpha = \frac{2}{3}$ and $\gamma = \left( 1 + \frac{\eta(\sigma_1 - \sigma_2)}{(1-\eta)b^2 \sigma_1 + \eta \sigma_2}\right)^{-2} = \left( 1 - \frac{\eta}{(1-\eta)b^2 + \eta} \left(\frac{\sigma_1 - \sigma_2}{\sigma_1} \right)\right)^2$, and in equality $(c)$, we simply used the fact that  $\eta= 1/2$. 
    Thus, we have just shown that for all $t \ge \tau_\star$,
    \begin{align}
        \vert \epsilon_{t + 2} \vert & \le \left( \frac{t -\tau_\star + 1}{4} + \frac{2(b-1)}{3}\right) \left( \frac{2}{3}  \vee  \left(1 - \frac{\sigma_1 - \sigma_2}{(b^2 + 1)\sigma_1}\right)^2  \right)^{t-\tau_\star}.
    \end{align}    
}
\end{proof}

\newpage 
\section{Proof of Theorem \ref{thm:bigtheorem}} \label{app:proof-big-thm}

{
\begin{proof}[Proof of Theorem \ref{thm:bigtheorem}.]
Let us recall that $M_1 = M$, and $M_{\ell} = M - \sum_{i=1}^{\ell-1} \hat{\sigma}_i \hat u_i \hat u_i^\top$ for $\ell \in \lbrace 2, \dots,  k \rbrace$, where for each $\ell \in [k]$, $\hat{\sigma}_\ell, \hat{u}_\ell$ are obtained by running the gradient descent iterations \eqref{eq:gradient} for a rank-1 approximation of the matrix $M_\ell$. Let us denote $\tilde{\sigma}_\ell$ the maximum eigenvalues of $M_\ell$ (i.e.,  $\tilde{\sigma}_\ell = \sigma_{\max}(M_\ell)$), and $\tilde{u}_\ell$ its corresponding eigenvectors (i.e., $M_\ell \tilde{u}_\ell = \tilde{\sigma}_\ell \tilde{u}_\ell$). Let us also denote $\tilde{\sigma}_{\ell,2}$ the second largest eigenvalue of $M_\ell$ (i.e., $ \tilde{\sigma}_{\ell,2} = \sigma_{2}(M_\ell)$).

Let $\rho > 0$. For each $\ell \in [k]$, let us denote $T_\ell(\rho)$ the number iterations required for the application of gradient descent on the matrix $M_\ell$ to ensure that  
\begin{align}\label{eq:grad-app}
    \vert \tilde{\sigma}_\ell - \hat \sigma_\ell \vert \le \rho \sigma_\ell \qquad \text{and} \qquad \Vert \tilde{u}_\ell - \hat u_\ell \Vert \wedge \Vert \tilde{u}_\ell +\hat u_\ell \Vert \le \rho  
\end{align}
Note that 
$T_\ell(\rho)$ exists by Lemma \ref{lem:conv singular vector} and Lemma \ref{lem:conv singular value} (or Theorem \ref{thm:convergence-1}). For now let us keep $\rho$ arbitrary, we will precise its choice later in the proof.

Before proceeding with the proof let us further define for all $\ell \in [k]$, a shorthand for the error term: 
\begin{align}
    \Delta_\ell := \left\Vert  M_\ell -    \sum_{i = \ell}^d \sigma_i u_i u_i^\top  \right\Vert. 
\end{align}

\underline{\textbf{Step 1.}} Let $\ell \ge 1$, Next, we have 
\begin{align}
    \vert \sigma_{\ell } - \hat{\sigma}_{\ell }\vert & \le \vert\sigma_{\ell} - \tilde{\sigma}_{\ell} \vert  + \vert \tilde \sigma_{\ell} - \hat{\sigma}_{\ell} \vert 
\end{align}
Since $\tilde{\sigma}_{\ell} = \sigma_{\max}(M_{\ell})$, we have by Weyl's inequality, 
\begin{align*}
    \vert \sigma_{\ell} - \tilde \sigma_{\ell} \vert, \vert \sigma_{\ell+1} - \tilde \sigma_{\ell,2} \vert \le \left\Vert M_{\ell} - \sum_{i=\ell}^{d} \sigma_i u_i u_i^\top \right\Vert = \Delta_{\ell}. 
\end{align*}
Moreover, by application of gradient descent \ref{eq:grad-app} we have $\vert \tilde{\sigma}_{\ell} - \hat \sigma_{\ell}\vert \le \rho \sigma_{\ell}$, thus we may have 
\begin{align}\label{eq:thm2:aux1}
    \vert \sigma_{\ell } - \hat{\sigma}_{\ell }\vert \le \rho \sigma_\ell  + \Delta_{\ell}. 
\end{align}

\underline{\textbf{Step 2.}} Let $\ell \ge 1$. By triangular inequality, we have  
\begin{align}
    \Vert  {u}_{\ell } - \hat{u}_{\ell } \Vert \wedge \Vert {u}_{\ell} + \hat{u}_{\ell} \Vert & \le  \Vert  \tilde{u}_{\ell} - u_{\ell} \Vert \wedge \Vert \tilde{u}_{\ell} + u_{\ell} \Vert  + \Vert  \tilde{u}_{\ell} - \hat{u}_{\ell} \Vert \wedge \Vert \tilde{u}_{\ell} + \hat{u}_{\ell } \Vert
\end{align}
Now, using Davis-Kahan's inequality, we obtain 
\begin{align}
    \Vert \tilde{u}_{\ell }  - u_{\ell } \Vert \wedge \Vert \tilde{u}_{\ell }  + u_{\ell } \Vert \le \frac{\sqrt{2}}{\sigma_{\ell } - \sigma_{\ell } }   \left\Vert M_{\ell} - \sum_{i=\ell }^{d} \sigma_i u_i u_i^\top \right\Vert =  \frac{\sqrt{2} }{\sigma_{\ell} - \sigma_{\ell + 1} } \Delta_{\ell} 
\end{align}
Furthermore, by application of gradient descent, the guarantee \eqref{eq:grad-app} entails that $\Vert  \tilde{u}_{\ell} - \hat{u}_{\ell} \Vert \wedge \Vert \tilde{u}_{\ell} + \hat{u}_{\ell} \Vert \le \rho$. Thus, we have 
\begin{align}\label{eq:thm2:aux2}
     \Vert  {u}_{\ell} - \hat{u}_{\ell} \Vert \wedge \Vert {u}_{\ell} + \hat{u}_{\ell} \Vert & \le \rho +  \frac{\sqrt{2} }{\sigma_{\ell} - \sigma_{\ell + 1} } \Delta_{\ell}. 
\end{align}

\underline{\textbf{Step 3.}} Let $\ell \ge 1$. First, note that 
\begin{align}
    \Delta_{\ell + 1} & : = \left\Vert M_{\ell + 1} - \sum_{i=\ell + 1}^{d} \sigma_i u_i u_i^\top \right\Vert  \\
    & = \left\Vert M_\ell -  \hat{\sigma}_{\ell} \hat u_{\ell} \hat u_\ell^\top  - \sum_{i=\ell + 1}^{d} \sigma_i u_i u_i^\top \right\Vert \\
    & = \left\Vert M_\ell   - \sum_{i=\ell }^{d} \sigma_i u_i u_i^\top  +  \sigma_{\ell} u_{\ell}  u_\ell^\top    -  \hat{\sigma}_{\ell} \hat u_{\ell} \hat u_\ell^\top  \right\Vert \\
    & \le \left\Vert  M_\ell   - \sum_{i=\ell }^{d} \sigma_i u_i u_i^\top \right\Vert  + \left\Vert  \sigma_{\ell} u_{\ell}  u_\ell^\top    -  \hat{\sigma}_{\ell} \hat u_{\ell} \hat u_\ell^\top \right\Vert  \\ 
    & = \Delta_{\ell} + \left\Vert (\sigma_\ell - \hat{\sigma}_\ell)\hat{u}_\ell \hat {u}_\ell^\top  +  \sigma_\ell ( u_\ell  {u}_\ell^\top  -  \hat{u}_\ell \hat {u}_\ell^\top)   \right\Vert  \\
    & \le \Delta_{\ell} + \vert \sigma_\ell - \hat{\sigma}_\ell\vert + \sqrt{2}\sigma_\ell \left(\Vert u_\ell - \hat{u}_\ell \Vert \wedge \Vert u_\ell + \hat{u}_\ell \Vert \right),
\end{align}
where we used the triangular inequality, then Lemma \ref{lem:elem-eq-2} in the last inequality. Thus, using the inequalities \eqref{eq:thm2:aux1} and \eqref{eq:thm2:aux2}, we obtain
\begin{align}
    \Delta_{\ell + 1} & \le (\sigma_\ell+ \sqrt{2}\sigma_\ell)\rho + \Delta_{\ell} + \frac{2\sigma_\ell}{\sigma_\ell - \sigma_{\ell+1}} \Delta_\ell   \\
    & \le 3\sigma_\ell \rho + \frac{3\sigma_\ell}{\sigma_{\ell} - \sigma_{\ell+1}} \Delta_\ell. 
\end{align}
\underline{\textbf{Step 4 (concluding inequalities).}} By iterating the inequalities above, and recalling that $\Delta_1 = 0$, we deduce that 
\begin{align}
    \Delta_{\ell + 1} & \le \sum_{i=1}^\ell \left(\prod_{j = i+1}^\ell \frac{3\sigma_j}{\sigma_j - \sigma_{j+1}}   \right) (3\sigma_i \rho) \\
    & \le 3 \ell \sigma_1 \left( \max_{i \in [\ell]} \frac{3\sigma_i}{\sigma_i - \sigma_{i+1}}  \right)^{\ell}  \rho
\end{align}
Consequently, we have for all $\ell \in [k]$,
\begin{align}
     \vert \sigma_{\ell } - \hat{\sigma}_{\ell }\vert  & \le  \sigma_\ell \left(\rho  +  3 (\ell-1) \frac{\sigma_1}{\sigma_\ell} \left( \max_{i \in [\ell-1]} \frac{3\sigma_i}{\sigma_i - \sigma_{i+1}}  \right)^{\ell-1} \rho \right) .   \\
     & \le \sigma_\ell \left(\rho  +  3 k \frac{\sigma_1}{\sigma_k} \left( \max_{i \in [k]} \frac{3\sigma_i}{\sigma_i - \sigma_{i+1}}  \right)^{k} \rho \right) 
\end{align}
and 
\begin{align}
     \Vert  {u}_{\ell} - \hat{u}_{\ell} \Vert \wedge \Vert {u}_{\ell} + \hat{u}_{\ell} \Vert & \le \rho +  \frac{3\sqrt{2} (\ell-1) \sigma_1}{\sigma_{\ell} - \sigma_{\ell + 1}} \left( \max_{i \in [\ell-1]} \frac{3\sigma_i}{\sigma_i - \sigma_{i+1}}  \right)^{\ell-1} \rho  \\
     & \le \rho +  \frac{\sqrt{2} (\ell - 1) \sigma_1}{\sigma_{\ell}} \left( \max_{i \in [\ell]} \frac{3\sigma_i}{\sigma_i - \sigma_{i+1}}  \right)^{\ell} \rho  \\
     & \le \rho +  3 k \frac{ \sigma_1}{\sigma_{k}} \left( \max_{i \in [k]} \frac{3\sigma_i}{\sigma_i - \sigma_{i+1}}  \right)^{k} \rho
\end{align}

\underline{\textbf{Step 5 (Choice of $\rho$).}} Let $\epsilon\in(0,1/2)$. We choose $\rho$ so that  
\begin{align}\label{eq:rho}
    3 k \frac{ \sigma_1}{\sigma_{k}} \left( \max_{i \in [k]} \frac{3\sigma_i}{\sigma_i - \sigma_{i+1}}  \right)^{k} \rho = \frac{\epsilon}{2} \iff \frac{1}{\rho} =  \frac{ \sigma_1}{\sigma_{k}}  \left( \max_{i \in [k]} \frac{3\sigma_i}{\sigma_i - \sigma_{i+1}}  \right)^{k}  \frac{6k}{\epsilon}
\end{align}
This entails that for all $\ell \in [k]$, 
\begin{align}
      \Delta_ \ell  & \le \frac{\epsilon \sigma_\ell }{2}  \\
       \vert \sigma_{\ell } - \hat{\sigma}_{\ell }\vert  & \le \epsilon \sigma_\ell,  \\
     \Vert  {u}_{\ell} - \hat{u}_{\ell} \Vert \wedge \Vert {u}_{\ell} + \hat{u}_{\ell} \Vert & \le \epsilon. 
\end{align}

\underline{\textbf{Step 6 (Required number of iterations)}}
By theorem  \ref{thm:convergence-1}, we know that for all $\ell \in [k]$,
\begin{align*}
    T_\ell (\rho) = \frac{c_{\ell, 1}\tilde{\sigma}_\ell}{\tilde{\sigma}_\ell- \tilde{\sigma}_{\ell,2}} \log\left(\frac{e \tilde{\sigma}_\ell}{(\tilde \sigma_\ell - \tilde\sigma_{\ell, 2}) \rho}  \right) + c_{\ell, 2} \log\left( \frac{e(\tilde \sigma_\ell^2 + 1)}{\tilde \sigma_\ell} \right). 
\end{align*}
where $c_{1,\ell}$ and $c_{2, \ell}$ are strictly positive constants. Observe that Weyl's inequality ensures that $ \vert \tilde \sigma_{\ell} - \sigma_{\ell }\vert, \vert \tilde \sigma_{\ell,2} - \sigma_{\ell+1 }\vert  \le \Delta_ \ell  \le \frac{\epsilon \sigma_\ell }{2}$ which gives 
\begin{align*}
    \left(1-\frac{\epsilon}{2}\right) \sigma_\ell  & \le \tilde{\sigma}_\ell \le \left(1+\frac{\epsilon}{2}\right) \sigma_\ell, \\
    \left(1-\frac{\epsilon}{2}\right) \sigma_{\ell+1}  & \le \tilde{\sigma}_{\ell,2} \le \left(1+\frac{\epsilon}{2}\right) \sigma_{\ell+1},
\end{align*}
thus
\begin{align*}
    \sigma_\ell \left( \frac{\sigma_\ell - \sigma_{\ell+1}}{\sigma_\ell} - \epsilon \right) \le \tilde{\sigma}_\ell - \tilde \sigma_{\ell , 2} \le \sigma_\ell \left( \frac{\sigma_\ell - \sigma_{\ell+1}}{\sigma_\ell} + \epsilon \right), 
\end{align*}
Since by assumption we have  
\begin{align*}
\epsilon \le \frac{1}{2}\min_{\ell \in [k]} \left(\frac{\sigma_\ell - \sigma_{\ell+1}}{\sigma_\ell} \right) \le  \frac{1}{2},    
\end{align*}
it further holds 
\begin{align*}
    \frac{2}{5} \left( \frac{\sigma_\ell - \sigma_{\ell+1}}{\sigma_\ell}  \right) \le \frac{\tilde{\sigma}_\ell - \tilde \sigma_{\ell , 2}}{\tilde{\sigma}_\ell} \le 2 \left( \frac{\sigma_\ell - \sigma_{\ell+1}}{\sigma_\ell} \right). 
\end{align*}
Therefore, we have for all $\ell \in [k]$,
\begin{align*}
    T_\ell(\rho) & \le \frac{c_{\ell, 1}' \sigma_\ell}{\sigma_\ell- {\sigma}_{\ell + 1}} \log\left(\frac{e {\sigma}_\ell}{(\sigma_\ell - \sigma_{\ell + 1}) \rho}  \right) + c_{\ell, 2}' \log\left( \frac{e(\sigma_\ell^2 + 1)}{\sigma_\ell} \right) \\
     & \le c_{1}' \max_{i \in [k]}\left(\frac{\sigma_i}{\sigma_i- {\sigma}_{i + 1}} \right) \log\left(\frac{e}{\rho} \max_{i \in [k]}\left(\frac{ {\sigma}_i}{(\sigma_i - \sigma_{i + 1}) } \right) \right) + c_{2}' \log\left( e\left( \sigma_1 + \frac{1}{\sigma_k} \right) \right).  
\end{align*}
where $c_1', c_2'$ are the new positive constants.
We note from the choice of $\rho$ that 
\begin{align*}
    \log\left(\frac{1}{\rho}\right) \le c_3' k \log\left(  \frac{\sigma_1}{\sigma_k} \max_{i \in [k]}\left(\frac{\sigma_i}{\sigma_i - \sigma_{i+1}} \right) \frac{k}{\epsilon}   \right).
\end{align*}
Hence, for all $\ell \in [k]$,
\begin{align*}
     T_\ell(\rho) \le C_1 k \max_{i \in [k]}\left(\frac{\sigma_i}{\sigma_i- {\sigma}_{i + 1}} \right) \log\left(\frac{\sigma_1}{\sigma_k} \max_{i \in [k]}\left(\frac{ {\sigma}_i}{(\sigma_i - \sigma_{i + 1}) } \right) \frac{k}{\epsilon}\right) + C_{2} \log\left( e\left( \sigma_1 + \frac{1}{\sigma_k} \right) \right)
\end{align*}
where $C_1, C_2$ are the new final positive constants. Summing over $\ell \in [k]$, gives the final required number of iterations. 
\end{proof}
}

%% file: 10.arXiv_appC_acceleration.tex
\newpage

\section{Local Convergence of Nesterov's Accelerated Gradient Descent}\label{app:acceleration}

In this section, we prove Theorem \ref{thm:local acceleration}. The key challenge in establishing this results lies in the fact that Nesterov's acceleration method is not a descent method and that the function $g$ is  only smooth and strongly-convex on a local neighborhood of the global minima (see Appendix \ref{app:non-convex-g}). Thus, we need to ensure   that  Nesterov's acceleration method remain in this local neighborhood once it enters it. To that end, we adapt the results of Nesterov and establish Theorem \ref{thm:local-acceleration-general} which may be of independent interest. The proof of Theorem \ref{thm:local-acceleration-general} is given in \textsection\ref{sec:local-acceleration-general-proof}. The proof of Theorem \ref{thm:local acceleration} is given in \textsection\ref{sec:local-acceleration-proof}.

\subsection{Local acceleration for Nesterov's gradient method in its general form}\label{sec:local-acceleration-general-proof}

Here, we present a general result regarding the local acceleration of Netsterov's method in its general form \cite{nesterov2013introductory}. 
Throughout this section, we consider function $f$ that satisfies:
\begin{assumption}
    $f$ is differentiable, $L$-smooth and $\mu$-strongly-convex on a local neighborhood $\cC_\xi = \lbrace z: \Vert x- x_\star \Vert \le \xi  \rbrace$ of a global minima $x_\star$, for some $\mu, L, \xi > 0$. 
\end{assumption}
For convenience, we define 
$
\underline{\cC}_\xi  = \lbrace x: \Vert x - x_\star \Vert \le \xi/2\rbrace. 
$

Before presenting the algorithm and result we start by introducing certain definitions and notations 
which are needed for describing the method. We also introduce a Lyapunov function that will be useful in characterizing the 
convergence properties of the method. This builds on the so-called \emph{estimate sequence} construction (see Definition 2.2.1 in \cite{nesterov2013introductory}).   

To that end, let $(\alpha_k)_{k \ge 0}$ be a sequence taking value in $(0,1)$, $(y_k)_{k \ge 0}$ be an arbitrary sequence taking values in $\underline{\cC}_\xi$, and $\gamma_0 > 0$. For now the choice of this sequences will be arbitrary but will be made precise later. 
Define a sequence of functions $(\phi_k)_{k \ge 0}$ on $\RR^n$ as follows:
\begin{equation}\label{eq:phi-functions}
\begin{split}
    \phi_0(x) & = f(x_0) + \frac{\gamma_0}{2} \Vert x - x_0 \Vert^2 \\
    \phi_{k+1}(x) & = (1- \alpha_k) \phi_k(x) + \alpha_k \left(f(y_k) + \langle \nabla f(y_k), x - y_k \rangle + \frac{\mu}{2} \Vert x - y_k \Vert^2   \right)
    \end{split}
\end{equation}

In the following Lemma, we see that the functions $(\phi_k)$ have quadratic form. 
\begin{lemma}[Lemma 2.2.3. in \cite{nesterov2013introductory}]\label{lem:estimate-sequence-canonical-form}
The $\phi_k(\cdot)$ defined as per \eqref{eq:phi-functions} have form, for any $x \in \RR^n$, 
\begin{align*}
        \phi_k(x) = \phi^\star_k + \frac{\gamma_k}{2} \Vert x  - v_k\Vert^2,  
\end{align*}
where $\phi_0^\star = f(x_0)$, and for all $k \ge 0$,
\begin{align}
        & \gamma_{k+1} = (1-\alpha_k) \gamma_k + \alpha_k \mu, \nonumber \\
        & v_{k+1}  = \frac{1}{\gamma_{k+1}} \left(  (1 - \alpha_k)\gamma_k v_k  + \alpha_k \mu y_k - \alpha_k \nabla f (y_k) \right) \label{eq: v and gamma} \\
        & \phi_{k+1}^\star  =  (1- \alpha_k) \phi_k^\star + \alpha_k f(y_k) - \frac{\alpha_k^2 }{2 \gamma_{k+1}} \Vert \nabla f(y_k)\Vert^2 \nonumber \\
        & \qquad  + \frac{\alpha_k (1- \alpha_k) \gamma_k}{\gamma_{k+1}} \left( \frac{\mu}{2} \Vert y_k - v_k\Vert^2  + \langle \nabla f(y_k), v_k - y_k \rangle \right). 
\end{align}    
\end{lemma}

Now, we describe the generic Nesterov's method: initialize $x_0 \in \underline{\cC}_\xi$, $v_0 = x_0$, $\gamma_0 > 0$; for $k \ge 0$, define
iterates satisfying
\begin{align}
      \alpha_k^2 & = \frac{(1-\alpha_k) \gamma_k + \alpha_k \mu}{\ell} = \frac{\gamma_{k+1}}{\ell}\label{eq:def-y_k}\\
    y_k & =  \frac{\alpha_k \gamma_k v_k + \gamma_{k+1} x_k }{\gamma_k + \alpha_k \mu} \label{eq:def-y_k}\\
    x_{k+1} & \ \text{s.t.} \ f(x_{k+1})  \le f(y_k) - \frac{1}{2\ell} \Vert \nabla f(y_k) \Vert^2, \text{and} \quad x_{k+1} \in \cC_\xi \label{eq:def-x_{k+1}} \\
    \gamma_{k+1}, v_{k+1} & \quad \text{as in \eqref{eq: v and gamma}}, \nonumber 
\end{align}
where $\ell = 2 L$. In the above, the choices for $\alpha_k, y_k$ and $x_{k+1}$ are motivated by the need to ensure that $\phi_{k+1}^\star \ge f(x_{k+1})$ which in turn will be useful to show decay of a certain Lyapunov function. This will be apparent shortly. 

\begin{lemma}\label{lem:lyapunov-lower}
For $k\ge 0$, if $x_k, y_k, v_k \in \underline{\cC}_\xi$, and $\phi^\star_{k} \ge f(x_k)$, then  $\phi_{k+1}^\star \ge f(x_{k+1})$.  
\end{lemma}

\begin{proof}[Proof of Lemma \ref{lem:lyapunov-lower}.]
We recall that 
\begin{align*}
     \phi_{k+1}^\star & =   (1- \alpha_k) \phi_k^\star  + \alpha_k f(y_k) - \frac{\alpha_k^2 }{2 \gamma_{k+1}} \Vert \nabla f(y_k)\Vert^2 \nonumber  \\ &  + \frac{\alpha_k (1- \alpha_k) \gamma_k}{ \gamma_{k+1}}  \left( \frac{\mu}{2} \Vert y_k - v_k\Vert^2  + \langle \nabla f(y_k), v_k - y_k \rangle \right) 
\end{align*}
which further gives, 
\begin{align*}
     \phi_{k+1}^\star & \ge (1- \alpha_k) \phi_k^\star  + \alpha_k f(y_k) - \frac{\alpha_k^2 }{2 \gamma_{k+1}} \Vert \nabla f(y_k)\Vert^2 + \frac{\alpha_k (1- \alpha_k) \gamma_k}{\gamma_{k+1}} \left( \langle \nabla f(y_k), v_k - y_k \rangle \right).
\end{align*}
We note that $\phi_k^\star \ge f(x_k)$ and since $x_k, y_k \in \underline{\cC}_\xi \subseteq \cC_\xi$ and $f$ is strongly convex on $\cC_\xi$ we have 
\begin{align*}
    (1-\alpha_k) \phi_k^\star + \alpha_k f(y_k)  \ge f(y_k) + (1-\alpha_k )\langle \nabla f (y_k), x_k - y_k\rangle. 
\end{align*}
Therefore, we obtain 
\begin{align*}
     \phi_{k+1}^\star & \ge f(y_k) + (1-\alpha_k )\langle \nabla f (y_k), x_k - y_k\rangle - \frac{\alpha_k^2 }{2 \gamma_{k+1}} \Vert \nabla f(y_k)\Vert^2 + \frac{\alpha_k (1- \alpha_k) \gamma_k}{\gamma_{k+1}} \langle \nabla f(y_k), v_k - y_k \rangle \\
     & = f(y_k) - \frac{\alpha_k^2 }{2 \gamma_{k+1}} \Vert \nabla f(y_k)\Vert^2 + \langle \nabla f(y_k), (1-\alpha_k)(x_k - y_k) + \frac{\alpha_k(1-\alpha_k)}{\gamma_{k+1}}(v_k - y_k) \rangle  
\end{align*}
We further note that by choice of $y_k$ (see \eqref{eq:def-y_k}) and $\gamma_{k+1}$ (see \eqref{eq: v and gamma}), that we have 
\begin{align*}
  (1-\alpha_k)(x_k - y_k) + \frac{\alpha_k(1-\alpha_k)\gamma_k}{\gamma_{k+1}} (v_k - y_k)   & = \frac{(1-\alpha_k)}{\gamma_{k+1}} \left( \gamma_{k+1} x_k + \alpha_k \gamma_k v_k - (\gamma_{k+1} + \alpha_k \gamma_k) y_k \right)   \\
  & =  \frac{(1-\alpha_k)}{\gamma_{k+1}} \left( \gamma_{k+1} x_k + \alpha_k \gamma_k v_k - (\gamma_k + \alpha_k \mu ) y_k \right)\\
   & = 0.
\end{align*}
Thus, we have also
\begin{align*}
     0 & = \langle \nabla f(y_k) , (1-\alpha_k)(x_k - y_k) + \frac{\alpha_k(1-\alpha_k)\gamma_k}{\gamma_{k+1}} (v_k - y_k)  \rangle.  
\end{align*}
Therefore, the above inequalities with the choice of $\alpha_k$ (see \eqref{eq:def-y_k}) and $x_{k+1}$ (see \eqref{eq:def-x_{k+1}}) yields 
\begin{align*}
    \phi^\star_{k+1} & \ge f(y_k) - \frac{\alpha_k^2}{2 \gamma_{k+1}} \Vert f(y_k)\Vert^2 \\ 
    & =  f(y_k) - \frac{1}{2 \ell} \Vert f(y_k)\Vert^2 \\
    & \ge f(x_{k+1}).
\end{align*}
\end{proof}

\begin{lemma}\label{lem:lyapunov-upper} For $k\ge 0$, 
if $x_k, y_k, v_k \in \underline{\cC}_\xi$, then for  $x \in \cC_\xi$, 
$$
    \phi_{k+1}(x) \le (1-\alpha_k) \phi_k(x) + \alpha_k f(x). $$
\end{lemma}

\begin{proof}[Proof of Lemma \ref{lem:lyapunov-upper}.] For all $x \in \cC_\xi$,
    \begin{align*}
        \phi_{k+1}(x) & = (1-\alpha_k) \phi_k(x) + \alpha_k \left(f(y_k) + \langle  \nabla f(y_k), x - y_k \rangle + \frac{\mu}{2} \Vert x - y_k \Vert^2\right) \\
        & \le (1-\alpha_k) \phi_k(x) + \alpha_k f(x),
    \end{align*}
    where in the last inequality we use that fact $x, y_k \in \underline{\cC}_\xi \subseteq  \cC_\xi$ and $f$ is $\mu$-strongly convex.
\end{proof}
Next we argue that $x_k, y_k, v_k$ remain in $\underline{\cC}_\xi$ provided $v_0 = x_0$ is sufficiently close to $x_\star$. 
\begin{theorem} \label{thm:local-acceleration-general}Assume that $\gamma_0 \in  (\mu, L)$, $v_0 = x_0 \in \underline{\cC}_\xi$ such that $\Vert x_0 - x_\star \Vert \le (\xi/2)\sqrt{\mu/L}$. For all $k \ge 0$, we have: 
\begin{itemize}
    \item [(i)] $x_k, y_k, v_k \in \underline{C}_\xi$,
    \item [(ii)] $\phi_k^\star \ge f(x_k)$,
    \item [(iii)] 
     $
        \phi_{k+1}^\star - f(x_{\star}) + \frac{\gamma_{k+1}}{2}\Vert v_{k+1} - x_\star \Vert^2 \le  
        \left(\prod_{i=0}^k (1-\alpha_i) \right) \left( f(x_0) - f(x_\star)  + \frac{\gamma_0}{2}\Vert x_0 - x_\star \Vert^2 \right).
    $
\end{itemize}   
\end{theorem}

\begin{proof}[Proof of Theorem \ref{thm:local-acceleration-general}.]
    It is easy to note, by construction, that for all $k \ge 0$, $\gamma_k$ is  weighted average of $\mu$ and $\gamma_0$, thus $\mu \le \gamma_k \le L$. Now we prove the result by induction.

    \underline{Base case. $k = 0$.}, \textbf{(1)} we have by definition that $x_0, v_0 \in \underline{\cC}_\xi$. Because $y_0$ is a weighted average of $x_0$ and $v_0$ and $\underline{\cC}_\xi$ is convex, we also have that $y_0 \in \underline{\cC}_\xi$.

    \textbf{(2)} By definition we have $\phi_0^\star = f(x_0)$. 
    
    \textbf{(3)} Next, since $x_0,y_0, v_0 \in \underline{\cC}_\xi$, we have by Lemma \ref{lem:lyapunov-upper}, that for all $x \in \cC_\xi$
    \begin{align*}
        \phi_1(x) \le (1-\alpha_0) \phi_0(x) + \alpha_0 f(x),
    \end{align*}
    and recalling that $\phi_k(x) = \phi_k^\star + \frac{\gamma_k}{2}\Vert v_k - x_\star \Vert^2$, we have, in particular ($x = x_\star$), 
    \begin{align*}
          \phi_1^\star - f(x_\star) + \frac{\gamma_1}{2} \Vert v_1 - x_\star \Vert^2   = \phi_1(x_\star) - f(x_\star)  \le (1-\alpha_0) \left(f(x_0) - f(x_\star)  + \frac{\gamma_0}{2} \Vert x_0 - x_\star\Vert^2 \right).
    \end{align*}

    \underline{Induction Step. $k \geq 1$}. Assume the desired statements are true for $k-1$:
    \begin{itemize}
    \item [{\it(i)}] $x_{k-1}, y_{k-1}, v_{k-1} \in \underline{C}_\xi$,
    \item [{\it (ii)}] $\phi_{k-1}^\star \ge f(x_{k-1})$,
    \item [{\it (iii)}] 
     $
        \phi_{k}^\star - f(x_{\star}) + \frac{\gamma_{k}}{2}\Vert v_{k} - x_\star \Vert^2 \le  
        \prod_{i=0}^{k-1} (1-\alpha_{i}) \left( f(x_0) - f(x_\star)  + \frac{\gamma_0}{2}\Vert x_0 - x_\star \Vert^2 \right).
    $
\end{itemize}   
    We wish to argue that the above remains true for $k$. We argue that next.
    
    \textbf{(1)} First, since $x_{k-1}, y_{k-1} ,v_{k-1} \in \underline{\cC}_\xi$, and $\phi_{k-1}^\star \ge f(x_{k-1})$, then by Lemma \ref{lem:lyapunov-lower}, we also have $\phi_k^\star \ge f(x_{k})$.

    \textbf{(2)} Next, combining the fact that $\phi_k^\star \ge f(x_{k})$ with the inequality $(iii)$ from the induction hypothesis (for $k-1$), we obtain 
    \begin{align}\label{eq:H1}
         f(x_k) - f(x_{\star}) + \frac{\gamma_{k}}{2}\Vert v_{k} - x_\star \Vert^2  \le  
        \prod_{i=0}^{k-1} (1-\alpha_{i}) \left( f(x_0) - f(x_\star)  + \frac{\gamma_0}{2}\Vert x_0 - x_\star \Vert^2 \right). 
    \end{align}
     Because by choice of $x_k$ (see \eqref{eq:def-x_{k+1}}), we have  $x_k, x_\star \in \cC_\xi$, and because $f$ is $L$-smooth and $\mu$-strongly convex on $\cC_\xi$, it follows that 
    \begin{align*}
        f(x_k) - f(x_\star)  \ge \frac{\mu}{2} \left\Vert  x_k - x_\star \right \Vert^2   \quad \text{and}\quad 
        f(x_0) - f(x_\star)  \le \frac{L}{2} \Vert x_0 - x_\star \Vert^2,   
    \end{align*}
    Combining the above inequalities with \eqref{eq:H1}, yields
    \begin{align*}
         (\Vert x_\star - x_{k}\Vert  \vee  \Vert x_\star - v_{k} \Vert)^2 
        & \le \left(\prod_{i=0}^{k-1}(1-\alpha_i) \right)\left(\frac{L + \gamma_0}{ \mu + \gamma_1} \right)   \Vert x_\star - x_0\Vert^2   \\
        &  \le \left(\frac{L}{ \mu} \right)   \Vert x_\star - x_0 \Vert^2 \le \frac{\xi^2}{4}, 
    \end{align*}
    where we used the fact that $\alpha_i \in (0,1)$ for all $i \ge 0$, and $\Vert x_\star - x_0 \Vert \le (\xi/2)\sqrt{\mu/L}$, and $ \mu \le \gamma_0, \gamma_1 \le L$.
    This means that $x_{k}, v_{k} \in \underline{\cC}_\xi$. Since $y_{k}$ is weighted average of $x_{k}, v_{k}$, and $\underline{\cC}_\xi$ is convex, we also have $y_{k} \in \underline{\cC}_\xi$.

    \textbf{(3)}  Since $x_{k}, y_{k}, v_{k} \in \underline{\cC}_\xi$ which we have just established, by Lemma \ref{lem:lyapunov-upper}, we have for all $x \in \mathcal{C}_\xi$
    \begin{align*}
        \phi_{k+1}(x) \le (1-\alpha_{k}) \phi_{k}(x) + \alpha_{k} f(x).
    \end{align*}
    In particular, with $x = x_\star \in \cC_\xi$, we have 
    \begin{align*}
        \phi_{k+1}(x_\star) - f(x_\star)  &\le (1-\alpha_{k}) (\phi_{k}(x_\star) - f(x_\star)) \\
        & \le (1-\alpha_{k}) \left(\phi_{k}^\star  - f(x_\star) + \frac{\gamma_{k}}{2} \Vert v_{k} - x_\star\Vert^2 \right)  
    \end{align*}
    Recalling that $\phi_{k}(x_\star) = \phi_k^\star + \frac{\gamma_k}{2} \Vert v_k - x_\star \Vert^2$, and using the above inequality together with the induction hypothesis gives 
    \begin{align*}
         \phi_{k+1}^\star  - f(x_\star) + \frac{\gamma_{k+1}}{2} \Vert v_{k+1} - x_\star \Vert^2  \le \prod_{i=0}^k (1-\alpha_i) \left(f(x_0) - f(x_\star) + \frac{\gamma_0}{2} \Vert x_0 - x_\star \Vert^2 \right) 
    \end{align*}
\end{proof}

We will next provide the following result which quantifies the growth of $\prod_{i=0}^k (1-\alpha_i)$. 
\begin{lemma}[Lemma 2.2.4 in \cite{nesterov2013introductory}]
    For $\gamma_0 \ge \mu $, we have 
    \begin{align*}
        \prod_{i=0}^k(1-\alpha_i)  \le \min \left\lbrace \left( 1 - \sqrt{\frac{\mu}{\ell}}
 \right)^k , \frac{4\ell}{(2 \sqrt{\ell} + k\sqrt{\gamma_0})^2}\right\rbrace.
    \end{align*}
\end{lemma}

\subsection{Proof of Theorem \ref{thm:local acceleration}} \label{sec:local-acceleration-proof}

To prove Theorem \ref{thm:local acceleration}, we apply Theorem \ref{thm:local-acceleration-general}. To that end, let us start by defining $x_\star \in \lbrace - \sqrt{\sigma_1} u_1, +\sqrt{\sigma_1} u_1 \rbrace$,  $\xi := \frac{\sigma_1 - \sigma_2}{15\sqrt{\sigma_1}}$, and   
\begin{align*}
    \cC := \cC_\xi & = \left\lbrace x: \Vert x -x_\star \Vert \le \xi  \right\rbrace, \quad 
      \underline{\cC}  := \underline{\cC}_\xi   =\left\lbrace x: \Vert x - x_\star \Vert \le {\xi}/2  \right\rbrace. 
\end{align*}
From Lemma \ref{lem: smooth + strong convexity}, we 
see that the set $\cC$ constitute a basin of attraction for the function $g$, in the sense that $g$ is $L$-smooth and $\mu$-strongly-convex with 
\begin{align*}
    L = \frac{9\sigma_1}{2}, \qquad \mu = \frac{\sigma_1 - \sigma_2}{4}.
\end{align*}
Next, we have  that $x_{k+1}$ is updated as follows:
$$
    x_{k+1} = y_k - \frac{1}{6 \Vert y_k \Vert^2} \nabla g(y_k).
$$
Before applying Theorem \ref{thm:local-acceleration-general}, we need to ensure the following result: 
\begin{lemma}\label{lem:descent}
    Assume that $y_k \in \underline{\cC}_\xi$. Then, $
        g(x_{k+1}) \le g(y_k) - \frac{1}{4L} \Vert \nabla g(y_k)\Vert^2$, and $x_{k+1} \in \cC_\xi$.
\end{lemma}
\begin{proof}[Proof of Lemma \ref{lem:descent}.]
    We start by noting that $y_k \in \underline{\cC}$. In particular, we have $\Vert y_k  - x_\star \Vert \le \frac{\sigma_1 - \sigma_2}{15\sqrt{\sigma_1}} \le \frac{\sigma_1}{(\sqrt{6} + 2)\sqrt{\sigma_1}}$ and this implies that 
    \begin{align}\label{eq: norm y}
       \frac{3\sigma_1}{4} \le \Vert y_t \Vert^2 \le \frac{3\sigma_1}{2}.
    \end{align}
    Next, we recall that 
    \begin{align}\label{eq:update rule nag plain}
        x_{k+1} =  y_k  - \frac{1}{6\Vert y_k \Vert^2} \nabla g(y_k),
    \end{align}
    which is equivalent to
    \begin{align}\label{eq:update rule nag}
        x_{k+1} - x_\star = y_k - x_\star - \frac{1}{6\Vert y_k \Vert^2} \nabla g(y_k).
    \end{align}
    Because $y_k \in \underline{\cC}$, by $L$-smoothness of $g$ on $\cC$, we have $\Vert \nabla g(y_k) \Vert \le L \Vert y_k - x_\star \Vert$. Starting from the above equality and using triangular inequality gives  
    \begin{align*}
        \Vert x_{k+1} - x_\star \Vert & \le \left( 1 + \frac{L}{6 \Vert y_k \Vert^2 }\right)\Vert y_k - x_\star \Vert  \\
        & \le \left( 1 + \frac{2L}{9  \sigma_1 }\right)\Vert y_k - x_\star \Vert \\
        & \le 2 \Vert y_k - x_\star \Vert \le \xi,
    \end{align*}
    where we use the fact that $L = 9\sigma_1/2$.
    Thus, $x_{k+1} \in \cC$. Again, using the fact that $g$ is $L$-smooth on $\cC$, we also have 
    \begin{align*}
         g(x_{k+1})  & \le  g(y_k) + \langle \nabla g(y_k),  x_{k+1} -y_k  \rangle + \frac{L}{2} \Vert  x_{k+1}- y_k  \Vert^2 \\ 
        & \qquad \qquad (\text{by }L-\text{smoothness})\\
        & \ = g(y_k) - \left(\frac{1}{6 \Vert y_k\Vert^2} - \frac{L}{72 \Vert y_k \Vert^4 } \right) \Vert \nabla g(y_k)\Vert^2 \\
        & \qquad \qquad (\text{using \eqref{eq:update rule nag plain}})\\
        & \ \le  g(y_k) - \frac{1}{18 \sigma_1} \Vert 
 \nabla g(y_k)\Vert^2.  
    \end{align*}
    The last inequality follows because, using \eqref{eq: norm y}, we have   
    \begin{align*}
        \left(\frac{1}{6 \Vert y_k\Vert^2} - \frac{L}{72 \Vert y_k \Vert^4 } \right) & = \frac{1}{6 \Vert y_k \Vert^2} \left( 1 - \frac{9\sigma_1}{24 \Vert y_k \Vert^2}\right) \\ &  \ge \frac{1}{6} \times \frac{2}{3\sigma_1} \left( 1 -  \frac{9 \sigma_1}{24} \times \frac{4}{3\sigma_1}\right) \\
        & = \frac{1}{9 \sigma_1} \left( 1 - \frac{1}{2}\right) = \frac{1}{18 \sigma_1}.
    \end{align*}
    This concludes the proof.
\end{proof}
Specializing further the general scheme of Nesterov's accelerated gradient descent to the case $\gamma_0 = \mu$, and $x_{k+1}$ as per presented before gives the scheme: for all $k \ge 0$ 
\begin{align*}
    \alpha_k & = \alpha := \sqrt{\frac{\mu}{2L}} \\
    \gamma_k & = \mu \\
    y_k & =  \frac{1}{1+ \alpha} x_k + \frac{\alpha }{1 + \alpha } v_k\\
    x_{k+1} & =  y_k - \frac{1}{6 \Vert y_k \Vert^2} \nabla g(y_k)  \\ 
    v_{k+1} & = (1- \alpha) v_k + \alpha \left( y_k - \frac{1}{\mu} \nabla g(y_k)\right) 
\end{align*}
with $v_0 = x_0 $ and $\Vert x_0 - x_\star \Vert \le (\xi/2)\sqrt{\mu/L}$. Note, that  thanks to Lemma \ref{lem:descent} we have that \begin{align*}
    g(x_{k+1}) \le g(y_k) - \frac{1}{18\sigma_1} \Vert  \nabla g(y_k) \Vert^2, \quad \text{and}, \quad x_{k+1} \in \mathcal{C}.
\end{align*}
We may therefore, apply Theorem \ref{thm:local-acceleration-general} to obtain
\begin{align*}
     g(x_{k+1}) - g(x_\star) + \frac{\mu}{2} \Vert v_{k+1} - x_\star \Vert^2  \le \min \left\lbrace \left( 1 - \sqrt{\frac{\mu}{2 L}} \right)^k , \frac{8L}{(2 \sqrt{2L} + k\sqrt{\mu})^2}\right\rbrace   \left( g(x_0) - g(x_\star) + \frac{\mu}{2} \Vert x_{0} - x_\star \Vert^2 \right)
\end{align*}
Since $x_{k+1}, x_{\star}, x_{0}, \in \cC_\xi$, we have by $\mu$-strong convexity of $g$ on $\cC$ we deduce that 
\begin{align*}
    \frac{\mu}{2} \Vert x_{k+1} - x_\star \Vert^2 \le 
    \min \left\lbrace \left( 1 - \sqrt{\frac{\mu}{2 L}}
 \right)^k , \frac{8L}{(2 \sqrt{2L} + k\sqrt{\mu})^2}\right\rbrace  \frac{L + \mu}{2} \Vert x_0 - x_\star \Vert^2
\end{align*}
With the choice of $x_0$, using the inequality $(L/\mu) \Vert x_0 - x_\star \Vert \le \xi$, with the above inequality and replacing $L$ and $\mu$ by there values gives 
\begin{align*}
    \Vert x_{k+1} - x_\star \Vert^2 \le  \min\left\lbrace \left( 1 - \sqrt{\frac{\sigma_1 - \sigma_2}{36 \sigma_1}}\right)^{k},   \frac{144\sigma_1}{(12\sqrt{\sigma_1} + k\sqrt{\sigma_1 - \sigma_2} )^2}\right\rbrace   \left( \frac{\sigma_1 - \sigma_2}{15\sqrt{\sigma_1}} \right). 
\end{align*}
We conclude by noting that Lemma \ref{lem:elem-eq} ensures that 
\begin{align*}
    \Vert x_{k+1} - x_\star \Vert^2  
    & = \vert \Vert x_{k+1}\Vert - \sqrt{\sigma_1} \vert^2 + \Vert x_{k+1}\Vert \sqrt{\sigma_1} \left\Vert \frac{x_{k+1}}{\Vert x_{k+1}\Vert} \pm u_1 \right\Vert^2 \\
    & \ge \vert \Vert x_{k+1}\Vert - \sqrt{\sigma_1} \vert^2 + \frac{\sqrt{3}\sigma_1}{2} \left\Vert \frac{x_{k+1}}{\Vert x_{k+1}\Vert} \pm u_1 \right\Vert^2 \\
    & \ge \max \left\lbrace \vert \Vert x_{k+1}\Vert - \sqrt{\sigma_1} \vert^2,  \frac{\sqrt{3}\sigma_1}{2} \left\Vert \frac{x_{k+1}}{\Vert x_{k+1}\Vert} \pm u_1 \right\Vert^2 \right\rbrace,
\end{align*}
where the second to last inequality follows from the fact that $x_{k+1} \in \cC$, and thus $\Vert x_{k+1}\Vert^2 \ge 3\sigma_1/4$. This concludes the proof. 

\begin{remark}
    Note that in the proof of the result, we can replace  $\frac{\sigma_1 - \sigma_2}{4}$ by $\mu$ for any value of $\mu \le \frac{\sigma_1 - \sigma_2}{4}$ and the result will still hold  with an error rate that is of order 
    \begin{align*}
         \Vert x_{k+1} - x_\star \Vert^2 \le 
         & \min \left\lbrace \left( 1 - \sqrt{\frac{\mu}{2 L}}
        \right)^k , \frac{8L}{(2 \sqrt{2L} + k\sqrt{\mu})^2}\right\rbrace   \left( \frac{\sigma_1 - \sigma_2 }{15\sqrt{\sigma_1}}\right) 
    \end{align*}
\end{remark}

%% file: 11.arXiv_appD_miscellineous.tex
\newpage

\section{Miscellaneous tools and lemmas}

In this section, we present some tools and concepts that we make use of consistently in our proofs. A review of strong convexity and smoothness are presented in \textsection\ref{app:subsection:convex}. In \textsection \ref{app:non-convex-g}, we present Lemma \ref{lem:opt-conditions} characterizing the critical points of the non-convex function $g$, and Lemma \ref{lem: smooth + strong convexity} which establishes the strong convexity and smoothness of $g$ locally. Finally, in \textsection\ref{app:misc:error} we present a few error decomposition lemmas.

\subsection{Smoothness and strong convexity} \label{app:subsection:convex}
In this subsection we review some concepts and results from convex optimization that are relevant to our analysis. We will not prove these results as these are standard (e.g., see \cite{d2021acceleration}).

The first property corresponds to that of smoothness which is characterized as follows:
\begin{definition}[$L$-smoothness]
    Let $f$ be a differentiable function on $\cC$. We say that it is $L$-smooth on $\cC$ if for all $x, y \in \cC$, we have $$
    \Vert \nabla f (x) - \nabla f(y) \Vert \le L  \Vert x - y \Vert.$$
\end{definition}
An important consequence of a function being $L$-smooth is the following inequality:
\begin{lemma}
    Let $f$ be a differentiable function on $\cC$. If it is $L$-smooth on $\cC$, than for all $x,y \in \cC$, $$
    f(x) \le f(y) + \langle \nabla f(y), x - y \rangle + \frac{L}{2} \Vert x - y \Vert^2.   
    $$ 
\end{lemma}
A consequence of the above result is the following:
\begin{lemma}
    Let $f$ be a differentiable function on $\cC$. Assume that $f$ is $L$-smooth on $\cC$, and $x_\star \in \cC$ is a critical point of $f$ (i.e., $\nabla f(x_\star)$ = 0), then  for all $x \in \cX$, 
    \begin{align*}
       \Vert  \nabla f(x) \Vert \le L \Vert x - x_\star \Vert \\
       \qquad \text{and} \qquad f(x) - f(x_\star) \le \frac{L}{2} \Vert x - x_\star \Vert^2 
    \end{align*}
\end{lemma}

To verify that a function $f$ that is twice differentiable is $L$-smooth for some $L >0$, we often inspect the largest eigenvalue of its Hessian. The following lemma clarifies why. 
\begin{lemma}
    Let $f$ be a twice differentiable function on $\cC$.  If  $\lambda_{\max}(\nabla^2f(x)) \le L$ for all $x \in \cC$, then $f$ is $L$-smooth on $\cC$.
\end{lemma}

The second property concerns the convexity of a given function, namely that of strong convexity. Specifically, this property is often characterized via an important inequality as we present below:
\begin{definition}
    Let $f$ be a differentiable function on $\cC$. We say that it is $\mu$-strongly convex on $\cC$ if for all $x,y \in \cC$, 
    $$
    f(x) \ge f(y) + \langle \nabla f(x), y - x \rangle + \frac{\mu}{2} \Vert y - x \Vert^2.   
    $$
\end{definition}
An immediate consequence of strong convexity is the following result
\begin{corollary}
    Let $f$ be a differentiable function on $\cC$. Assume that $f$ is $\mu$-strongly convex on $\cC$, and $x_\star \in \cC$ is a critical point of $f$ (i.e., $\nabla f(x_\star)$ = 0), then  for all $x \in \cX$, 
    \begin{align*}
     f(x) - f(x_\star) \ge \frac{\mu}{2} \Vert x - x_\star \Vert^2 
    \end{align*}
\end{corollary}

To verify whether a function $f$ that is twice differentiable is $\mu$-strongly convex for some $\mu> 0$, we often inspect the minimum eigenvalue of its Hessian.
\begin{lemma}
    Let $f$ be a twice differentiable function on $\cC$. If  $\lambda_{\min}(\nabla^2 f(x)) \ge \mu$ for all $x \in \cC$, then $f$ is $\mu$-strongly-convex on $\cC$.
\end{lemma}

\subsection{Properties of the non-convex function $g$}\label{app:non-convex-g} 

The properties of the function $g$ and its relation to the spectral properties of $M$ have pointed out in early works (e.g., \cite{auchmuty1989unconstrained, chi2019nonconvex}). Here we overview some of these properties which are relevant to our analysis. In the following lemma, we present a characterization of the properties of critical points of $g$:   

\begin{lemma}[1$^{\textrm{st}}$ and 2$^{\textrm{nd}}$ order optimality conditions]\label{lem:opt-conditions}
     Let $x \in \RR^{n}$, the following properties hold:
\begin{itemize}
    \item [(i)] $x$ is a critical point, i.e. $\nabla g(x;M) = 0$, if and only if $x = 0$ or $x = \pm \sqrt{\sigma_i} u_i$, $i\in [k]$.
    \item [(ii)] if for all $i \neq 1$, $\sigma_1 > \sigma_i$, then all local minima of the loss function $g$ are also global minima\footnote{We recall that $x$ is a local minima of $g$ iff $\nabla g(x) = 0$ and $\lambda_{\min}(\nabla^2 g(x)) \ge 0$.}, $0$ is a local maxima, and all other critical points are strict saddle\footnote{We recall that $x$ is a strict saddle point iff $\nabla g(x) = 0$, and  $\lambda_{\min}(\nabla^2 g(x)) < 0$.}. More specifically, the only local and global minima of $g$ are the points $ - \sqrt{\sigma_1} u_1$ and $+\sqrt{\sigma_1} u_1$.
\end{itemize}
\end{lemma}
A proof of Lemma \ref{lem:opt-conditions} can be found for instance in \cite{chi2019nonconvex}.

The function $g$ is locally smooth and strongly convex. This statement is made precise in the following lemma: 
\begin{lemma}\label{lem: smooth + strong convexity}
Define $\cC = \lbrace x \in \RR^n: \Vert x \pm \sqrt{\sigma_1} u_1 \Vert \le \frac{\sigma_1 - \sigma_2}{15\sqrt{\sigma_1}} \rbrace$. Then for all $x \in \cC$, we have 
\begin{align*}
     \frac{(\sigma_1 - \sigma_2)}{4} I_n \preceq \nabla^2 g(x; M) \preceq \frac{9\sigma_1}{2}  I_n. 
\end{align*}
\end{lemma}
The proof follows that of \cite{chi2019nonconvex}, and we provide it below for completeness.
\begin{proof} [Proof of Lemma \ref{lem: smooth + strong convexity}.] 
Let $x \in \cC$, we can easily verify that 
\begin{align}\label{eq:ineq-x}
    \sigma_1 - 0.25 (\sigma_1 - \sigma_2) \le \Vert x \Vert^2 \le 1.15 \sigma_1 
    \quad \text{and} \quad 
    \Vert x - x_\star \Vert \Vert x \Vert \le (\sigma_1 - \sigma_2)/12.  
\end{align}
\underline{\emph{Proving the upper bound.}}
First, we recall that 
\begin{align*}
    \nabla^2 g(x;M) = \Vert x \Vert^2 I_n + 2 xx^\top - M.
\end{align*}
Thus, using the inequalities \eqref{eq:ineq-x}, we obtain
\begin{align*}
    \Vert \nabla^2 g(x;M) \Vert \le 3 \Vert x \Vert^2 + \sigma_1 \le 4.5 \sigma_1.
\end{align*}

\underline{\emph{Proving the lower bound.}} We start  start by lower bounding $xx^\top$ for $x \in \cC$, 
\begin{align*}
     xx^\top  
    & = x_\star x_\star^\top +  x_\star (x-x_\star)^\top + (x - x_\star)x_\star^\top  + (x - x_\star)(x - x_\star)^\top  \\
    & \succeq \sigma_1 u_1 u_1^\top - 2 \Vert x- x_\star \Vert \Vert x_\star \Vert I_n \\
    & \succeq \sigma_1 u_1 u_1^\top - 0.25(\sigma_1 - \sigma_1) I_n  
\end{align*}
We know that $I_n = \sum_{i=1}^n u_i u_i^\top$ and $M = \sum_{i=1}^n \sigma_i u_i u_i^\top$. We may therefore write
\begin{align*}
    \nabla^2 g(x) & \succeq \Vert x \Vert^2 \sum_{i=1}^n u_i u_i^\top  - \sum_{i=1}^n \sigma_i u_i u_i^\top  + 2 \left(\sigma_1 u_1 u_1^\top - 0.25 (\sigma_1 - \sigma_2) \sum_{i=1}^n  u_i u_i^\top \right) \\
    & \succeq \Vert x \Vert^2 \sum_{i=1}^n u_i u_i^\top   - \sum_{i=1}^n \sigma_i u_i u_i^\top  + 2 \left(\sigma_1 u_1 u_1^\top - 0.25 (\sigma_1 - \sigma_2) \sum_{i=1}^n  u_i u_i^\top \right) \\
    & \succeq (\Vert x \Vert^2 + \sigma_1 - 0.5 (\sigma_1 - \sigma_2) ) u_1 u_1^\top + \sum_{i=2}^n (\Vert x \Vert^2 - \sigma_i - 0.5 (\sigma_1 - \sigma_2)) u_i u_i^\top \\
    & \succeq \left( \Vert x \Vert^2  - 0.5 (\sigma_1 - \sigma_2) -\sigma_2 \right) \sum_{i=1}^n u_i u_i^\top \\
    & \succeq 0.25(\sigma_1 - \sigma_2) I_n,
\end{align*}
where in the last inequality we used \eqref{eq:ineq-x}. This concludes our proof.
\end{proof}

 Defining the parameters 
\begin{align*}
    \mu = \frac{\sigma_1 - \sigma_2}{4} \qquad \text{and} \qquad L = 
 \frac{9 \sigma_1}{2},
\end{align*} 
we note that Lemma \ref{lem: smooth + strong convexity} ensures that the function $g$ enjoys two key properties. First, $g$ is locally $L$-smooth near its global minima, meaning that for all $x, y \in \cC$,
 \begin{align*}
     g(x) \le g(y) + \langle \nabla g (y), x - y\rangle + \frac{L}{2} \Vert x - y \Vert^2. 
 \end{align*}
 Second, $g$ is locally $\mu$-strongly convex near its global minima, meaning that for all $x, y \in \cC$,
  \begin{align*}
     g(x) \ge g(y) + \langle \nabla g (y), x - y\rangle + \frac{\mu}{2} \Vert x - y \Vert^2. 
 \end{align*}

\subsection{Error decompositions}\label{app:misc:error}

\begin{lemma}\label{lem:spectral-perturbation}
    Let $M, \tilde{M} \in \RR^{n \times n}$ symmetric matrices. Let $\sigma_1, \dots, \sigma_n$ (resp. $\tilde{\sigma}_1, \dots, \tilde{\sigma}_n$) be the eigenvalues of $M$ (resp. $\tilde{M}$) in decreasing order, and $u_1, \dots, u_d$ (resp. $\tilde{u}_1, \dots, \tilde{u}_k$) be their corresponding eigenvectors. Then, for all $\ell \in [k]$, we have   
    \begin{align}
        \min_{W \in \mathcal{O}^{\ell \times \ell}} \Vert U_{1:\ell} - \tilde{U}_{1:\ell} W \Vert  & \le \sqrt{2} \frac{\Vert M - \tilde{M} \Vert }{\sigma_\ell - \sigma_{\ell+1}} \\
        \vert \sigma_\ell - \tilde{\sigma}_{\ell} \vert  & \le  \Vert M - \widetilde{M} \Vert  
    \end{align}
    where $\mathcal{O}^{\ell \times \ell}$ denotes the set of $\ell \times \ell$ orthogonal matrices, and using the convention that $\sigma_{d+1} = 0$.
\end{lemma}
\begin{proof}[Proof of Lemma \ref{lem:spectral-perturbation}.]
    The proof of the result is an immediate consequence of Davis-Kahan and the inequality $\min_{W \in \mathcal{O}^{\ell \times \ell}} \Vert U_{1:\ell} - \tilde{U}_{1:\ell} W \Vert \le \Vert \sin( U_{1:\ell}, \tilde{U}_{1:\ell}) \Vert $ (e.g., see \cite{cape2019two}). The second inequality is simply Weyl's inequality.
\end{proof}

\begin{lemma}\label{lem:error-decomposition}
    For all $x \in \RR^{n} \backslash\lbrace 0 \rbrace$, we have,
\begin{align*}
     \Vert x - \sqrt{\sigma_1} u_1 \Vert^2 & = \sigma_1 \left(\left( \frac{\Vert x \Vert}{\sqrt{\sigma_1}} - 1 \right)^2  + \frac{\Vert x \Vert}{\sqrt{\sigma_1}}  \left\Vert \frac{x}{\Vert x \Vert} -  u_1 \right\Vert^2 \right),  \\
     \Vert x + \sqrt{\sigma_1} u_1 \Vert^2 & = \sigma_1 \left(\left( \frac{\Vert x \Vert}{\sqrt{\sigma_1}} - 1 \right)^2  + \frac{\Vert x \Vert}{\sqrt{\sigma_1}}  \left\Vert \frac{x}{\Vert x \Vert} +  u_1 \right\Vert^2 \right),    \\
    \Vert xx^\top - \sigma_1 u_1 u_1^\top \Vert^2 & \le \sigma_1^2 \left(
    \left( \frac{\Vert x \Vert}{\sqrt{\sigma_1}} - 1\right)^2
    \left( \frac{\Vert x \Vert}{\sqrt{\sigma_1}} + 1 \right)^2  + \frac{\Vert x  \Vert^2}{2\sigma_1} 
    \left\Vert \frac{x}{\Vert x \Vert} -  u_1 \right\Vert^2
    \left\Vert \frac{x}{\Vert x \Vert} +  u_1 \right\Vert^2
    \right).
\end{align*}
\end{lemma}

\begin{proof}[Proof of Lemma \ref{lem:error-decomposition}.]
    The first and second equality follow immediately  by invoking Lemma \ref{lem:elem-eq} with $y = \sqrt{\sigma_1}u_1$ and $y = - \sqrt{\sigma_1}u_1$.
    The third inequality follows by first upper bounding $\Vert xx^\top - \sigma_1 u_1 u_1^\top   \Vert \le \Vert xx^\top - \sigma_1 u_1 u_1^\top \Vert_\F$, then using Lemma \ref{lem:elem-eq-2}.
\end{proof}

\begin{lemma}\label{lem:elem-eq}
    For all $x, y \in \RR^{n} \backslash \lbrace 0 \rbrace$, we have
    \begin{align*}
        \Vert x - y \Vert^2 = \left\vert \Vert x \Vert - \Vert y \Vert \right\vert ^2  + \Vert x \Vert \Vert y \Vert  \left \Vert  \frac{x}{\Vert x \Vert } - \frac{y}{\Vert y\Vert }   \right \Vert^2    
    \end{align*}
\end{lemma}
\begin{proof}[Proof of Lemma \ref{lem:elem-eq}]
Let $x, y \in \RR^{n} \backslash \lbrace 0 \rbrace$. We have, through simple algebra,
   \begin{align*}
       \Vert x - y \Vert^2 & = \Vert x \Vert^2 + \Vert y \Vert^2 - 2 \langle x , y\rangle \\  
       & = (\Vert x \Vert - \Vert y \Vert)^2 + 2 \Vert x \Vert \Vert  y \Vert   - 2 \langle x , y\rangle \\
       & = (\Vert x \Vert - \Vert y \Vert)^2 +  \Vert x \Vert \Vert  y \Vert  \left(2 - 2  \left\langle \frac{x}{\Vert x \Vert} , \frac{y}{\Vert y \Vert } \right\rangle \right) \\
       & = (\Vert x \Vert - \Vert y \Vert)^2 +  \Vert x \Vert \Vert  y \Vert  \left\Vert  \frac{x}{\Vert x \Vert} - \frac{y}{\Vert y\Vert }\right\Vert^2. 
   \end{align*}
   This concludes the proof.
\end{proof}

\begin{lemma}\label{lem:elem-eq-2}
    For all $x, y \in \RR^{n} \backslash \lbrace 0 \rbrace$, we have
    \begin{align*}
        & \Vert x x^\top - y y^\top \Vert_\F^2 = \left( \Vert x \Vert^2  - \Vert y \Vert^2 \right)^2 + \frac{\Vert x \Vert^2 \Vert y \Vert^2}{2} \left\Vert    \frac{x}{\Vert x \Vert} -  \frac{y}{\Vert y\Vert}  \right\Vert^2 \left\Vert   \frac{x}{\Vert x \Vert} +  \frac{y}{\Vert y\Vert} \right\Vert^2.     
    \end{align*}
\end{lemma}

\begin{proof}[Proof of Lemma \ref{lem:elem-eq-2}.]
    Let $x, y \in \RR^{n} \backslash \lbrace 0 \rbrace$. We have, through simple algebra,
    \begin{align*}
         \left\Vert x x^\top - y y^\top \right\Vert_\F^2 & = \left\Vert x x^\top  \right\Vert_\F^2 + \left\Vert y y^\top  \right\Vert_\F^2  - 2 \tr(xx^\top y y^\top) \\
         & = \left\Vert x \right\Vert^4 + \left\Vert y  \right\Vert^4  - 2 \vert \langle x, y \rangle \vert^2 \\
         & = \left(\Vert x \Vert^2  - \Vert y \Vert^2 \right)^2 + 2 \Vert x \Vert^2 \Vert y \Vert^2  - 2 \vert \langle x, y \rangle \vert^2  \\
         & = \left( \Vert x \Vert^2  - \Vert y \Vert^2 \right)^2  +  \frac{\Vert x \Vert^2 \Vert y \Vert^2}{2} \left( 2 - 2  \left\langle \frac{x}{\Vert x \Vert}, \frac{y}{\Vert y\Vert} \right\rangle \right)  \left( 2 + 2  \left\langle \frac{x}{\Vert x \Vert}, \frac{y}{\Vert y\Vert} \right\rangle \right) \\
         & = \left( \Vert x \Vert^2  - \Vert y \Vert^2 \right)^2  +  \frac{\Vert x \Vert^2 \Vert y \Vert^2}{2} \left\Vert    \frac{x}{\Vert x \Vert} -  \frac{y}{\Vert y\Vert}  \right\Vert^2 \left\Vert   \frac{x}{\Vert x \Vert} +  \frac{y}{\Vert y\Vert} \right\Vert^2 
    \end{align*}
    This concludes the proof.
\end{proof}

\subsection{Helper Lemmas}

{
\begin{lemma}\label{lem:series inequality}
    Let $\alpha,  \gamma \in (0,1)$. We have, for all $k \ge 1$, $\sum_{i=0}^{k-1} \alpha^{k-1-i} \gamma^{i} \le k (\alpha \vee \gamma)^{k-1}$. 
\end{lemma}
\begin{proof}[Proof of Lemma \ref{lem:series inequality}.]
First, we recall the elementary fact, that for all $k \ge 1$, we have. 
\begin{align}
    \alpha^{k} - \gamma^{k} = (\alpha - \gamma) \left(\sum_{i=0}^{k-1} \alpha^{k-1-i} \gamma^i\right).
\end{align}
Now, observe that 
\begin{align}
    \sum_{i=0}^{k-1} \alpha^{k- 1- i} \gamma^{i} & = \textbf{1}\lbrace \gamma = \alpha\rbrace k \gamma^{k-1}  + \textbf{1}\lbrace \gamma \neq \alpha \rbrace  \frac{\alpha^k - \gamma^k}{\alpha - \gamma} \\
    & \overset{(a)}{\le} \textbf{1}\lbrace \gamma = \alpha\rbrace k \gamma^{k-1} + \textbf{1}\lbrace \gamma \neq \alpha \rbrace  k (\alpha \vee \gamma)^{k-1} \\
    & = k (\alpha \vee \gamma)^{k-1} 
\end{align}
where in the inequality $(a)$ we used the mean value theorem to obtain $\vert \alpha^k - \gamma^k \vert \le k (\alpha \vee \gamma)^{k-1}$. 
\end{proof}
}

%% file: 12.arXiv_appE_experiments.tex
\newpage
\section{Algorithm Details}\label{app:experiments}

We provide pseudo-code of the algorithms \texttt{GDSVD} and Power Iteration here.

\begin{algorithm}
\caption{$k$-SVD via gradient descent 
(\texttt{GDSVD})}\label{alg:gdsvd}
\SetKwComment{Comment}{$\triangleright$\ }{}
    \textbf{Input} a symmetric matrix $M$, approximation rank $k$, step-size parameter $\eta$, tolerance $\epsilon$ \;
    
    $M_1 \gets M $ \;
    
    \For{$\ell = 1, \dots, k $}{
    
        $x_0\gets M_\ell z$ where $z$ is a random unit norm vector in $\RR^n$\Comment*[r]{\textcolor{blue}{initialization}}
        $t\gets 0$\; 
        
        $\epsilon^\sigma_{0}, \epsilon^{u}_0 \gets 2\epsilon$\;
        
        \While{$(\epsilon_{t}^\sigma > \epsilon)$ or ($\epsilon_{t}^\sigma > \epsilon)$}{
            $x_{t+1} \gets x_t - \frac{\eta}{\Vert x_{t}\Vert^2} \nabla g(x_t; M_\ell)$ \Comment*[r]{\textcolor{blue}{gradient descent step}}
            $\epsilon^u_{t+1} \gets \left\Vert \frac{x_{t+1}}{\Vert x_{t+1} \Vert} - \frac{x_{t}}{\Vert x_{t}\Vert} \right\Vert$\Comment*[r]{\textcolor{blue}{approximation error of $u_\ell$}}  
            $\epsilon^\sigma_{t+1} \gets \left\vert \Vert x_{t+1} \Vert^2 - \Vert x_{t}\Vert^2 \right\vert$ \Comment*[r]{\textcolor{blue}{approximation error of $\sigma_\ell$}}
            $t \gets t + 1$\;
        }
        $\hat{\sigma}_{\ell} \gets \Vert x_{t}\Vert^2$, $\hat{u}_{\ell} \gets \frac{x_{t}}{\Vert x_{t}\Vert}$ \Comment*[r]{\textcolor{blue}{Recovery of  $\hat{\sigma}_\ell$ and $\hat{u}_\ell$}}
        $M_{\ell+1} \gets M_\ell -\hat{\sigma_\ell}\hat{u}_\ell\hat{u}_\ell^\top$\;
    }
    \KwRet{$\hat{\sigma}_1, \dots, \hat{\sigma}_k$ and $\hat{u}_1, \dots, \hat{u}_k$}\;
\end{algorithm}

\begin{algorithm}
\caption{$k$-SVD via power iterations (\texttt{Power Method})}\label{alg:powersvd}
\SetKwComment{Comment}{$\triangleright$\ }{}
    \textbf{Input} a symmetric matrix $M$, approximation rank $k$, step-size parameter $\eta$, tolerance $\epsilon$ \;
    
    $M_1 \gets M $ \;
    
    \For{$\ell = 1, \dots, k $}
    {
        $x_0\gets M_\ell z$ where $z$ is a random unit norm vector in $\RR^n$ \; \Comment*[r]{\textcolor{blue}{initialization}}
        $t\gets 0$\; 
        
        $\epsilon^\sigma_{0}, \epsilon^{u}_0 \gets 2\epsilon$\;
        
        \While{$(\epsilon_{t}^\sigma > \epsilon)$ or ($\epsilon_{t}^\sigma > \epsilon)$}{
            $x_{t+1} \gets \frac{M x_t}{ \Vert M x_t \Vert }$ \Comment*[r]{\textcolor{blue}{power iteration step}}
            $\epsilon^u_{t+1} \gets \left\Vert x_{t+1} -  x_{t}  \right\Vert$\Comment*[r]{\textcolor{blue}{approximation error of $u_\ell$}}  
            $\epsilon^\sigma_{t+1} \gets \left\vert \Vert M x_{t+1} \Vert - \Vert M x_{t}\Vert \right\vert$ \Comment*[r]{\textcolor{blue}{approximation error of $\sigma_\ell$}}
            $t \gets t + 1$\;
        }
        $\hat{\sigma}_{\ell} \gets \Vert M x_{t}\Vert$, $\hat{u}_{\ell} \gets x_{t} $ \Comment*[r]{\textcolor{blue}{Recovery of  $\hat{\sigma}_\ell$ and $\hat{u}_\ell$}}
        $M_{\ell+1} \gets M_\ell -\hat{\sigma_\ell}\hat{u}_\ell\hat{u}_\ell^\top$\;
    }
    \KwRet{$\hat{\sigma_1}, \dots, \hat{\sigma_k}$ and $\hat{u}_1, \dots, \hat{u}_k$}\;
\end{algorithm}